\documentclass[twoside]{article}

\usepackage[preprint]{aistats2026}

\usepackage[round]{natbib}

\usepackage[utf8]{inputenc}
\usepackage[T1]{fontenc}
\usepackage{amsthm,amssymb,amsmath,amsfonts}
\usepackage{mathtools}
\usepackage{dsfont}
\usepackage[shortlabels]{enumitem}
\usepackage[french,english]{babel}
\usepackage{hyperref}
\usepackage[font=scriptsize]{caption}
\usepackage{multirow} 
\usepackage{xcolor}
\usepackage{tabularx}
\usepackage{subcaption}
\usepackage{graphicx}

\newtheorem{theorem}{\textbf{Theorem}}[section]

\newtheorem{lemma}[theorem]{\textbf{Lemma}}
\newtheorem{corollary}[theorem]{\textbf{Corollary}}
\newtheorem{proposition}[theorem]{\textbf{Proposition}}

\def\a{\alpha}
\def\aa{\alpha}

\def\dd{\delta}

\def\ss{\sigma}

\def\th{\theta}

\def\ee{\varepsilon}
\def\vp{\varphi}

\newcommand{\xo}{{\overline x}}

\newcommand{\ro}{{\overline r}}

\newcommand{\Yo}{{\overline Y}}
\newcommand{\wo}{{\overline w}}
\newcommand{\yo}{{\overline y}}
\newcommand{\zo}{{\overline z}}

\newcommand{\Mo}{{\overline M}}

\newcommand{\etao}{{\overline \eta}}
\newcommand{\llo}{{\overline \lambda}}

\newcommand{\tho}{{\overline{\theta}}}

\newcommand{\RR}{ \mathbb{R}}
\newcommand{\CC}{ \mathbb{C}}

\newcommand{\NN}{ \mathbb{N}}

\newcommand{\EE}{{\mathbb E}}

\newcommand{\cF}{{\mathcal F}}

\newcommand{\cI}{{\mathcal I}}

\newcommand{\cN}{{\mathcal N}}
\newcommand{\cO}{{\mathcal O}}
\newcommand{\cP}{{\mathcal P}}
\newcommand{\cR}{{\mathcal R}}
\newcommand{\cS}{{\mathcal S}}

\newcommand{\cU}{{\mathcal U}}

\newcommand{\cX}{{\mathcal X}}

\newcommand{\bt}{{\widetilde b}}

\newcommand{\xt}{{\widetilde x}}

\newcommand{\zt}{{\widetilde z}}

\newcommand{\aat}{{\widetilde \alpha}}

\newcommand{\ddt}{{\widetilde \delta}}

\newcommand{\vet}{{\widetilde \varepsilon}}

\newcommand{\sst}{{\widetilde \sigma}}
\newcommand{\SSt}{{\widetilde \Sigma}}
\newcommand{\tht}{{\widetilde \theta}}
\newcommand{\oot}{{\widetilde \omega}}

\newcommand{\Ht}{{\widetilde H}}

\newcommand{\Qt}{{\widetilde Q}}
\newcommand{\Rt}{{\widetilde R}}

\newcommand{\cSt}{\widetilde {\mathcal{S}}}
\newcommand{\Ut}{{\widetilde U}}

\newcommand{\Xt}{{\widetilde X}}

\newcommand{\vpt}{{\widetilde{\varphi}}}
\newcommand{\psit}{{\widetilde{\psi}}}

\newcommand{\bh}{{\widehat b}}

\newcommand{\Hh}{{\widehat H}}
\newcommand{\hhh}{{\widehat h}}

\newcommand{\Sh}{{\widehat{S}}}
\newcommand{\Uh}{{\widehat U}}
\newcommand{\vh}{{\widehat v}}

\newcommand{\aah}{{\widehat \a}}

\newcommand{\ssh}{{\widehat{\sigma}}}
\newcommand{\SSh}{{\widehat{\Sigma}}}
\newcommand{\vph}{{\widehat{\varphi}}}
\newcommand{\thh}{{\widehat \theta}}

\DeclareMathOperator{\diag}{{\mathop{\rm diag}}}

\DeclareMathOperator{\supp}{\mathop{\rm supp}}

\DeclareMathOperator{\Var}{\mathop{\rm Var}}

\newcommand{\indic}[1][]{\mathds{1}_{\{#1\}}}

\newcommand{\norm}[3][]{{\left\| #2 \right\|_{#3}^{#1}}}

\newcommand{\norminf}[2][]{{\left\| #2 \right\|_{\infty}^{#1}}}

\newcommand{\lp}{\left(}
\newcommand{\rp}{\right)}

\newcommand{\lc}{\left\{}
\newcommand{\rc}{\right\}}
\newcommand{\lb}{\left[}
\newcommand{\rb}{\right]}

\newcommand{\labs}{\left|}
\newcommand{\rabs}{\right|}

\newcommand{\tr}{\mbox{{\rm tr}}}

\newcommand{\note}[1]{{\textbf{\color{red}#1}}}

\begin{document}

%
\runningtitle{Fast and robust convergence rate for TD(0)
with linear function approximation}

%

\twocolumn[

\aistatstitle{Fast and Robust Convergence Rate for TD(0)
with Linear Function Approximation, Universal Learning Steps and I.I.D. Samples}

\aistatsauthor{ Ziad Kobeissi \And \'Eloïse Berthier }

\aistatsaddress{ L2S, INRIA
\\
Université Paris-Saclay, CentraleSupelec,\\
Gif-sur-Yvette, France
\And U2IS, ENSTA \\ Institut Polytechnique de Paris,\\ Palaiseau, France } ]

\begin{abstract}
    In this paper,
    we study the finite-time behavior
    of the TD(0) temporal-difference method 
    with linear function approximation (LFA).
    We consider on-policy independent and identically distributed (i.i.d.) samples,
    a constant learning step,
    and the Polyak-Juditsky averaging method.
    We establish a new convergence rate,
    for the Mean-Square Error (MSE) on the approximated function,
    that is
    (i) \emph{fast} in the sense that
    it admits an optimal dependency 
    in the number of iterations $k$ (i.e., of order $1/k$),
    (ii) \emph{robust} to ill-conditioning:
    it only depends on an initial error and 
    model-independent constants
    and (iii) \emph{sharp} up to a multiplicative
    constant lower than $11$.
    In particular, it does not depend on 
    the smallest eigenvalue of the uncentered covariance matrix
    of the linear parametrization,
    unlike all pre-existing $O(1/k)$ rates in the TD(0) literature.
    We also introduce PCTD(0), a variant of TD(0), which benefits
    from better convergence properties under an additional
    assumption of strong mixing on the Markov Chain.
\end{abstract}

\vspace*{-0.3cm}

\begin{center}
\textit{\small This is an extended version of a paper accepted at AISTATS 2026.}
\end{center}

\vspace*{-0.3cm}

\section{INTRODUCTION}
\label{sec:intro}
Temporal Difference (TD) learning \citep{sutton1988learning,MR3889951},
is a fundamental algorithm in reinforcement learning (RL)
for estimating the value function associated with a policy.
Among the various TD variants, the TD(0) algorithm stands out
for its simplicity and practical efficiency,
particularly when combined with linear function approximation (LFA).
Despite its widespread empirical success,
obtaining sharp non-asymptotic convergence guarantees for TD(0)
remains a significant theoretical challenge.

When used with LFA,
TD(0) enters the more general framework of
\emph{linear stochastic approximation} (LSA) methods.
%
Since the early work of \cite{robbins1951stochastic}, 
it has been known that stochastic approximation (SA) schemes can achieve a
$O(1/k)$ behavior under strong convexity assumptions, 
but these guarantees heavily depend on the constant of strong convexity.
Later developments, such as the ``robust SA''
and mirror-descent approaches 
of \cite{nemirovski2009robust} 
instead obtain non-asymptotic $O(1/\sqrt{k})$ bounds that
hold for general convex problems and no longer depend
on a strong convexity constant.

For practical SA algorithms, rates are in general
either fast or robust, but rarely both.
The most famous and impactful example of a
fast and robust convergence rate applies to SGD in a particular setting.
While many classical SGD analyses require strong convexity to get $O(1/k)$ rates, 
the breakthrough paper by \cite{bach2013non} showed that in least-squares regression, 
averaged SGD with a constant stepsize still attains a 
$O(1/k)$ convergence rate without any strong convexity assumption,
thus providing \emph{fast} yet \emph{robust}
guarantees in the non-strongly-convex setting.

In the whole field of RL, prior to our work,
we are not aware of any rate
on a non-tabular sample-based numerical method
that is both fast and robust.
For TD(0) with LFA, existing non-asymptotic convergence rates
divide into two classes, aligned with the SA literature.
The first one consists of fast $O(1/k)$ rates,
\cite{lakshminarayanan2018linear,bhandari2018finite,srikant2019finite,patil2023finite,samsonov2024improved,mitra2024simple},
where the constant in the $O(1/k)$ can be arbitrarily large
in practice, especially for large dimensions where
systems arising from modern ML are most often ill-conditioned.
The dependence on the conditioning in the latter rates
generally appears through the smallest eigenvalue
of the uncentered covariance matrix of the linear parameterization,
which we denote $\omega$.
We qualify this constant as being \emph{model-dependent},
where the wording \emph{model} refers, in the usual statistical sense,
to the choice of the parameterization.
It has to be noted that this $\omega$ admits the exact
same definition as the strong convexity in \cite{bach2013non},
making the comparison with this paper relevant.

The second class consists of robust rates,
\cite{bhandari2018finite,liu2021temporal,lee2025finite}.
Even though they are of order $O(1/\sqrt{k})$,
they are often thought to lead to tighter upper bounds in practice
since they do not deteriorate when $\omega$ tends to zero.
On the smallness of $\omega$, \cite{bach2013non} write that
\emph{``typical ML problems are high dimensional and have correlated variables
so that [$\omega$] is zero or very close to zero, 
and in any case smaller than [$O(1/\sqrt{k})$].
This then makes the non-strongly convex methods better.''}

In the relevant literature,
the most popular conjecture on the open question
of whether or not a fast and robust rate is attainable
seems to be the negative answer.
The second author in \cite{bach2013non},
together with their coauthors in \cite{samsonov2024improved},
wrote:
\emph{``Note that the leading term of the bound in 
[our first main result for TD(0) (with i.i.d. samples)]
includes factors of [$1/\omega$].
This dependence is generally unavoidable if one aims to obtain the MSE bound 
[on the approximated functions] that scales as $[1/k]$''.}
Furthermore, while
the main purpose in \cite{lakshminarayanan2018linear}
(which appears in their title)
is to obtain the sharpest rate on LSA and TD(0)
using the same averaging method as in the present work,
they did not attain a robust rate.
They even provide a lower bound that scales as $\omega^{-2}$
on the MSE on the parameters (while we consider the MSE 
on the approximated functions, similarly to \cite{samsonov2024improved}, and others).

\paragraph{Contributions.}
Our main contributions are:
\begin{itemize}
    \item
        Under standard assumptions, 
        for learning rates lower than $\frac{2(1-\gamma)}{(1+\gamma)^2}$,
        we prove the first fast and robust
        convergence rate for TD(0), see Theorem \ref{thm:TD0}.
        Proposition \ref{prop:sharpness} shows
        that this rate is sharp up to a multiplicative
        constant lower than $11$.
        This significantly improves on existing results
        as summarized in Table \ref{tab:SOA}.
    \item 
        In Theorem \ref{thm:TD0_larger_alpha},
        we prove another convergence rate that benefits from the
        largest range of admissible learning rates in the literature
        up to our knowledge, namely $\alpha<\frac{2}{1+\gamma}$.
        Unlike the former rate, it is not robust
        but it remains faster than any existing rate prior to this work.
        We also prove that for $\alpha>\frac{2}{1+\gamma}$, 
        convergence of TD(0) may fail. 
    \item
        Theorem \ref{thm:minibatch_TD0} states that using minibatches of size
        $B\leq (1-\gamma)^{-1}$ reduces the total error by approximately a factor
        $B^{-1}$, affecting not only the part of the error due to the variance but also
        the part due to the bias, which is non-standard.
    \item
        We propose two methods to reduce the dependence on $(1-\gamma)^{-1}$
        of our robust rates when the transition operator of the Markov
        Chain admits a positive spectral gap $g>0$.
        In the first method, we consider different learning
        rates for the constant part of the parametrization
        and the rest, see Theorem \ref{thm:TD0_param_trick}.
        For the second one, we introduce PCTD(0) which uses the differences
        of two independent copies of the Markov Chain to obtain centered
        feature maps.
        This allows one to:
        {\emph (i)} consider $\gamma$ larger than one;
        {\emph (ii)} get a convergence rate independent of $(1-\gamma)^{-1}$
        when $\gamma\leq1$ (but dependent on $g^{-1}$ that may be large in practice).
    \item
        For a more general class of LSA methods, which satisfy Assumptions
        \ref{hypo:SPD}-\ref{hypo:ineq_LSA},
        we prove a convergence rate in $O(1/k)$ with a dependence
        on the model only appearing in a faster converging $O(1/k^2)$ term, 
       see Theorem \ref{thm:LSA}. 
\end{itemize}

\begin{table*}[t]
  \centering
  \begin{tabularx}{\textwidth}{|p{3.9cm}|c|c|c|c|}
      \hline 
      Paper
      &Type of error
      &Step size
      &Averaging
      &Convergence rate
      \\
      \hline \hline
      \multirow{2}{*}{\citet{bhandari2018finite}}
      &$\EE[|v(X,\tho_k)-v(X,\theta^*)|^2]$
      &$O(K^{-\frac12})$
      &yes
      &$O(K^{-\frac12})$
      \\
      &$\EE\lb|\theta_k-\theta^*|^2\rb$
      &$O(k^{-1}\omega^{-1})$
      &no
      &$O(k^{-1}\omega^{-2})$
      \\
      \hline
      \cite{srikant2019finite}
      & $\EE\lb|\theta_k-\theta^*|^2\rb$
      & $ O(k^{-1} \omega^{-1})$
      &no
      & $ O(k^{-1} \omega^{-1})$
      \\
      \hline
      \cite{liu2021temporal} 
      &$\EE[|v(X,\tho_K)-v(X,\theta^*)|^2]$
      & $O(K^{-1/2})$
      & yes
      & $O(K^{-1/2})$
      \\
       \hline
      \cite{lee2025finite}
      &$\EE[|v(X,\tho_K)-v(X,\theta^*)|^2]$
      & $O(K^{-1/2})$
      & yes
      & $O(K^{-1/2})$
      \\
      \hline
      \citet{dalal2017finite}
      &$\EE\lb|\theta_k-\theta^*|^2\rb$
      &$O(k^{-\beta})$
      &no
      &$O(k^{-\beta}\omega^{-1}
      e^{C(1-\beta)^{-1}\omega^{-\frac1{\beta}+1}})$
      \\
      \hline
      \citet{lakshminarayanan2018linear}
      &$\EE\lb|\tho_k-\theta^*|^2\rb$
      &$O(1)$
      &yes
      &$O(k^{-1}\omega^{-2})$
      \\
      \hline
      \citet{patil2023finite}
      &$\EE\lb|\tho_k-\theta^*|^2\rb$
      &$O(1)$
      &yes
      &$O(k^{-1}\omega^{-2})$
      \\
      \hline
      \cite{mitra2024simple}
      &$\EE[|v(X,\tho_k)-v(X,\theta^*)|^2]$
      & $O(1)$
      &yes 
      & $O( k^{-1} \omega^{-2})$
      \\
      \hline
      \citet{samsonov2024improved}
      &$\EE[|v(X,\tho_k)-v(X,\theta^*)|^2]$
      &$O(1)$
      &yes
      &$O(k^{-1}\omega^{-2})$
      \\
      \hline
      \textbf{Our Theorem \ref{thm:TD0}}
      &$\EE[|v(X,\tho_k)-v(X,\theta^*)|^2]$
      &$O(1)$
      &yes
      &$O(k^{-1})$
      \\
      \hline
      \textbf{Our Theorem \ref{thm:TD0_larger_alpha}}
      &$\EE[|v(X,\tho_k)-v(X,\theta^*)|^2]$
      &$O(1)$
      &yes
      &$O(k^{-1}\omega^{-1})$
      \\
      \hline
  \end{tabularx}
  \caption{\textbf{Different rates of convergence in expectation
  from the state of the art on TD(0)
  with i.i.d. samples.}
  On the first, fourth and fifth lines, $K$ is the total number of iterations,
  that has to be known before starting the computations.
  For Dalal et al., $\beta\in(0,1)$ is a fixed parameter,
  observe that: letting $\beta$ tend to one makes the term in 
  the exponential blow up;
  taking $\beta$ not near one implies that the dependency with
  respect to $\omega^{-1}$ is exponential.
  Some poly-logarithmic dependencies are omitted in the rates for simplicity. Note that some, but not all, results presented in the table also apply to the Markov sampling setting.}
  \label{tab:SOA}
\end{table*}

\paragraph{Plan of the paper.}
In Section \ref{sec:biblio}, we review the related literature.
Section \ref{sec:context} introduces our
mathematical setting and states our main convergence results
for TD(0).
Section \ref{sec:LSA} presents our main
convergence results for more general LSA methods.
Some insights into our proof strategies for TD(0) are given 
in Section \ref{sec:insights}.
Numerical simulations are presented in Section \ref{sec:simulations},
and concluding remarks are given in Section \ref{sec:conclusions}.

\paragraph{Notations.}
The euclidean norm on $\RR^d$ is denoted by $|\cdot|$,
using the same symbol as the absolute value for real numbers
and modulus for complex numbers, by a slight abuse of notation.
The set of $d\times d$-sized square matrices is denoted by $\RR^{d\times d}$.
The operator norm on $\RR^{d\times d}$ associated with the euclidean norm
is denoted $\|\cdot\|_{\rm op}$,
and the Frobenius norm is denoted $\|\cdot\|_{F}$.
For $S_1,S_2$ two symmetric matrices, the notation
$S_1\leq S_2$ means that $S_2-S_1$ is positive semi-definite
(similarly we define $\geq$, $<$ and $>$ on the set of symmetric matrices).
If $S$ is a symmetric positive semi-definite matrix,
its square root is denoted by $S^{\frac12}$.

\section{RELATED WORKS}
\label{sec:biblio}



\subsection{TD(0) with LFA}
The literature corpus of theoretical analyses of TD-learning is rich and has continuously developed over the past thirty years.  The tabular version of the TD-learning algorithm was introduced by~\cite{sutton1988learning}, with the first convergence results.  Later, \cite{tsitsiklis1997analysis} proved that TD(0) with LFA converges with probability one, in the case of linearly independent features, yet without an explicit convergence rate. \cite{lakshminarayanan2018linear} proposed a non-asymptotic analysis in the i.i.d.~sampling setting, while \cite{bhandari2018finite} proposed a different analysis, applicable to the Markov sampling setting. Around the same period, \cite{dalal2017finite} and \cite{srikant2019finite} derived alternative convergence rates, and \cite{khamaru2021temporal} introduced a variance-reduced variant of TD(0). Exploring different approximation schemes, \cite{cai2019neural} proposed an analysis of a version of TD(0) where approximation is done using a finite-width one-hidden layer neural network, while \cite{TD0RKHS} studied a non-parametric version of TD learning, using an infinite-dimensional linear approximation scheme.
More recent analyses of TD(0) include those by  \cite{mitra2024simple}, \cite{samsonov2024improved}, \cite{liu2021temporal}, \cite{lee2025finite} and 
\cite{patil2023finite}.
Non-asymptotic convergence rates for TD(0) extended to continuous-time models
have been derived in \cite{kobeissi2022temporal}.

These analyses differ in a few crucial respects: the convergence rate---fast or
slow---often tied to the learning rate, the dependence on unknown model-dependent
quantities, the coverage of Markov sampling, and the possible need for
projection. We detail some below.

\subsection{Convergence rates} An important criterion to compare different analyses is whether or not, and how, the required learning rates and obtained convergence rates depend on quantities defined by the approximation model. While some dependence on the problem instance, like the initial error or the variance of the samples, is inevitable, a dependence on the model is usually unwanted. Indeed, such constants are typically unknown to the user, which might complicate the practical implementation of the method. Moreover, if such a dependence is present in the rate, it might become arbitrarily slow, typically as the dimension of the model grows. 
On the contrary, a rate is called robust if it does not depend on such model-dependent quantities.
Such rates often extend to infinite-dimensional, universal  approximation schemes, like in \cite{TD0RKHS},
which in turn remove approximation errors.

More specifically, in this paper, we study the mean squared error (MSE) in the setting of on-policy i.i.d. samples, whose optimal
convergence rates admit upper bounds proportional to $1/k$, where $k$ is the number of iterations~\citep{bhandari2018finite, lakshminarayanan2018linear, patil2023finite, samsonov2024improved}. In~\cite{bhandari2018finite}, the $1/k$-convergence rate is impractical because the learning rate depends on
an intractable problem-specific constant, namely, $\omega$ that will be introduced later and was already discussed in Section \ref{sec:intro}.
After that, it became an important issue for subsequent authors to use
universal learning steps, which do not depend on any intractable constants.
Such a goal was achieved by \cite{lakshminarayanan2018linear, patil2023finite, samsonov2024improved, mitra2024simple}, whereas all their convergence rates continued to depend on $\omega$.

Table~\ref{tab:SOA} shows the explicit dependency on $\omega$ of some of the state-of-the-art convergence results.
On the one hand, arbitrarily small  $\omega$ may significantly weaken the rates from the above listed works. 
On the other hand, \cite{bhandari2018finite} proved another convergence result
with a rate independent of $\omega$ and any model-dependent constants, at the price of reducing
the speed of convergence in $k$, to $O(1/\sqrt{k})$.
Later works by \cite{liu2021temporal,lee2025finite} obtained similar rates.
Let us point out another common inconvenience of the latter three results:
they require knowing $K$ the total number of iterations before starting the computations,
in order to take a constant learning step of order $O(1/\sqrt{K})$.
This implies that: (i) the rate is no longer guaranteed if we increase the number
of iterations afterwards; (ii) learning can be slow at the beginning since the learning step is small for large~$K$.
As it will appear in our main result, Theorem \ref{thm:TD0}, our rates do not suffer from
this issue.


\subsection{Sampling schemes}  In the literature, the convergence results for TD(0) with linear function 
approximation can be divided into two groups depending on the assumptions made on the data.
In the first case, we assume that the samples used to compute the TD(0)
iterates are i.i.d., directly sampled from the invariant distribution of the Markov chain.
Whereas in the second situation,
it is assumed that the samples are obtained from following the exponentially mixing Markov Chain associated with
the dynamics. This case introduces
a major difficulty:
two consecutive samples are usually correlated, which introduces some complications in the analysis, handled by coupling arguments. Consequently, many analyses in the Markov sampling setting require one
to actively bound the iterates, which is usually ensured by adding a projection to the TD updates~\citep{bhandari2018finite}. Recent works suggest that this projection is not mandatory to obtain convergence~\citep{samsonov2024improved, mitra2024simple, lee2025finite}.

In the present paper, we stick to the i.i.d.~sampling setting. This choice is motivated by the simplicity of the analysis
and the fact that it is often closer to practitioners' settings,
where the samples can be obtained offline, or replayed using buffers,
and do not necessarily strictly follow a trajectory. 
For instance, both AlphaGo Zero \citep{silver2017mastering} and its open-source reimplementation \citep{tian2019elf}
used a replay buffer with a size of $500,000$ games (then, on average, a game contains more than $200$ moves).
In such a situation, two consecutively drawn samples have a probability of around $1/500,000$
of not being independent, so the sampling is much closer to an i.i.d. model than to a single Markov trajectory.
Finally, an important take-home message from the vast TD literature is that the convergence rates obtained in the Markov case do not significantly differ from the i.i.d.~case, except for the introduction of a multiplicative factor due to mixing. 
For all these reasons, extending our present analysis to the Markov setting is left for future work.


\section{CONTEXT AND MAIN RESULTS}
\label{sec:context}
This section focuses solely on the TD(0) method.
The mathematical framework and the main assumptions
are discussed in Subsection \ref{subsec:setting}.
We state our main convergence results
without additional mixing assumptions
in Subsection \ref{subsec:cvg_TD0}.
Under an additional mixing assumption on the Markov chain,
Subsection \ref{subsec:reduce_gamma_dep} shows that
the dependence on $(1-\gamma)^{-1}$ can be reduced
by using two separate learning rates;
Subsection \ref{subsec:PCTD(0)}
shows that this dependence can be entirely
eliminated by replacing standard temporal differences
with differences of independent copies of such quantities,
thereby introducing the method we call PCTD(0).

\subsection{Mathematical setting}
\label{subsec:setting}
For $(X_{\ell},R_{\ell})_{\ell\geq0}$ a Markov Reward Process (MRP),
we consider the problem of policy evaluation,
consisting in approximating the value function $V$
defined by
\begin{equation}
    \label{eq:def_V}
    V(x)
    =
    \EE\lb\sum_{\ell=0}^{\infty}\gamma^{\ell}R_{\ell}\,\Big|\,X_0=x\rb,
\end{equation}
for $\gamma\in[0,1)$,
where $R_{\ell}$ is the reward at step $\ell$ and
the law of $(R_{\ell},X_{\ell+1})$ is given by some probability kernel
$((r,x')\mapsto P(x,r,x'))_{x\in\cX}\in\cP(\RR\times\cX)^{\cX}$.
In particular, $(X_{\ell})_{\ell}$
is a Markov chain of probability transition
$(P_{x'}(x,\cdot))_{x\in\cX}\in\cP(\cX)^{\cX}$, where
$P_{x'}(x)$ is the second marginal of $P(x)$
and $P_r(x)$ its first one.
The value function satisfies the Bellman equation
\begin{equation}
    \label{eq:Bellman}
    V(x)
    =
    \EE\lb R+\gamma V(X')\,|\,(R,X')\sim P(x)\rb.
\end{equation}
Our goal is to approximate $V$ using a parameterized
function $v(\cdot,\theta)$ where $\theta\in\RR^d$
is the parameter that we will learn.
We make the following assumptions:
\begin{enumerate}[label={\bf TD\arabic*}]
    \item
        \label{hypo:iid}
        The Markov Chain induced by $P_{x'}$
        admits an invariant probability measure $m$
        and we access i.i.d. samples
        $(X_k,X'_k,R_k)$
        with $m$ being the law of $X_k$.
    \item
        \label{hypo:R}
        The second moment of $P_r(x)$ is uniformly bounded,
        i.e., $C_R:=\sup_{x\in\cX}\EE_{R\sim P_r(x)}[R^2]<\infty$.
    \item
        \label{hypo:linear}
        The parameterization is linear,
        i.e., $v(x,\theta)=\theta^{\top}\vp(x)$
        where $\vp:\cX\to\RR^d$ are linearly independent on the support of $m$
        and satisfy $|\vp|\leq1$.
\end{enumerate}
Let us define the uncentered covariance matrices:
\begin{equation*}
    \Sigma_0
    =
    \EE\lb\vp(X)\vp(X)^{\top}\rb
    \;\text{ and }\;
    \Sigma_1
    =
    \EE\lb\vp(X)(\vp(X'))^{\top}\rb.
\end{equation*}
The linear independence from Assumption \ref{hypo:linear} implies 
that $\Sigma_0$ is symmetric positive definite.
Let $\omega>0$ be its minimal eigenvalue,
that was already discussed in Sections \ref{sec:intro}, \ref{sec:biblio}.
Consider $U=\Sigma_0^{-\frac12}\Sigma_1\Sigma_0^{-\frac12}$,
we have
\begin{equation*}
    \Sigma_1
    =
    \Sigma_0^{\frac12}U\Sigma_0^{\frac12}
    \;\text{ with }\;
    \|U\|_{\rm op}\leq 1,
\end{equation*}
as a consequence of
the Cauchy-Schwarz inequality.

Under Assumptions \ref{hypo:iid}-\ref{hypo:linear},
the TD(0) algorithm consists in the following iterative method,
from an initial parameter $\theta_0$ and for $k\geq1$:
\begin{equation*}
    \theta_k
    =
    \theta_{k-1}
    -\alpha\dd_k\vp(X_k),
\end{equation*}
where 
$\dd_k=
v(X_k,\theta_{k-1})
-\gamma v(X'_k,\theta_{k-1})
-R_k$
is named the temporal difference.
As is usual in the literature,
we do not consider the convergence of~$\theta_k$
itself, but of its 
Polyak-Juditsky averaging, denoted $\tho_k$,
defined by
$$
\tho_k
=
\frac1k
\sum_{i=0}^{k-1}\theta_i.
$$
Under a suitable choice of the learning step $\alpha$,
it is known that $\tho_k$ converges to $\theta^*$,
which satisfies
\begin{equation*}
    H_0\theta^*
    =
    b_0,
\end{equation*}
where $H_0=(\Sigma_0-\gamma\Sigma_1)$
and $b_0=\EE\lb R\,\vp(X)\rb$.
The value function associated
with $\theta^*$ is the unique solution
of the projected Bellman equation,
\begin{equation*}
    v(\cdot,\theta^*)
    =
    \Pi_{\lc\theta^{\top}\vp\rc}
    \lp
    x\mapsto \EE_{(X',R)\sim P(x)}\lb R+\gamma v(X',\theta^*)\rb
    \rp,
\end{equation*}
which is similar to the Bellman equation \eqref{eq:Bellman}
with an additional projection step onto the set of admissible
functions.
In particular, if
the actual value function~$V$ belongs to the set of admissible functions,
we get $V=v(\cdot,\theta^*)$ by uniqueness of the solution of the projected
Bellman equation. Otherwise, the projection step generally introduces an approximation error.

\subsection{Convergence rate of TD(0)}
\label{subsec:cvg_TD0}
See Figure~\ref{fig:alphas}
for an overview of the results stated in this section.
Our main result writes as follows.
\begin{theorem}
    \label{thm:TD0}
    Assume \ref{hypo:iid}-\ref{hypo:linear}.
    For $\gamma\in[0,1)$, define 
    $\alpha_0(\gamma)=\frac{2(1-\gamma)}{(1+\gamma)^2}$
    and $\alpha_1(\gamma)=\frac{2}{1+\gamma}>\alpha_0(\gamma)$.
    For $0<\alpha<\alpha_0(\gamma)$,
    and $k\geq1$, we have
    \begin{equation*}
    \begin{aligned}
        &\EE_{X\sim m}\lb|v(X,\tho_k)-v(X,\theta^*)|^2\rb
        \leq
        \frac1{(1-\gamma)^2k}
        \\
        &\lp\lp\frac{|\theta_0-\theta^*|^2}
        {2\lp 1-\frac{\alpha}{\alpha_1(\gamma)}\rp\frac{\alpha}{1-\gamma}}\rp^{\frac12}
        +\lp\frac{de\ss_0^2(1+\ee_{k,\gamma})}{1-\frac{\alpha}{\alpha_0(\gamma)}}\rp^{\frac12}\rp^2,
    \end{aligned}   
    \end{equation*}
    with $\ee_{k,\gamma}=
    \frac{4}{k(1-\gamma)}
    +3(1-\gamma)
    +\frac{10}k+\frac2{k(2-(1+\gamma)\alpha)}$,
    $\ss_0^2=\norminf{x\mapsto \EE[\dd(X,R,X',\theta^*)^2\,|\,X=x]}$
    and $\dd(x,s,x',\theta^*)=
    v(x,\theta^*)-\gamma v(x',\theta^*)-s$. 

\end{theorem}
As is customary in SA,
the above rate is the sum of
a bias term (which depends on $|\theta_0-\theta^*|$)
and a variance term (which depends on $\sigma_0^2$).

Given the above rate,
the interesting regime is for $k\gg(1-\gamma)^{-2}$.
In this case and for $\gamma$ near one,
$\ee_{k,\gamma}$ is small.


As explained in the introduction, the above rate does not
depend on $\omega$ the smallest eigenvalue of $\Sigma_0$,
nor on any constant that might generally be qualified as model-dependent.
The only two quantities in this rate that are unknown a priori come from
the mathematical problem at hand.
They are the initial error $|\theta_0-\theta^*|^2$
and the quantity $\ss_0^2$ that is often referred
to as a variance (since $\EE\lb\dd(X,R,X',\theta^*)\rb=(\Sigma_0-\gamma\Sigma_1)\theta^*-b=0$)
or a residual of the Bellman equation, at the limit~$\theta^*$.
We can easily bound $\ss_0^2$ using
$\ss_0^2\leq 
4|\theta^*|^2+2C_R$,
but we think that this inequality is in general very loose.

To the best of our knowledge, this is the first convergence result
for TD(0) that states a rate which is optimal in $k$
and independent of $\omega$ at the same time.
We refer to Table \ref{tab:SOA} for the dependency on $\omega$
of some state-of-the-art results in comparable settings to ours.

Observe that the upper bound in Theorem \ref{thm:TD0}
gets unbounded for $\alpha$ near $\alpha_0(\gamma)$.
In this case, it is standard in SA to restrict further
the range of $\alpha$ by a factor $\frac12$
and to consider only $\alpha\leq\frac{\alpha_0(\gamma)}2$.
Under this standard regime, 
let us discuss the sharpness of Theorem \ref{thm:TD0}.
\begin{proposition}
\label{prop:sharpness}
    Consider the constant Markov Chain on $\cX=\{1,\dots,d\}$,
    i.e. $X'=X$ almost surely.
    Take 
    $(X_k)_{k\geq0}$ i.i.d. with $X_1\sim m$,
    where $m\equiv\frac1d$ is
    the uniform probability measure.
    Take $(R_k)_{k\geq1}$ i.i.d.
    with $R_1\sim\cN(0,\ss_0^2)$ for $\ss_0^2>0$,
    and $\vp_i(x)=\sqrt{d\omega}\indic[i=x]$
    for $\omega\in(0,\frac1d]$ and $1\leq i,x\leq d$.
    Assumptions \ref{hypo:iid}-\ref{hypo:linear} are satisfied
    and $\omega$ is the smallest eigenvalue of $\Sigma_0$.
    Fix $\lambda\in(0,\frac12]$ and $\alpha=\lambda\alpha_0(\gamma)$,
    in the regime $k\to\infty$, $\gamma\to1$
    and $k(1-\gamma)\alpha\omega\to 5.5$, we have
    \begin{equation*}
    \begin{aligned}
        &\EE_{X\sim m}\lb|v(X,\tho_k)-v(X,\theta^*)|^2\rb
        \gtrsim
        \frac1{11(1-\gamma)^2k}
        \\
        &\lp\lp\frac{|\theta_0-\theta^*|^2}
        {2\lp 1-\frac{\alpha}{\alpha_1(\gamma)}\rp\frac{\alpha}{1-\gamma}}\rp^{\frac12}
        +\lp\frac{de\ss_0^2(1+\ee_{k,\gamma})}{1-\frac{\alpha}{\alpha_0(\gamma)}}\rp^{\frac12}\rp^2,
    \end{aligned}   
    \end{equation*}
\end{proposition}
This implies that our upper bound 
in Theorem \ref{thm:TD0} is sharp up to a constant 
at most equal to $11$, for 
$0<\alpha\leq\frac{\alpha_0(\gamma)}2$.
We refer to Corollary \ref{cor:sharpness}
where we prove that our upper bounds for the bias
and variance terms are sharp up to constants
$1.25$ and $2e$ respectively.

Observe that this proposition holds for any $\omega\in(0,1)$
and does not rely on
examples with specific or ill-conditioned structures;
instead, our lower bound is obtained
on one of the simplest examples one can think of.
Therefore, we believe that our upper bound in Theorem \ref{thm:TD0}
captures the actual convergence speed
of practical computations of TD(0), on average and
up to a multiplicative constant that should not be too big.

Now, let us describe what happens when
$\alpha\geq\alpha_0(\gamma)$.
\begin{theorem}
\label{thm:TD0_larger_alpha}
    Assume \ref{hypo:iid}-\ref{hypo:linear}.
    For $\gamma\in[0,1)$,
    $0<\alpha<\alpha_1(\gamma)=\frac{2}{1+\gamma}$
    and $k\geq1$, we have
    \begin{equation*}
    \begin{aligned}
        &\EE_{m}\lb|v(X,\tho_k)-v(X,\theta^*)|^2\rb
        \leq
        \frac{|\theta_0-\theta^*|^2}
        {\lp 1-\frac{\alpha}{\alpha_1(\gamma)}\rp\alpha(1-\gamma)k}
        \\
        &+
        \frac{2d}{(1-\gamma)^2k}
        \lp \frac{\alpha\ss_1^2}{\lp 1-\frac{\alpha}{\alpha_1(\gamma)}\rp\alpha_0(\gamma)\omega}
        +\ss_0^2\rp
        \\
        &\times
        \lp 3+\frac{1}{2\lp 1-\frac{\alpha}{\alpha_1(\gamma)}\rp\alpha(1-\gamma)\omega k}\rp,
    \end{aligned}   
    \end{equation*}
    with $\ss_1^2=\Var(\dd(X,R,X',\theta^*))\leq\ss_0^2$.
    Moreover,
    there exist an MRP and feature functions
    such that Assumptions \ref{hypo:iid}-\ref{hypo:linear}
    are satisfied
    and, for any $\alpha>\alpha_1(\gamma)$,
    \begin{equation*}
        \lim_{k\to\infty}
        \EE_{m}\lb|v(X,\tho_k)-v(X,\theta^*)|^2\rb
        =
        \infty
        =
        \lim_{k\to\infty}\EE\lb|\theta_k|^2\rb.
    \end{equation*}
\end{theorem}
\begin{figure}
    \centering
    \includegraphics[width=0.9\linewidth]{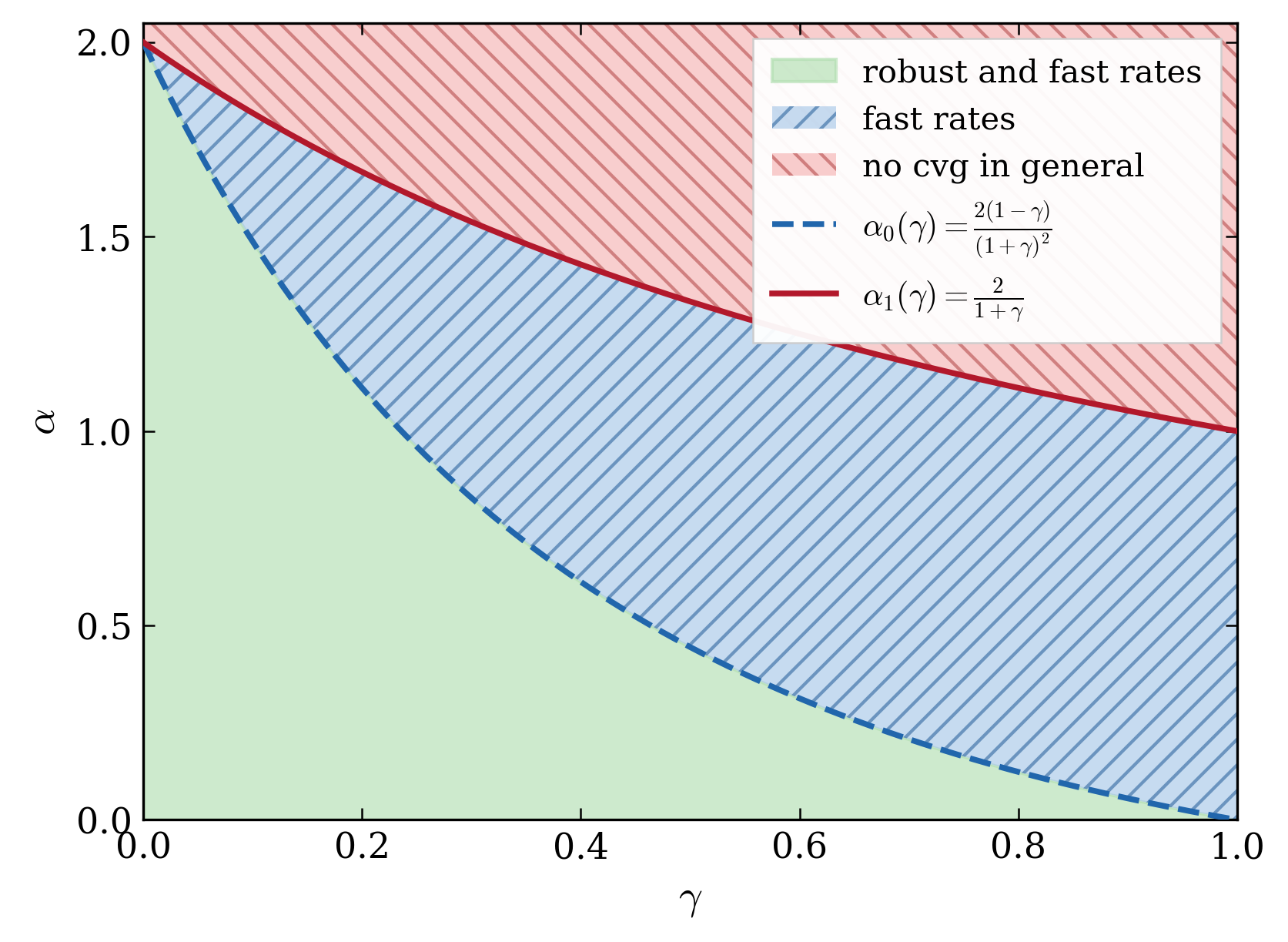}
    \caption{Proven convergence behaviors depending on the
    learning rate $\alpha$.
    For $0<\alpha<\alpha_0(\gamma)$, we proved a
    robust and fast convergence rate (Theorem \ref{thm:TD0})
    that is sharp up to a
    constant lower than $11$ (Proposition \ref{prop:sharpness}).
    For $\alpha_0(\gamma)\leq\alpha<\alpha_1(\gamma)$, 
    we proved faster convergence rate than any
    previous work (Theorem \ref{thm:TD0_larger_alpha}).
    For $\alpha>\alpha_1(\gamma)$, we proved that
    TD(0) may diverge
    (Theorem \ref{thm:TD0_larger_alpha}).}
    \label{fig:alphas}
\end{figure}
Let us discuss the latter convergence result.
In the regime $\alpha\geq\alpha_0(\gamma)$,
the rate is slower than that of Theorem \ref{thm:TD0}.
However, it applies to a wider range of learning rates.
Furthermore, it is both faster and valid for a broader
range of $\alpha$ than any previously known rate.
To our knowledge,
\cite{lakshminarayanan2018linear} allowed the 
widest range of learning rates, namely $\alpha<1$,
while
\cite{samsonov2024improved} achieved
the fastest rate (for $\alpha\leq\frac{1-\gamma}{256}$),
with a variance term of order
$O(\frac{\alpha\ss_0^2}{(1-\gamma)^3\omega^2 k})$.
In comparison, we obtain
$O(\frac{\alpha\ss_1^2}{(1-\gamma)^3\omega k})$.
Moreover, 
Theorem \ref{thm:TD0_larger_alpha} shows that this range
cannot be extended further in general.

\subsection{Improved rates using minibatch}
\label{subsec:minibatch}
Recall that the upper bounds on the convergence rates
stated in the previous subsection decompose 
into a bias term and a variance term. 
As is customary in ML, using minibatches of size $B\geq1$
divides the variance term by $B$,
while leaving the bias term essentially unchanged.
More surprisingly, here
the range of admissible learning steps 
from Theorem \ref{thm:TD0} grows approximately linearly in $B$
in the regime $B\leq (1-\gamma)^{-1}$.
This in turn also makes it possible to reduce the bias term by a factor
proportional to $B$.
\begin{figure}
    \centering
    \includegraphics[width=0.9\linewidth]{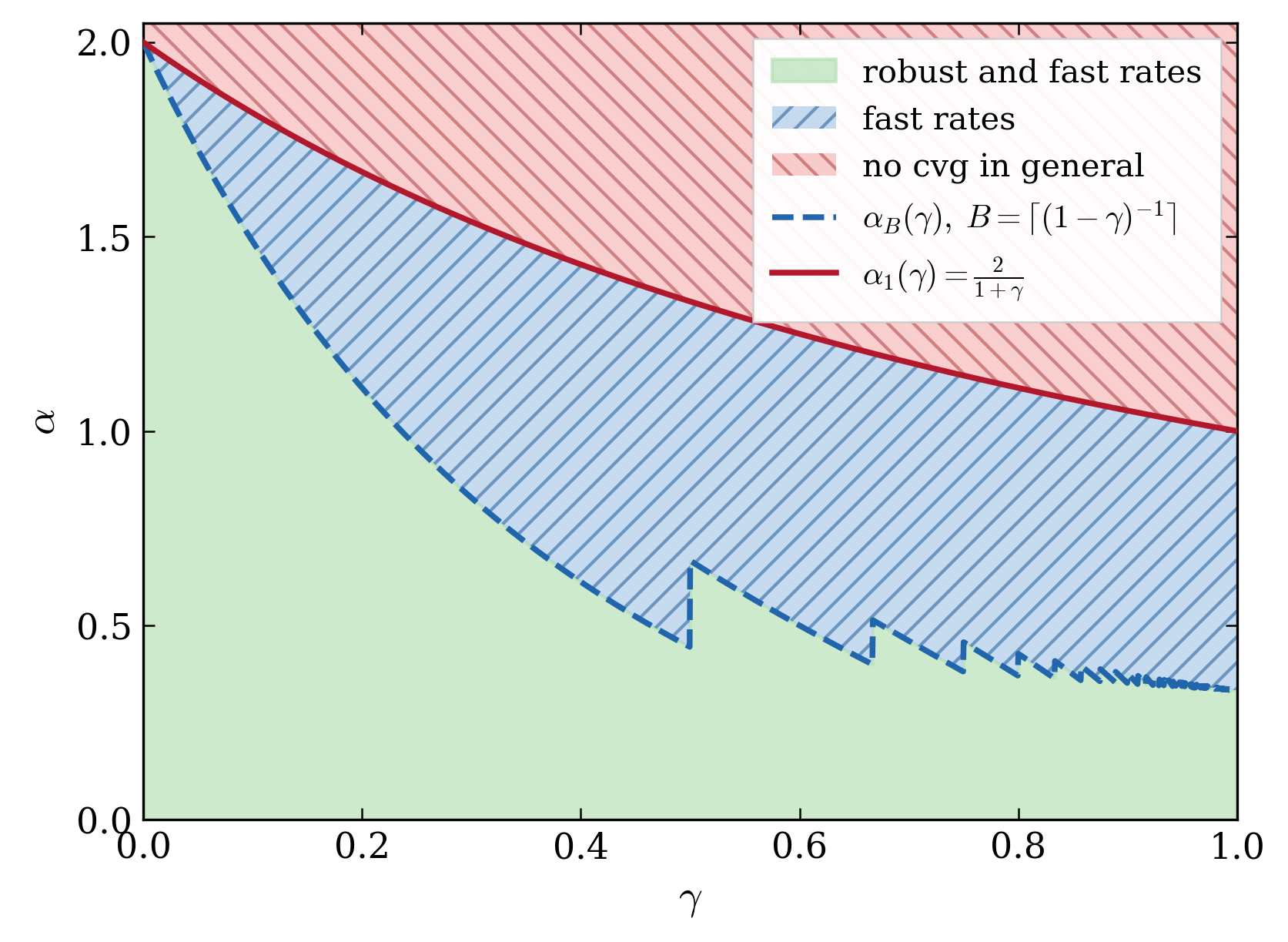}
    \caption{
    Same regime as in Figure \ref{fig:alphas},
    but with TD(0) using minibatches
    of size $B=\lceil(1-\gamma)^{-1}\rceil$.
   }
    \label{fig:alphas_minibatch}
\end{figure}
\begin{theorem}
    \label{thm:minibatch_TD0}
    Assume \ref{hypo:iid}-\ref{hypo:linear}.
    Consider TD(0) with minibatches of size $B\geq1$.
    For $\gamma\in[0,1)$, 
    $0<\alpha<\alpha_B(\gamma)
    :=\frac{2B(1-\gamma)}{(1+\gamma)(2\gamma+B(1-\gamma))}$
    and $k\geq1$, we~have
    \begin{equation*}
    \begin{aligned}
        &\EE_{X\sim m}\lb|v(X,\tho_k)-v(X,\theta^*)|^2\rb
        \leq
        \frac1{B(1-\gamma)^2k}
        \\
        &\lp\lp\frac{|\theta_0-\theta^*|^2}
        {2\lp 1-\frac{\alpha}{\alpha_1(\gamma)}\rp\frac{\alpha}{B(1-\gamma)}}\rp^{\frac12}
        +\lp\frac{de\ss_0^2(1+\vet_{k,\gamma})}{1-\frac{\alpha}{\alpha_B(\gamma)}}\rp^{\frac12}\rp^2,
    \end{aligned}   
    \end{equation*}
    with $\vet_{k,\gamma}=
    \frac{8}{k(1-\gamma)^2}
    +3(1-\gamma)
    +\frac{4}{k(1-\frac{\alpha}{\alpha_1(\gamma)})^2}$.
\end{theorem}
This result suggests that the minibatch size should be chosen as a function of $(1-\gamma)^{-1}$.
In particular, when $1-\gamma$ is small,
setting $B=\lceil c_B(1-\gamma)^{-1}\rceil$
for some fixed $c_B\in[0.1,1]$ yields a convergence rate of
$O\lp\frac{|\theta_0-\theta^*|^2+\ss_0^2}{1-\gamma}\rp$
with the learning rate
$\alpha=\frac{\alpha_B(\gamma)}2\approx \frac{c_B}{2(2+c_B)}$.
This improves upon the rate of Theorem \ref{thm:TD0} 
by a factor of $1-\gamma$,
at the cost of using $O((1-\gamma)^{-1})$ additional samples per iteration.
In other words, the ratio of the error to the number of samples
used is essentially the same whether the computations are carried out sequentially
(i.e., with $B=1$)
or in parallel with a batch of size $O((1-\gamma)^{-1})$.
Since the wall-clock time required to process a batch 
is in general comparable to that of a single sample, 
this translates into a factor $(1-\gamma)^{-1}$
saved in actual computation time.

Proposition \ref{prop:sharpness} and Theorem \ref{thm:TD0_larger_alpha}
extend straightforwardly to the minibatch setting,
but without a gain comparable to that between
Theorems \ref{thm:TD0} and \ref{thm:minibatch_TD0}, 
so we refrain from carrying out such an extension.

\subsection{Reducing the dependence on $1-\gamma$}
\label{subsec:reduce_gamma_dep}
In fact, $| \theta^* |$ generally grows as $O((1-\gamma)^{-1})$
when $\gamma$ tends to $1$,
see Lemma \ref{lem:theta*_to_infty}.
Therefore, the bias term in Theorem \ref{thm:TD0}
is of order $O(k^{-1}(1-\gamma)^{-4})$.
Here, we propose a parameterization trick
that allows us to keep~$\theta^*$ bounded
and obtain a $O(k^{-1}(1-\gamma)^{-2})$ rate.
\begin{enumerate}[label={\bf td3}]
    \item
        \label{hypo:linear_bis}
        Let $\vp:\cX\to\RR^d$ be linearly independent on
        the support of $m$,
        with $\vp_1\equiv c_{\gamma}$ and
        $|\vp_{(-1)}|\leq1$.
\end{enumerate}
In fact, the latter assumption
simply corresponds to having a different
learning rate for the constant part of the parametrization
and for the rest, as follows.
\begin{lemma}
\label{lem:parametrization_trick}
    For $C>0$, TD(0) computes the same value function
    in any of the following cases:
    \\
    \hspace*{0.2cm}
    {\emph(i)}
            \ref{hypo:linear_bis} with $c_{\gamma}=C$,
            $\theta_0$ and $\alpha>0$.
    \\
    \hspace*{0.15cm}
    {\emph(ii)}
            \ref{hypo:linear_bis} with $c_{\gamma}=1$,
            $\tht_0:=(C\theta_{0,1},\theta_{0,(-1)})$ and 
            $\aat=\diag(C^2\alpha,\alpha,\dots,\alpha)\in\RR^{d\times d}$.
\end{lemma}
Let us introduce the following standard assumption.
\begin{enumerate}[label={\bf TD4}]
    \item
        \label{hypo:spectral_gap}
        The Markov Chain admits
        a positive spectral gap,
        i.e., there exists $g>0$ such that,
        for any $f\in L^2(m)$ with $\EE[f(X)]=0$,
        \begin{equation*}
            \EE_m\lb\EE[f(X')\,|\,X]^2\rb
            \leq
            (1-g)^2\EE_m\lb f(X)^2\rb
        \end{equation*}
\end{enumerate}
Observe that $g$ is a priori unknown 
and may be arbitrarily small in practice.
However, it is model-independent: it does not
depend on the choice of the feature maps $\vp$
but only on the Markov Chain itself.
\begin{theorem}
\label{thm:TD0_param_trick}
    Assume \ref{hypo:iid}, \ref{hypo:R},
    \ref{hypo:linear_bis}
    with $c_{\gamma}=\frac1{1-\gamma}$
    and \ref{hypo:spectral_gap}.
    We have
    $(\theta^*,\ss_0^2)$ 
     uniformly bounded
    in~$\gamma$,
    for $\gamma\in[\frac12,1)$ and 
    $0<\alpha<\alpha_2(\gamma)
    :=\frac{2(1-\gamma)}{1+(1+\gamma)^2}$, 
    \begin{align*}
        &\EE_{X\sim m}\lb|v(X,\tho_k)-v(X,\theta^*)|^2\rb
        \leq
        \frac{|\theta_0-\theta^*|^2}
        {2\alpha(1-\gamma)(1-\frac{\alpha}{\alpha_2(\gamma)})k}
        \\
        &+\frac{de(1+\ee_{k,\gamma})}{(1-\gamma)^2k}
        \lp\sqrt{5}|\theta_0-\theta^*|
        +\frac{\ss_0}{\sqrt{1-\frac{\alpha}{\alpha_2(\gamma)}}}\rp^2.
    \end{align*}  
\end{theorem}
Under the setting of Theorem \ref{thm:TD0_param_trick}, TD(0) satisfies the third
inequality of Assumption \ref{hypo:ineq_LSA} only with a constant that grows
unboundedly as $\gamma$ tends to one. We must therefore adapt the proof of
Theorem \ref{thm:TD0} to avoid relying on it. Unfortunately, we could not extend
Theorems \ref{thm:TD0_larger_alpha} and \ref{thm:minibatch_TD0} to this setting,
as their proofs rely too heavily on this inequality.
Theorem \ref{thm:TD0_param_trick} nonetheless extends to minibatches, with the
variance term divided by $B$, but without the wider range of admissible learning
rates obtained in Theorem \ref{thm:minibatch_TD0}.

\subsection{Pairwise Centered TD(0) (PCTD(0))}
\label{subsec:PCTD(0)}
As illustrated in the previous subsection,
the constant part of the value function requires significantly more computation
to learn than the non-constant part.
Moreover, in RL it is generally sufficient to know the value function 
up to an additive constant. This is why, in this subsection, 
we discuss a method that only approximates the value function up to an additive constant,
allowing for faster convergence over a broader range of admissible learning rates.

Without loss of generality,
the samples of Assumption \ref{hypo:iid}
can be divided into two independent sequences
$(X_k,X'_k,R_k)$ and $(\Xt_k,\Xt'_k,\Rt_k)$.
We then define the \emph{pairwise centered temporal difference}
(PCTD) as
\begin{equation}
\label{eq:PCTD}
\begin{aligned}
    \ddt_k
    =&
    \frac12\bigl(v(X_k,\theta_{k-1})
    -\gamma v(X'_k,\theta_{k-1})
    -R_k
    \\
    &
    -(v(\Xt_k,\theta_{k-1})
    -\gamma v(\Xt'_k,\theta_{k-1})
    -\Rt_k)\bigr).
\end{aligned}
\end{equation}
Then PCTD(0) is defined using the update rule
\begin{equation}
\label{eq:PCTD(0)}
    \theta_k
    =
    \theta_{k-1}
    -\alpha\ddt_k(\vp(X_k)-\vp(\Xt_k)).
\end{equation}
Its limit is $\thh^*=\Hh^{-1}\EE[R(\vp(X)-\mu)]$
where $\Hh:=H-(1-\gamma)\mu\mu^{\top}$
is invertible under
\ref{hypo:spectral_gap} if the
set of parameterized functions does not contain
$\mathds{1}$.
\begin{proposition}
    \label{prop:PCTD0}
    PCTD(0) is an instance of TD(0) on
    the MRP defined 
    on $\cX\times\cX$
    by $((X_{\ell},\Xt_{\ell}),R_{\ell}-\Rt_{\ell})_{\ell\geq0}$
    with feature maps $\vpt(x,\xt)=\frac1{\sqrt2}(\vp(x)-\vp(\xt))$,
    where $(X_{\ell},R_{\ell})_{\ell\geq0}$ 
    and $(\Xt_{\ell},\Rt_{\ell})_{\ell\geq0}$ 
    are two independent copies of the
    MRP defined in Section \ref{subsec:setting}.
    Therefore,
    all existing results for TD(0)
    apply to PCTD(0),
    for both i.i.d. and Markovian sampling schemes.
\end{proposition}

However, PCTD(0) has an important advantage
over TD(0): it is centered, i.e.,
$\EE_{m\otimes m}[\vpt(X,\Xt)]=0$.

\begin{theorem}
    \label{thm:PCTD0}
    Assume \ref{hypo:iid}-\ref{hypo:spectral_gap}.
    The rates from Theorems \ref{thm:TD0},
    \ref{thm:TD0_larger_alpha} and \ref{thm:minibatch_TD0}
    hold with $\vp$ replaced by $\vpt$,
    $\theta^*$ by $\thh^*$,
    $\gamma$ by $(1-g)\gamma$,
    $\alpha_0(\gamma)$ by $\aah_0(\gamma):=\frac{1-(1-g)\gamma}{(1+\gamma)^2}$,
    $\alpha_1(\gamma)$ by $\aah_1(\gamma):=\frac{1}{1+\gamma}$,
    $\alpha_B(\gamma)$ by $\aah_B(\gamma):=
    \frac{B(1-(1-g)\gamma)}{(1+\gamma)(2\gamma+B(1-(1-g)\gamma))}$,
    $\ss_0^2$ by $2\ssh_0^2$ with
    $\ssh_0^2:=\norminf{x\mapsto 
    \EE[(\Pi_{\mathds{1}^{\perp}}\dd(X,R,X',\thh^*))^2\,|\,X=x]}$,
    and the range of admissible $\gamma$
    becoming $\gamma\in[0,\frac1{1-g})$,
    where $\Pi_{\mathds{1}^{\perp}}$ is the projection
    over $\mathds{1}^{\perp}$ that is orthogonal with respect
    to the $L^2(m)$-scalar product.
\end{theorem}
Observe that the latter result allows for
$\gamma=1$, and even slightly beyond.
Moreover, for $\gamma\leq1$, it yields
the first convergence rate independent of $\gamma$
for a variant of TD(0) with general mixing Markov Chains
up to our knowledge.
Indeed, \cite{liu2021temporal} claim to obtain
a slow rate of order $O(\frac1{\sqrt{k}})$ independent
of the choice of $\gamma<1$,
but it is in fact only of order
$O((1-\gamma)^{-2}k^{-\frac12})$ since it depends
on $|\theta^*|^2$ which is of order $(1-\gamma)^{-2}$ in general,
as proven in Lemma \ref{lem:theta*_to_infty}.

Let us mention that most TD algorithms (e.g., $TD(\lambda)$)
admit PC variants which benefit from properties
similar to Proposition \ref{prop:PCTD0},
allowing one to extend the results
of the original versions to the variants 
without further analysis.
We believe that in most cases, these variants
benefit from better convergence behaviors 
than their original versions,
as exemplified by PCTD(0) using Theorem \ref{thm:PCTD0}.
For Markovian sampling schemes (instead of i.i.d.),
such extensions apply but they require access to
two independent trajectories of the Markov Chain,
which might be considered a restrictive assumption
for some specific applications.

\section{RESULTS ON LSA}
\label{sec:LSA}
As is customary in the literature establishing
convergence rates of TD(0) with LFA and i.i.d.~samples,
we will consider the larger class of
\emph{linear stochastic approximation} methods.
To make a clear distinction from the framework of TD(0)
introduced in the previous section, we are going
to use different notations.

The goal here is to approximate
$y^*\in\RR^d$ satisfying
\begin{equation*}
    Hy^*
    =
    b,
\end{equation*}
for some invertible matrix $H\in\RR^{d\times d}$
and vector $b\in\RR^d$.
To do so, we are going to use
the iterative method described by,
for some initial state $y_0$,
for any $k\geq1$,
\begin{equation*}
        y_{k}
        =
        y_{k-1}
        -\alpha (h_ky_{k-1}-b_k),
\end{equation*}
where $\alpha>0$ is the learning step,
$(h_k,b_k)$ are couples of random variables in $\RR^{d\times d}\times\RR^d$.
Let us assume
\begin{enumerate}[label={\bf LS\arabic*}]
    \item
    \label{hypo:SPD}
        The symmetric part of $H$, denoted $S:=\frac{H+H^{\top}}2$,
        satisfies
        $S\geq \mu I_d$
        for some $\mu>0$.
    \item
        \label{hypo:iid_LSA}
        $(h_k,b_k)_{k\geq1}$ is i.i.d.
        with $\EE[h_k]=H$ and $\EE[b_k]=b$.
    \item 
        \label{hypo:ineq_LSA}
        There exist $\Sigma\in\RR^{d\times d}$
        symmetric definite positive
        and constants $c_{\Sigma},\beta,\beta_1,\ss^2$ such that
        $\Sigma\leq c_{\Sigma}S$ and
        \begin{align*}
            &\EE\lb h_kh_k^{\top}\rb\leq \beta \Sigma,
            \\
            &\EE\lb(h_ky^*-b_k)(h_ky^*-b_k)^{\top}\rb\leq \ss^2\Sigma,
            \\
            &\EE\lb h_k^{\top}h_k\rb\leq \beta_1S.
        \end{align*}
\end{enumerate}
Observe that we can always
take $\Sigma=S$ and $c_{\Sigma}=1$ in Assumption \ref{hypo:ineq_LSA};
but, to obtain our sharper rates on TD(0),
we will need to take $\Sigma$ and $S$ distinct.

We first state a proposition that will be useful
to prove our main results on TD(0)
as well as for our more general result on LSA, Theorem~\ref{thm:LSA} below.
\begin{proposition}
    \label{prop:LSA}
    Under Assumptions \ref{hypo:SPD}-\ref{hypo:ineq_LSA},
    for $0<\alpha<\min(2c_{\Sigma}^{-1}\beta^{-1},2\beta_1^{-1})$,
    we have
    \begin{multline*}
        \EE\lb(\yo_k^{\top}-y^*)^{\top}\Sigma(\yo_k-y^*)\rb
        \\
        \leq
        \lp\lp
        \frac{c_{\Sigma}|y_0-y^*|^2}{\alpha(2-\alpha\beta_1) k}
        \rp^{\frac12}
        +\lp\frac{2\ss^2\cS_k}{(2-\alpha\beta c_{\Sigma})k}\rp^{\frac12}\rp^2.
    \end{multline*}
    and $\mathcal{S}_k\leq c_{\Sigma}^2(3d+\cSt_k)$,
    where 
    $\mathcal{S}_k$ and $\cSt_k$ are defined by
    \begin{align*}
        \mathcal{S}_k
        &=
        \frac1k\sum_{j=0}^{k-1}
        \norm[2]{\lp I_d-(I_d-\alpha\Sigma^{\frac12}H\Sigma^{-\frac12})^j\rp
        \Sigma^{\frac12}H^{-1}\Sigma^{\frac12}}{F}
        \\
        \widetilde{\mathcal{S}}_k
        &=
        \frac1k\sum_{j=0}^{k-1}
        \norm[2]{\lp I_d-\alpha\Sigma^{\frac12}H\Sigma^{-\frac12}\rp^j}{F}.
    \end{align*}
\end{proposition}
Our main convergence result on LSA methods is:
\begin{theorem}
    \label{thm:LSA}
    Under Assumptions \ref{hypo:SPD}-\ref{hypo:ineq_LSA},
    for $0<\alpha<\min(2c_{\Sigma}^{-1}\beta^{-1},2\beta_1^{-1})$,
    we have
    \begin{multline*}
        \EE\lb(\yo_k^{\top}-y^*)^{\top}\Sigma(\yo_k-y^*)\rb
        \leq
        \frac{2c_{\Sigma}|y_0-y^*|^2}{(2-\alpha\beta_1)\alpha k}
        \\
        +\frac{4\ss^2c_{\Sigma}^2d}{k(2-\alpha\beta c_{\Sigma})}
        \lp3+\frac{1}{\alpha(2-\alpha\beta_1)\mu k}\rp.
    \end{multline*}
\end{theorem}
Let us mention that this theorem
is more general than Theorem \ref{thm:TD0}
in the sense that it holds for a larger class of methods.
However, its convergence
rate is weaker:
only its leading term with respect to $k$ is
model-independent,
and there is an additional error term that is model-dependent
but admits a faster decay in $k^{-2}$.
In the end, the rate is of order $O(k^{-1}+\mu^{-1}k^{-2})$
where $\mu$ is typically the smallest eigenvalue of $S$.
Up to our knowledge, the fastest convergence rates for LSA methods
proved in the literature of TD(0) with similar assumptions
are of order $O(k^{-1}\mu^{-2})$
(see \citet{lakshminarayanan2018linear, samsonov2024improved}),
which makes our rate faster.

In the appendix, we present two extensions
of this theorem.
The first one, Theorem \ref{thm:LSA_alternative},
does not require the third inequality in Assumption \ref{hypo:ineq_LSA},
applies to learning rates $\alpha<\frac{2}{c_{\Sigma}\beta}$
and admits a $O(k^{-1}+\mu^{-1}k^{-2})$ convergence rate,
similar to that of Theorem \ref{thm:LSA}.
The second extension, Theorem \ref{thm:LSA_alternative2},
allows one to take $\alpha<\frac{2}{\beta_1}$
and admits a slower convergence rate of $O(k^{-1}\mu^{-1})$.
Up to our knowledge, even this latter rate is
faster than any existing one
prior to this work.

\section{THE SPECIFIC CASE OF TD(0)}
\label{sec:insights}
\subsection{A first Kreiss-like convergence rate}
Given Proposition \ref{prop:LSA},
the remainder of the proof of Theorem \ref{thm:TD0}
relies on arguments from complex analysis
and is highly technical.
However, it is possible to get a robust
$O(k^{-1})$ rate much more easily.
This is what we will show in this subsection,
in order to illustrate the type of arguments we used
to prove Theorem~\ref{thm:TD0}.
More precisely, this simpler result mainly relies on
a classical result
from the literature
known as Kreiss' Theorem,
that we state here in the same formulation
as the one from \citet{trefethen2005spectra}.
\begin{theorem}[Kreiss Theorem]
    \label{thm:Kreiss}
    Let $Q\in\RR^{d\times d}$ be a matrix
    such that
    $\sup_{k\geq1}\norm{Q^k}{\rm op}<\infty$,
    then we have
    \begin{equation*}
        \mathcal{K}(Q)
        \leq
        \sup_{k\geq1}\norm{Q^k}{\rm op}
        \leq
        d\,e\,\mathcal{K}(Q),
    \end{equation*}
    where $\mathcal{K}(Q)$ is the Kreiss constant
    defined by 
    \begin{equation*}
        \mathcal{K}(Q)
        =
        \sup_{|z|>1}(|z|-1)\norm{(Q-z)^{-1}}{\rm op}.
    \end{equation*}
\end{theorem}
In general, the Kreiss constant may depend on
the spectrum of $Q$, so the latter theorem
cannot be used to directly obtain a
robust rate on LSA methods.
However, in the case of TD(0),
we will be able to bound the Kreiss constant of
$Q_0:=(I_d-\alpha \Sigma_0^{\frac12}H_0\Sigma_0^{-\frac12})$
uniformly with respect to the linear model.
To do so, we will use the specific structure of $H_0$,
that is 
$    H_0
    =
    \Sigma_0^{\frac12}(I_d-\gamma U)\Sigma_0^{\frac12}$,
with $\norm{U}{\rm op}\leq1$.
For $z\in\CC$ with $|z|>1$, we then get
\begin{align*}
    (z-Q_0)^{-1}
    =&
    \lp(z-1)I_d+\alpha\Sigma_0-\gamma\alpha\Sigma_0U\rp^{-1}
    \\
    =&
    \lp I_d-\gamma\alpha((z-1)I_d+\alpha\Sigma_0)^{-1}\Sigma_0U\rp^{-1}
    \\
    &\times\lp(z-1)I_d+\alpha\Sigma_0\rp^{-1}.
\end{align*}
On the one hand, using that $\Sigma_0$
is symmetric positive definite and the triangular
inequality on $\CC$, we have
\begin{align*}
    \norm{(z-1)I_d+\alpha\Sigma_0)^{-1}}{\rm op}
    &\leq
    \sup_{\lambda\in{\rm Sp}(\Sigma_0)}\frac{1}{|z-1+\alpha\lambda|}
    \\
    &\leq
    \sup_{\lambda\in{\rm Sp}(\Sigma_0)}\frac{1}{|z|-(1-\alpha\lambda)}
    \\
    &\leq
    \frac{1}{|z|-1}.
\end{align*}
On the second hand,
observe that
\begin{align*}
    &\norm{\alpha((z-1)I_d+\alpha\Sigma_0)^{-1}\Sigma_0U}{\rm op}
    \\
    &\hspace*{3cm}\leq
    \norm{\alpha((z-1)I_d+\alpha\Sigma_0)^{-1}\Sigma_0}{\rm op}
    \\
    &\hspace*{3cm}=
    \sup_{\lambda\in{\rm Sp}(\Sigma_0)}\frac{\alpha\lambda}{|z-1+\alpha\lambda|}
    \\
    &\hspace*{3cm}\leq
    \sup_{\lambda\in{\rm Sp}(\Sigma_0)}\frac{\alpha\lambda}{|z|-(1-\alpha\lambda)}
    \leq 1,
\end{align*}
which implies that
\begin{equation*}
    \norm{\lp I_d-\gamma\alpha((z-1)I_d+\alpha\Sigma_0)^{-1}\Sigma_0U\rp^{-1}}{\rm op}
    \leq
    \frac1{1-\gamma}.
\end{equation*}
From the definition of the Kreiss constant
and the above inequalities,
we get
\begin{equation}
    \label{eq:kreiss_constant}
    \mathcal{K}(Q_0)
    \leq
    \frac1{1-\gamma}
\end{equation}
Assuming that Proposition \ref{prop:LSA} applies
to TD(0) with $\Sigma=\Sigma_0$,
$\beta=4$,
$c_{\Sigma}=\frac1{1-\gamma}$,
$\beta_1=2$,
and $\ss^2=\ss_0^2$,
we obtain 
$\cSt_k\leq de\mathcal{K}(Q_0)
\leq\frac{de}{1-\gamma}$
and the following robust $O(1/k)$-convergent rate:
\begin{theorem}
    \label{thm:simple_kreiss}
    Under Assumptions \ref{hypo:iid}-\ref{hypo:linear},
    for $0<\alpha\leq\frac{1-\gamma}8$,
    for $k\geq1$, we have
    \begin{multline*}
        \EE_{X\sim m}\lb|v(X,\tho_k)-v(X,\theta^*)|^2\rb
        \\
        \leq
        \frac{2|\theta_0-\theta^*|^2}{(1-\gamma)\alpha k}
        +\frac{4\ss_0^2d}{(1-\gamma)^2k}
        \lp 1+
        \frac{d^2e^2}{(1-\gamma)^2}\rp.
    \end{multline*}
\end{theorem}
We do not consider this theorem one of our main results,
since its convergence rate is loose
compared to the one of Theorem \ref{thm:TD0}.
We only presented it to show a simple proof of
a first model-independent $O(k^{-1})$-rate.

\subsection{The additional argument for Theorem~\ref{thm:TD0}}
One may prove that our upper bound on the Kreiss constant
\eqref{eq:kreiss_constant} is sharp on the set of admissible
matrices~$Q_0$ obtained from instances of TD(0).
Moreover, 
it is known, see \cite{trefethen2005spectra},
that the inequalities in Theorem \ref{thm:Kreiss}
are sharp for general power-bounded matrices~$Q$.
So one may ask:
how can the constants in the rate
of Theorem~\ref{thm:simple_kreiss} be improved
to get the one in Theorem~\ref{thm:TD0}?
The answer lies in the fact that 
the right inequality from Theorem~\ref{thm:Kreiss}
is not sharp over the class of admissible~$Q_0$.
In the end, using the particular structure of~$H_0$ 
once more,
we obtain the subsequent
proposition that is the last missing part
to prove Theorem~\ref{thm:TD0}.
\begin{proposition}
    \label{prop:Kreiss}
    Assume \ref{hypo:iid}-\ref{hypo:linear}.
    For $\ee_{k,\gamma}$ and $\cS_k$
    defined as in Theorem~\ref{thm:TD0}
    and Proposition~\ref{prop:LSA},
    for $0<\alpha<\alpha_0(\gamma)$,
    $k\geq1$,
    we have
        $\cS_k
        \leq
        \frac{de}{(1-\gamma)^2}
        \lp 1+ \ee_{k,\gamma}\rp.$
\end{proposition}
Theorem~\ref{thm:TD0} is then a direct consequence of
the latter proposition and Proposition~\ref{prop:LSA}
with $\Sigma=\Sigma_0$, $c_{\Sigma}=(1-\gamma)^{-1}$,
$\beta=(1+\gamma)^2$, $\beta_1=1+\gamma$ and $\ss^2=\ss_0^2$.

\section{NUMERICAL SIMULATIONS}
\label{sec:simulations}
\begin{figure*}
    \begin{center}
    \begin{subfigure}{0.31\textwidth}
        \includegraphics[width=\textwidth]{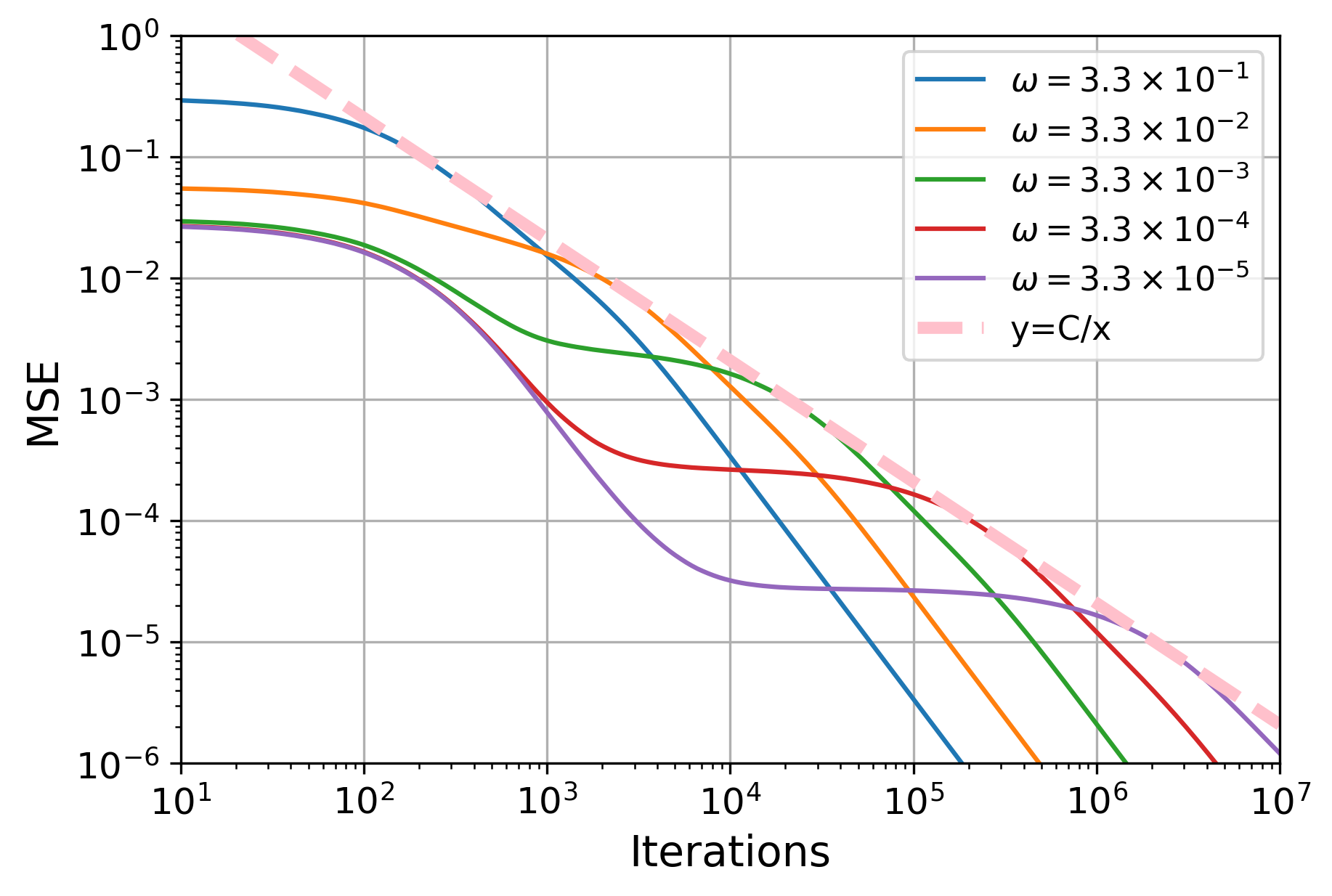}
        \caption{Bias}
        \label{subfib:bias}
    \end{subfigure}
    \hspace*{0.01\textwidth}
    \begin{subfigure}{0.31\textwidth}
        \includegraphics[width=\textwidth]{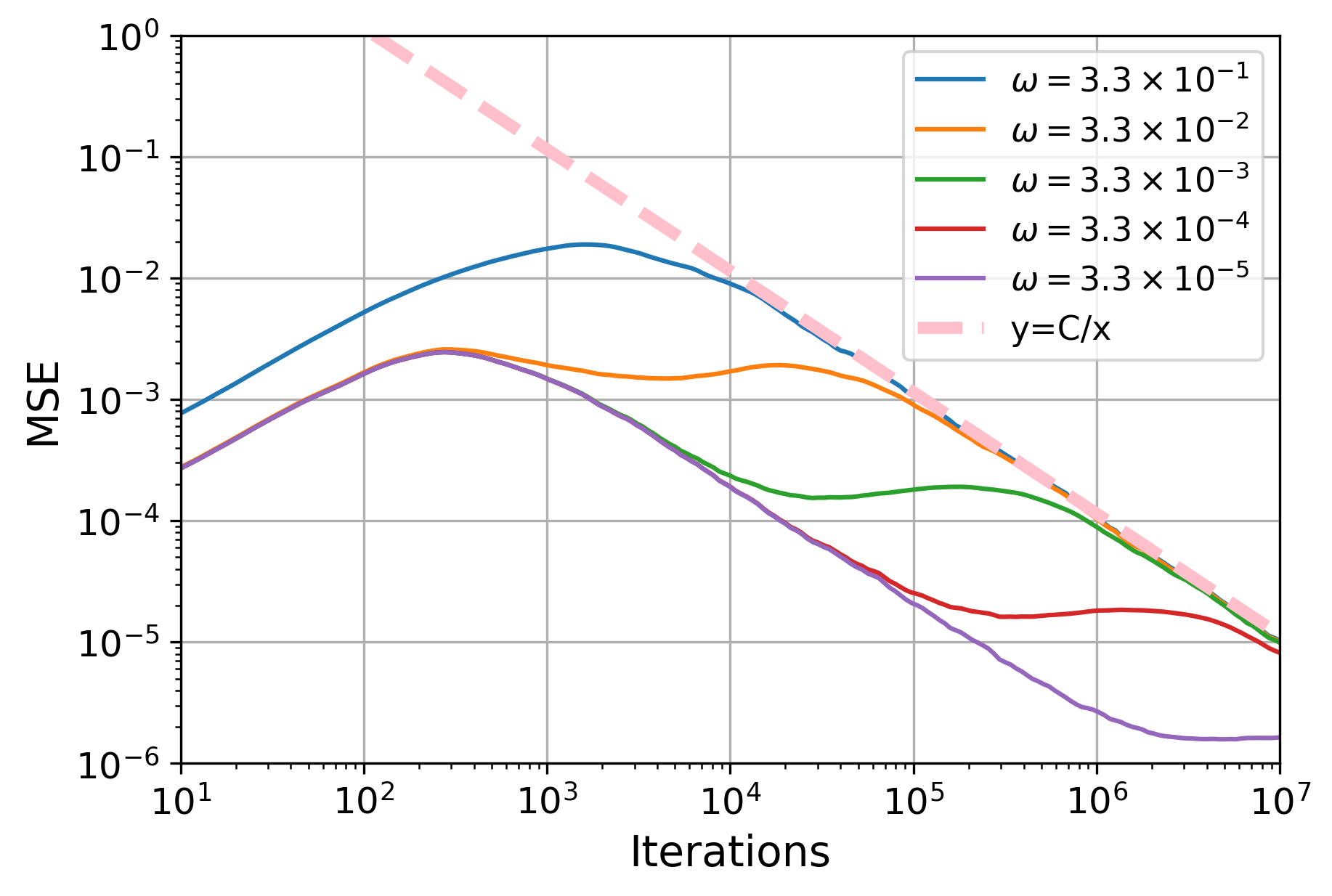}
        \caption{Variance}
        \label{subfib:variance}
    \end{subfigure}
    \hspace*{0.01\textwidth}
    \begin{subfigure}{0.31\textwidth}
        \includegraphics[width=\textwidth]{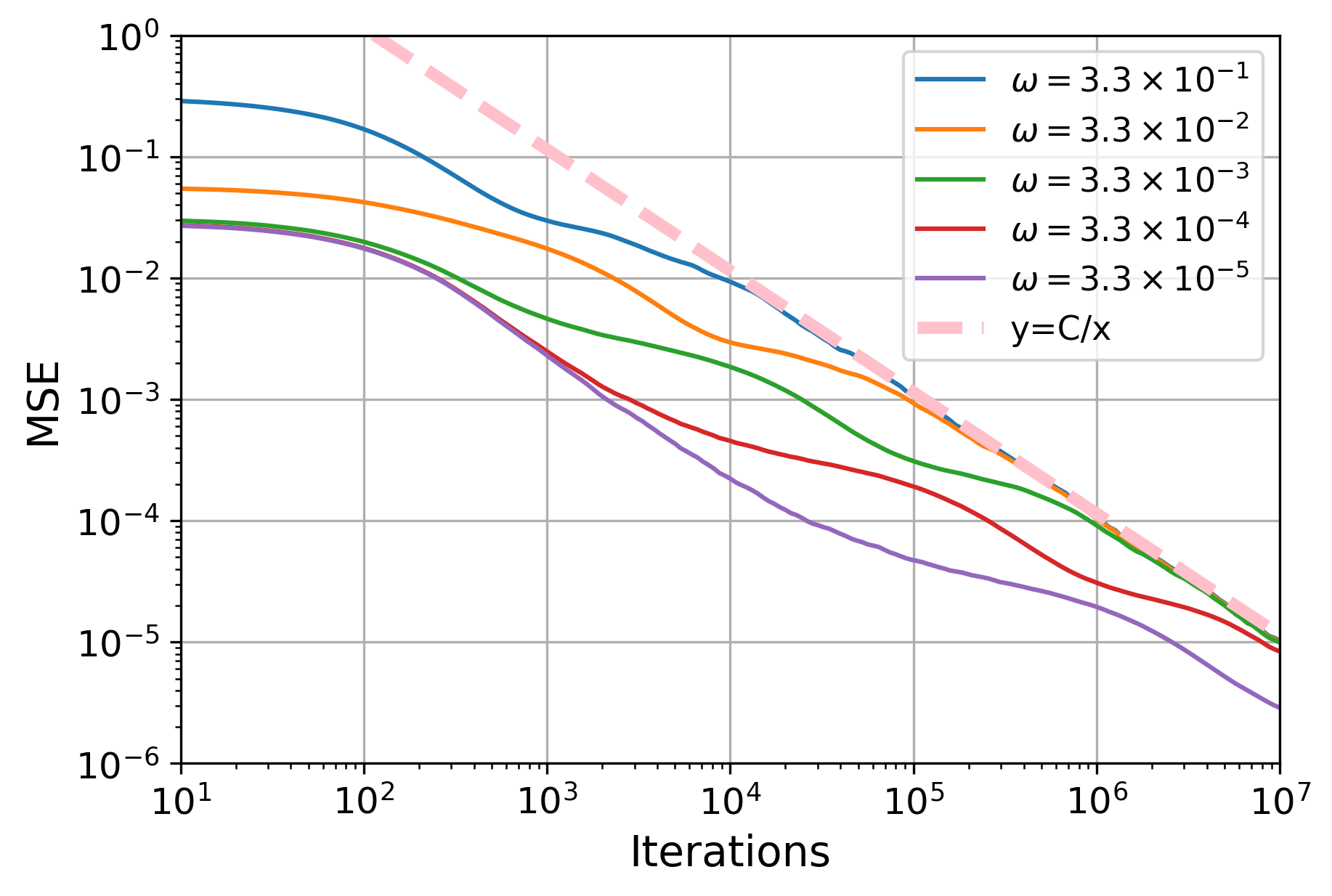}
        \caption{Total error}
        \label{subfib:bias+variance}
    \end{subfigure}
    \end{center}
    \caption{\label{fig:d=3}
    We take $d=3$ with $\lambda_1=1$ and $\lambda_2=\lambda_3=10^{-j/2}$ for $0\leq j\leq 4$,
    so that $\omega\in\{\frac13\times 10^{-j}, \, 0\leq j \leq 4\}$.
    We use $\gamma=0.9$ and $\alpha=\frac{1-\gamma}4=0.025$.
    We draw the bias on the left (taking $\ss_R^2=0$),
    the variance in the middle (taking $\theta_0=\theta^*=0$ and $\ss_R^2=1$)
    and the total MSE error on the right.
    The curves are averaged over 10 simulations for \ref{subfib:bias}
    and 1000 simulations for \ref{subfib:variance} and \ref{subfib:bias+variance},
    always with the same transition matrix,
    initial parameter and matrices $P,Q$
    (that have been randomly drawn with $\alpha_w=1$).
    In any case, the dashed pink line is the smallest curve of the
    form $y=C/x$ that upper bounds all curves from the same figure.
    }
\end{figure*}

The code to reproduce all experiments is available at \url{https://github.com/ziadkobeissi/Robust_and_Fast_Convergence_TD0.git}.

\paragraph{Convergence independent of $\omega$}
We consider $d=3$ and draw a random MRP on $\cX=\{1,2,3\}$
with $\omega=\frac13\times 10^{-j}$ for $0\leq j\leq 4$
as described in Appendix \ref{subsec:MRP_controlled_EV}.
Figure \ref{fig:d=3} shows the convergence of 
the bias, the variance and the total MSE error on the values.

Focusing first on the bias, we observe that each curve exhibits a threshold, 
seemingly proportional to $\omega^{-1}$, after which the bias decays quadratically in $k$.
Before this threshold, the bias stays below a bound of the form $y=C/x$, with $C$ apparently independent of $\omega$.

The variance term first rises, as too little noise has been seen so far, then
falls. It also seems to alternate between two envelopes of the form $y=C/x$ at an
iteration count proportional to~$\omega^{-1}$.

For the total error (which is upper bounded by twice the sum of the bias and the variance terms),
the bias dominates at the beginning; once it enters its quadratic decay regime, the variance eventually becomes dominant.
The crossover again occurs after a number of iterations proportional to $\omega^{-1}$.
Finally, all curves remain bounded by a function $C/x$ with $C$ independent of $\omega$, in agreement with Theorem~\ref{thm:TD0}.

\paragraph{Standard TD(0) versus TD(0) with two learning rates.}
Figure \ref{subfib:td0_vs2lr} compares the standard TD(0) with the variant using
two different learning rates described in
Subsection \ref{subsec:reduce_gamma_dep}. The details of the model are given in
Subsection \ref{subsec:fig_d=100}; we take $\gamma=0.99$, $d=100$ and
$|\cX|=1000$. We use $B=1$ since our theoretical results for the two-learning-rate variant
only apply to a minibatch size of one. All curves are
averaged over $100$ runs. We observe that the standard TD(0) performs
significantly worse than the two-learning-rate variant, which is consistent with
our theoretical findings.

\paragraph{PCTD(0) versus TD(0).}
Figure \ref{subfib:td0_vspctd0} compares: (i) the standard TD(0), (ii) TD(0) with
two learning rates (Subsection \ref{subsec:reduce_gamma_dep}), and (iii)
PCTD(0), on the same model with minibatches of size $B=100$. The larger batch
reduces both the iterations needed to reach the asymptotic regime and the
variance of the iterates, so that $10$ runs already yield smooth curves (against
$100$ in the previous figure). Since the error of PCTD(0) is only reduced up to
an additive constant, for this method we consider the mean-square value-function
error projected onto the orthogonal complement of the constant functions (this
matches the error obtained when $\vp$ is replaced by $\vpt$, see
Subsection \ref{subsec:PCTD(0)}); for TD(0), we report both this metric and the
standard MSE. We observe that the standard TD(0) performs significantly worse
than the two other methods, consistent with our theoretical findings.

\paragraph{PCTD(0) with $\gamma=1$ or beyond.}
Using the same Markov decision process as in the previous paragraph, described in
Subsection \ref{subsec:fig_d=100}, we compute the spectral gap, which is
approximately $g=0.145$. By Theorem \ref{thm:PCTD0}, this allows us to consider
values of $\gamma$ larger than one, while still satisfying
$\gamma\in[0,\frac{1}{1-g})$. We compare four instances of PCTD(0), each with a
different $\gamma\in\{0.9,0.99,1,1.15\}$. As predicted by
Theorem \ref{thm:PCTD0}, the iterates converge even for $\gamma\geq1$, and
PCTD(0) appears to be robust with respect to changes in $\gamma$.

\begin{figure*}
    \begin{center}
    \begin{subfigure}{0.31\textwidth}
        \includegraphics[width=\textwidth]{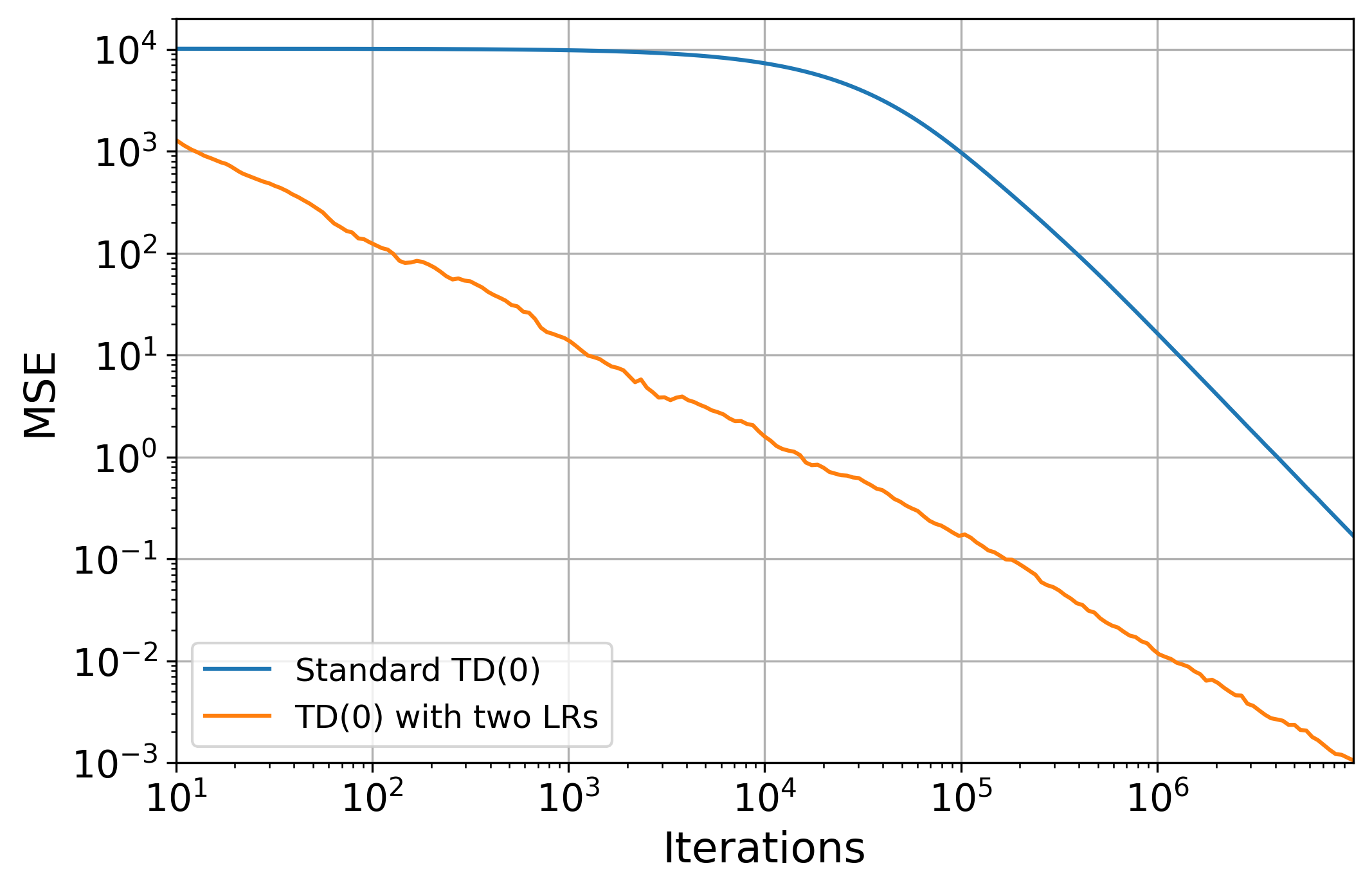}
        \caption{TD(0) with one or two LRs}
        \label{subfib:td0_vs2lr}
    \end{subfigure}
    \hspace*{0.01\textwidth}
    \begin{subfigure}{0.31\textwidth}
        \includegraphics[width=\textwidth]{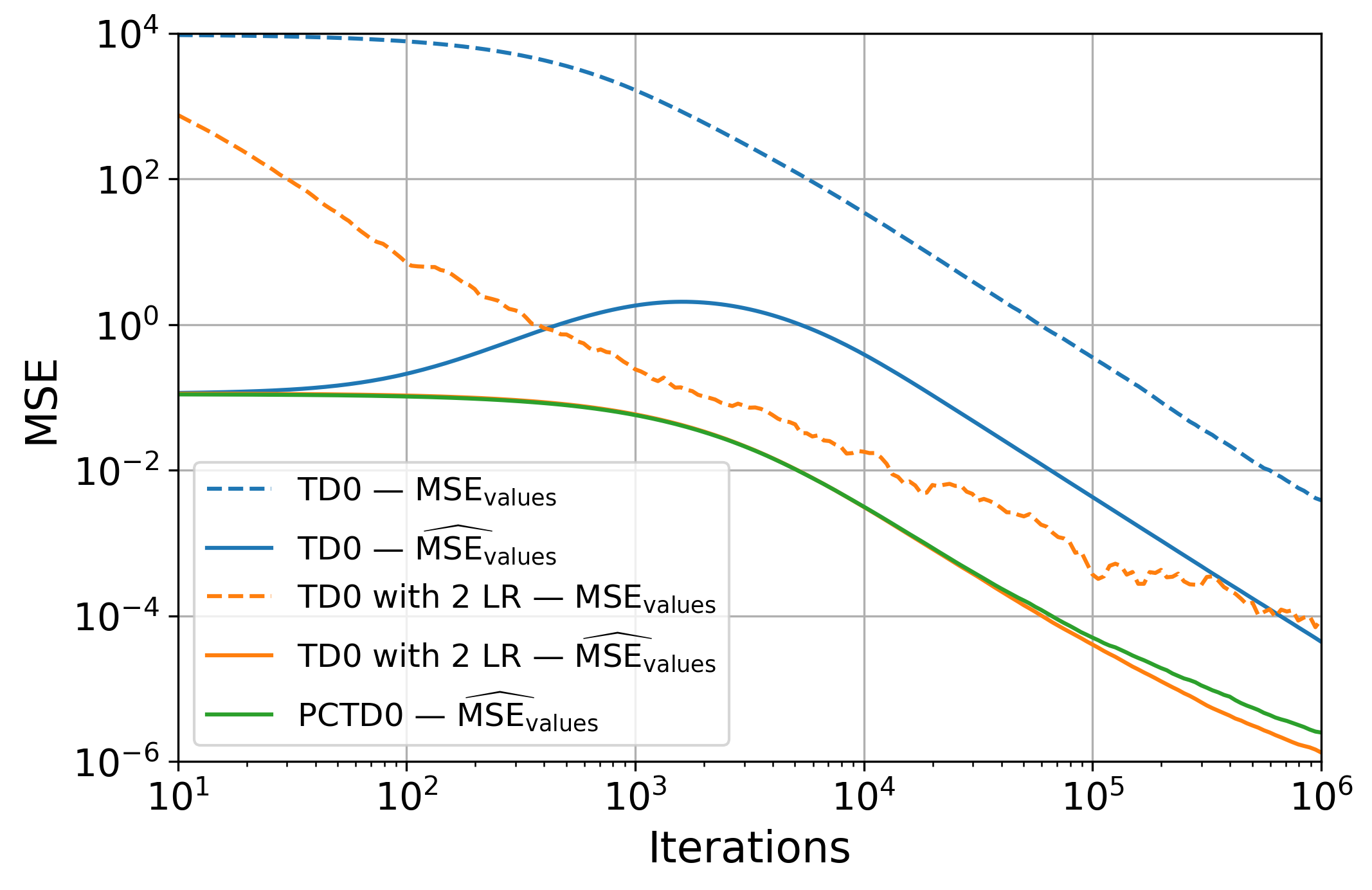}
        \caption{TD(0) vs. PCTD(0)}
        \label{subfib:td0_vspctd0}
    \end{subfigure}
    \hspace*{0.01\textwidth}
    \begin{subfigure}{0.31\textwidth}
        \includegraphics[width=\textwidth]{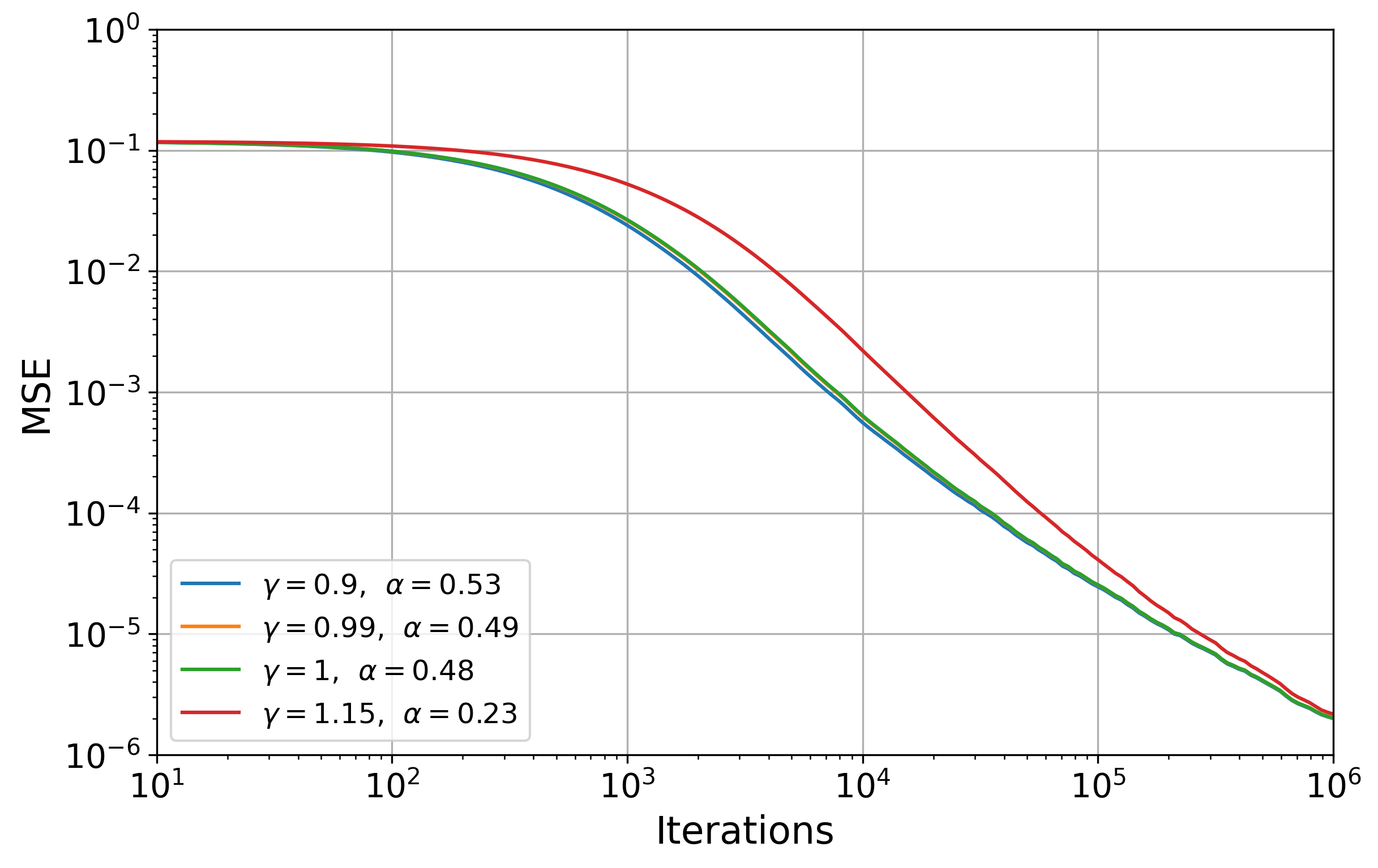}
        \caption{PCTD(0) with different $\gamma$}
        \label{subfib:pctd0s}
    \end{subfigure}
    \end{center}
    \caption{\label{fig:d=100}
We take $d=100$ and $\cX=\{1,\dots,1000\}$. The transition matrix is drawn at
random: each column is a sample from a Dirichlet distribution with parameter
$2/|\cX|=2\times 10^{-3}$. On the left, using $\gamma=0.99$ and $B=1$, we compare
the standard TD(0) with an instance of TD(0) using two learning rates, following
the theoretical results of Theorems \ref{thm:TD0} and
\ref{thm:TD0_param_trick}. In the center, using $\gamma=0.99$ and $B=100$, we
compare PCTD(0), under the setting of Theorem \ref{thm:PCTD0}, with the standard
TD(0) (satisfying the assumptions of Theorem \ref{thm:TD0}) and with TD(0) using
two learning rates (the latter does not fall under the setting of
Subsection \ref{subsec:reduce_gamma_dep}, which does not cover the larger values
of $\alpha$ permitted by the use of minibatches). The dashed lines correspond to
the MSE of the value functions, while the solid lines correspond to the mean
square of the value-function error projected onto the orthogonal complement of
the constant functions (this matches the error obtained when $\vp$ is replaced by
$\vpt$, as described in Subsection \ref{subsec:PCTD(0)}). On the right, we
compare different instances of PCTD(0) with $\gamma\in\{0.9,0.99,1,1.15\}$ and
$\alpha(\gamma)=\frac{\alpha_B((1-g)\gamma)}2$, where $g\approx 0.145$. The curves
on the left are averaged over $100$ simulations, the others over only $10$.
    }
\end{figure*}

\section{CONCLUSIONS}
\label{sec:conclusions}
In this paper, we derived a robust convergence rate of order $1/k$ for the
TD(0) algorithm with i.i.d.~samples, using a universal learning rate. The key
ingredient behind these results is a set of new Kreiss-like estimates tailored
to the particular structure of TD(0). These estimates describe the behavior of
dynamical systems obtained by iterating a linear system whose associated matrix
is non-symmetric. We also established a lower bound on the convergence of TD(0)
that matches our rate up to a multiplicative constant smaller than $11$. Since
this lower bound is obtained on simple examples, we believe that our analysis
accurately captures the actual asymptotic behavior of TD(0), a belief that is
further supported by our numerical simulations.

As a consequence, we showed through theoretical arguments that the main
bottleneck in the convergence of TD(0) is in fact the smallness of $(1-\gamma)$.
Without further assumptions, in the regime of our main result,
Theorem \ref{thm:TD0}, the bias part of the error is of order
$O((1-\gamma)^{-4}k^{-1})$, while the variance part is only of order
$O((1-\gamma)^{-2}k^{-1})$. Using minibatches then allows one to divide both
parts of the error by a factor proportional to the batch size; remarkably, this
applies not only to the variance but also to the bias, which is non-standard
(see Theorem \ref{thm:minibatch_TD0}). Moreover, when the Markov chain is
strongly mixing, we propose two methods to reduce the dependence on
$(1-\gamma)^{-1}$.
The second of these is particularly promising: it consists in
introducing a variant of TD(0), which we call PCTD(0), that can be computed for
$\gamma=1$ and even slightly beyond. This variant enjoys a robust and fast
convergence rate whose dependence is governed by $(1-(1-g)\gamma)^{-1}$ rather
than the usual $(1-\gamma)^{-1}$, where $g>0$ denotes the spectral gap, which is
unknown a priori and may be small in practice.

Across all the convergence regimes presented here, we believe that the fastest
in practice is PCTD(0) with minibatch size $B=\lceil c_B(1-\gamma)^{-1}\rceil$
and $c_B\in[0.1,1]$. This yields a convergence rate of order
$O((1-\gamma)^{-1}k^{-1})$,
which substantially outperforms all previously
existing rates.

Alongside these main results, we developed related bounds for
general LSA schemes and obtained a convergence rate of order $1/k$ with a
model-independent leading term, the model dependence being confined to a
faster-decaying term of order $1/k^2$.

While these results significantly improve our understanding of the behavior of
TD(0), several open questions remain:
\begin{enumerate}
    \item Defining $\alpha_{\rm rob}(\gamma)>0$ as the supremum of learning rates that allow robust and fast convergence
    in the present regime,
    we proved that $\alpha_1(\gamma)\leq\alpha_{\rm rob}(\gamma)\leq\alpha_2(\gamma)$.
    Can we characterize $\alpha_{\rm rob}$ more precisely?
    \item Is it possible to adapt the main result to infinite-dimensional linear approximation, so that one could potentially combine optimal rates and universal approximation?
    \item Can the proof be extended to the Markovian sampling setup while
    preserving similar rates, affected only by mixing constants?
    Alternatively, can one identify fundamental limitations in the Markovian setting that would preclude such results?
    \item Do these proof techniques extend to other, more complex reinforcement
    learning algorithms beyond policy evaluation, such as SARSA, Q-learning, or
    actor-critic methods?
\end{enumerate}

\paragraph{Acknowledgement.} The first author has been partially financed
by French government under the France 2030 program,
reference ANR-11-IDEX-0003 within the OI H-Code.

\bibliographystyle{plainnat}
\bibliography{biblioML.bib}

\newpage


\clearpage
\appendix
\thispagestyle{empty}

\onecolumn
\section{Additional information on the numerical simulations}
All experiments were run on a standard personal laptop (Intel Core i7, 32 GB RAM) using only CPU.
Total computation time amounted to a few hours.
\subsection{Figure \ref{fig:d=3}: MRP with fixed eigenvalues}
\label{subsec:MRP_controlled_EV}
In this section, we describe how we create an MRP with the following properties:
\begin{enumerate}
    \item 
        we can fix the spectrum of $\Sigma_0$:
        it is particularly useful to illustrate the independence of the convergence
        with respect to $\omega$, the smallest eigenvalue of $\Sigma_0$.
    \item
        It is not reversible. Indeed, we believe that reversibility is both unrealistic and overly simplistic.
        For reversible MRPs, the theoretical analysis becomes significantly simpler.
        In this case, most of the matrices involved are symmetric and commute with one another.
        As a result, one can obtain results similar to those proved in the main body of this article with considerably less effort.
        In particular, much of our theoretical analysis becomes unnecessary, especially the part relying on complex analysis.
    \item 
        It is random.
\end{enumerate}

\paragraph{The Markov chain.}
We take $\{1,\dots,n\}$ as the state space
and $T\in\RR^{n\times n}$ as the transition matrix.
Recall that $T$ is stochastic by definition,
i.e., 
its coefficients are nonnegative and
the sum of any row equals one.
Here, we make the additional assumption that $T$
is bistochastic, i.e., any column sums to one as well.
This implies that the uniform distribution $m=\frac1n$
is invariant.
For $n\leq 2$, any bistochastic matrix is in fact symmetric
which implies that the Markov Chain is reversible.
For this reason, we always assume $n\geq3$ when working
with bistochastic matrices.

It is well-known that the set of bistochastic matrices
is exactly the convex hull of permutations matrices
and therefore is of dimension $(n-1)^2$.
Therefore, using Carathéodory's theorem,
we can randomly draw a bistochastic $T$
of the form $T={\sum_{i=1}^{(n-1)^2+1}w_ip_i}$,
where $(p_i)_i$ are permutation matrices and $w$ follows
a Dirichlet distribution ${\rm Dir}(\alpha_w,\dots,\alpha_w)$
for some $\alpha_w>0$.
The support of the latter random variable
is the whole set of bistochastic matrices.

\paragraph{The reward process.} We simply take
$(R_k)_{k\geq0}$ i.i.d. with $R_1\sim\cN(0,\ss_R^2)$.
In particular, the noise does not depend on the state
of the Markov Chain.
This implies that $\theta^*=0$.

\paragraph{The linear parametrization.}
We have to take $d\leq n$ in order for the linear independence
of the feature maps to be possible.
Since the number of states is finite,
the feature maps can be represented
using the matrix $\Phi:=(\vp_i(j))_{1\leq i\leq d,1\leq j\leq n}$.
Consider its singular value decomposition $\Phi=PD_{\vp}Q^{\top}$, 
where $P\in\RR^{d\times d},Q\in\RR^{n\times n}$ are orthogonal and
$D_{\vp}=\diag(\lambda_1,\dots,\lambda_d)\in\RR^{d\times n}$ with 
$1\geq\lambda_1\geq\dots\geq\lambda_d>0$.
Observe that, with the above assumption on the transition matrix,
we have 
\begin{equation*}
    \Sigma_0
    =
    \Phi D_m \Phi^{\top}
    =
    \frac1nPD_{\vp}D_{\vp}^{\top}P^{\top},
\end{equation*}
where $D_m=\frac1nI_n$ is the diagonal matrix
with the invariant distribution as diagonal coefficients.
Therefore, the eigenvalues of $\Sigma_0$ are given
by $\lambda_i^2/d$ for $1\leq i\leq d$.
Moreover,
the boundedness condition of Assumption \ref{hypo:linear} holds
since $|\vp(j)|^2=(QD_{\vp}^{\top}D_{\vp}Q^{\top})_{j,j}
\leq \lambda_1^2\leq1$.

To randomly draw a matrix $\Sigma_0$ with fixed spectrum
$(\frac{\lambda_i^2}n)_{1\leq i\leq d}$, it is then sufficient
to randomly draw $P,Q$.
We do so by computing the singular value decomposition
of a randomly drawn matrix in $\RR^{d\times n}$ with
i.i.d. coefficients distributed according to $\cU(-1,1)$.

\paragraph{The initial condition $\theta_0$.}
When we are only interested in the variance part of the error,
we take $\theta_0=0$.
Otherwise, each coordinate of $\theta_0$ is drawn uniformly over $[-1,1]$.

\paragraph{Figure \ref{fig:d=3} with $d=3$.}
We take $d=3$ because it is the smallest dimension for which
there exist bistochastic matrices $T$ that are not symmetric. 
We fix $\lambda_1=1$ and vary $\lambda_2=\lambda_3$
so that $\omega = \frac13\times 10^{-j}$ for $0\leq j\leq 4$.
Focusing first on the bias, we observe that each curve exhibits a threshold, 
seemingly proportional to $\omega^{-1}$, after which the bias decays quadratically in $k$.
Before this threshold, the bias stays below a bound of the form $y=C/x$, with $C$ apparently independent of $\omega$.

The variance term initially increases,
because too little noise has been observed so far, and then decreases.
It also appears to switch between two envelopes of the form $y=C/x$ at an iteration count proportional to $\omega^{-1}$.

For the total error (which is upper bounded by twice the sum of the bias and the variance terms),
the bias dominates at the beginning; once it enters its quadratic decay regime, the variance eventually becomes dominant.
The crossover again occurs after a number of iterations proportional to $\omega^{-1}$.
Finally, all curves remain bounded by a function $C/x$ with $C$ independent of $\omega$, in agreement with Theorem~\ref{thm:TD0}.

\subsection{Additional numerical simulations under similar settings}
In the first line of Figure \ref{fig:d=3_bis}, we take $\alpha=1$ so that 
Theorem \ref{thm:TD0_larger_alpha} implies that the convergence is fast
but does not ensure the robustness with respect to the smallness of $\omega$.
Yet, in our numerical simulations, we experience convergence speeds
that do not degrade when $\omega$ is small.

We observe that the bias term is much smaller than the one in Figure
\ref{fig:d=3}, but the variance and the total MSE are larger.
Therefore, for bounded $|\theta_0-\theta^*|$, the regime 
$\frac{1-\gamma}4$ may seem better at first sight.
However, this statement has to be nuanced with the fact that
$\theta^*$ can be of order $(1-\gamma)^{-1}$ in practice.

In the second line of Figure \ref{fig:d=3_bis},
we took $\gamma=0.99$ and observe approximately the same behavior
as in Figure \ref{fig:d=3} with a shift to the right.
Observe that the coefficient of the dashed pink line is 
multiplied by a factor $100$ in the three figures,
which is aligned with the upper bound from Theorem \ref{thm:TD0}.

\begin{figure*}
    \begin{center}
    \begin{subfigure}{0.31\textwidth}
        \includegraphics[width=\textwidth]{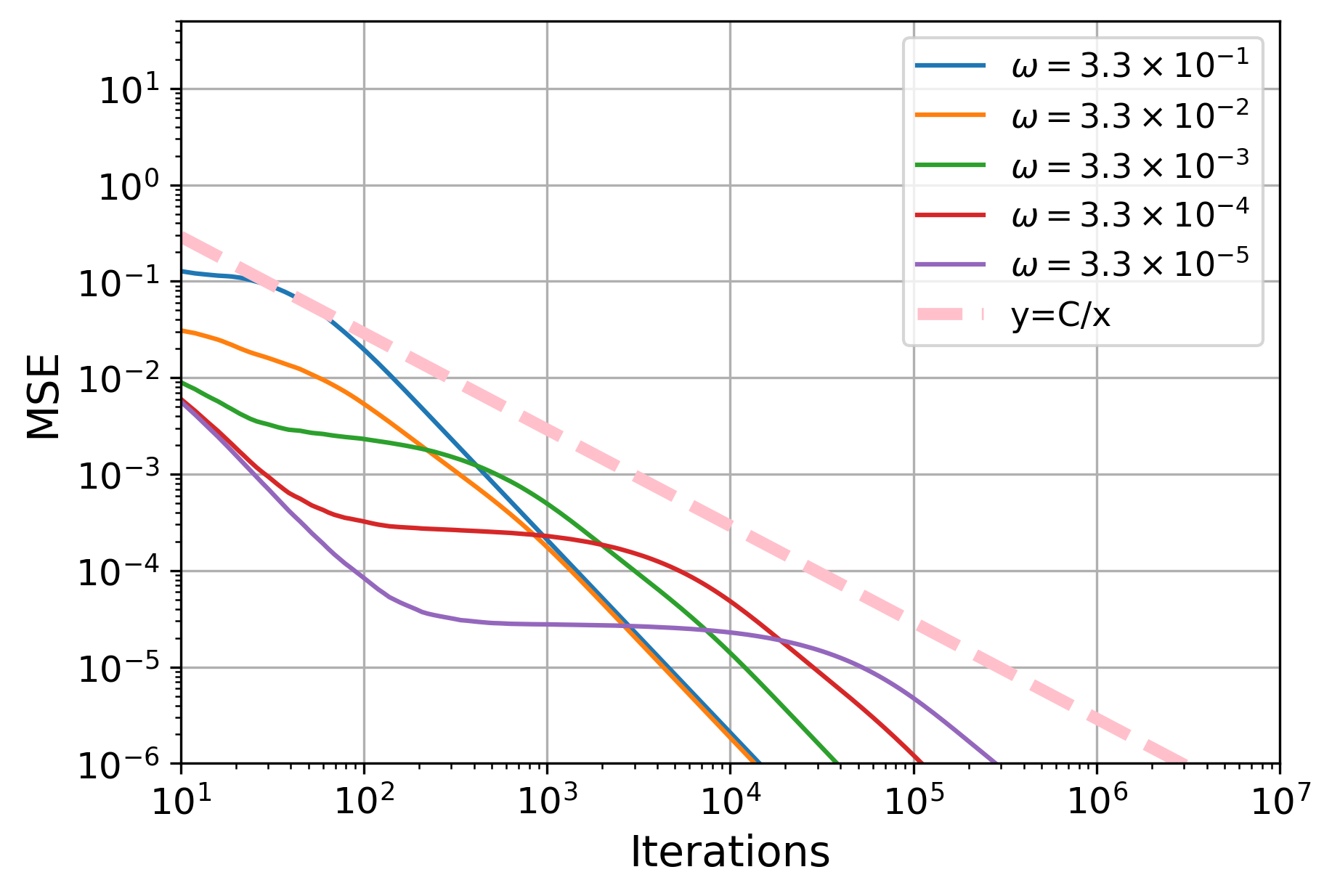}
        \caption{Bias}
    \end{subfigure}
    \hspace*{0.01\textwidth}
    \begin{subfigure}{0.31\textwidth}
        \includegraphics[width=\textwidth]{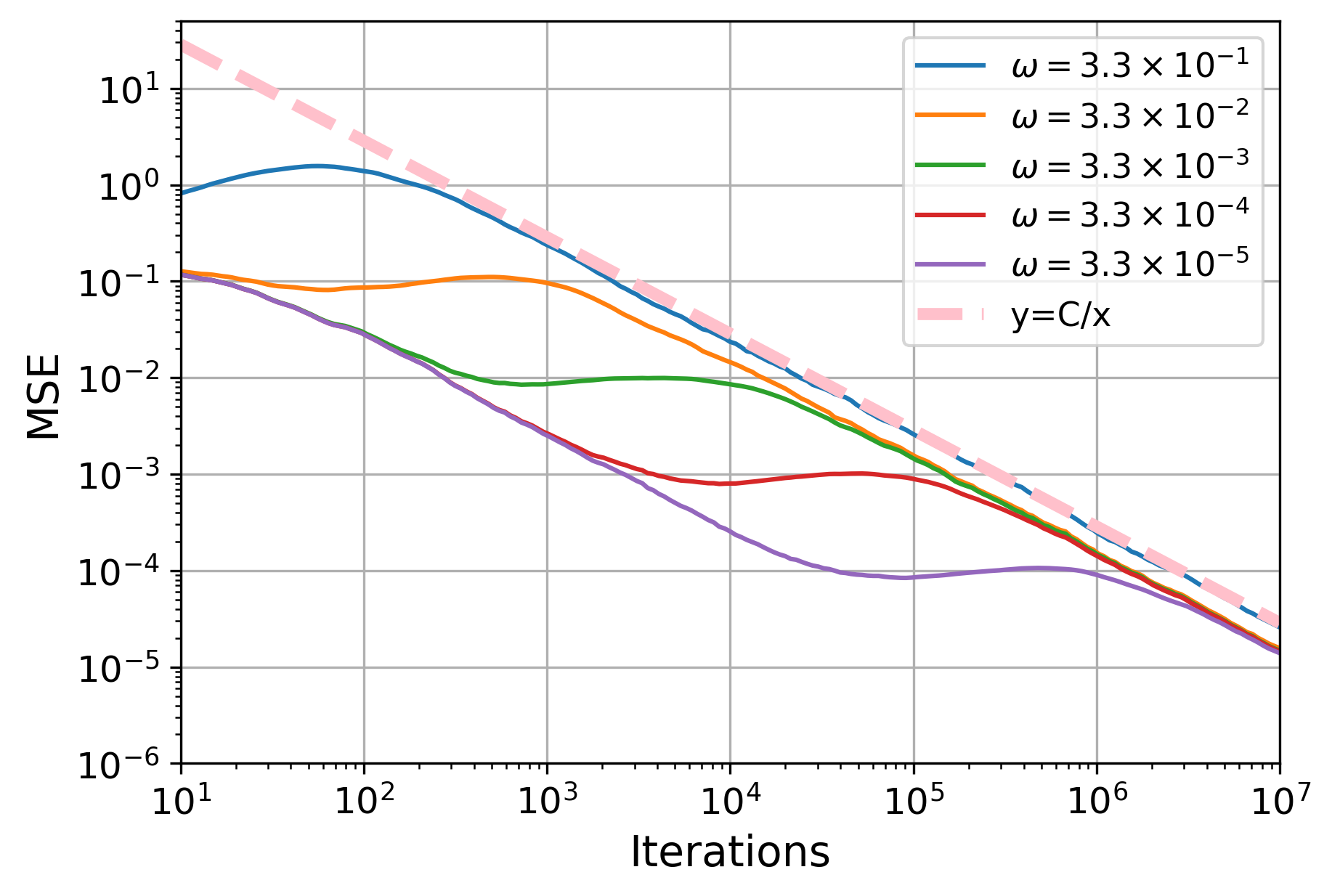}
        \caption{Variance}
    \end{subfigure}
    \hspace*{0.01\textwidth}
    \begin{subfigure}{0.31\textwidth}
        \includegraphics[width=\textwidth]{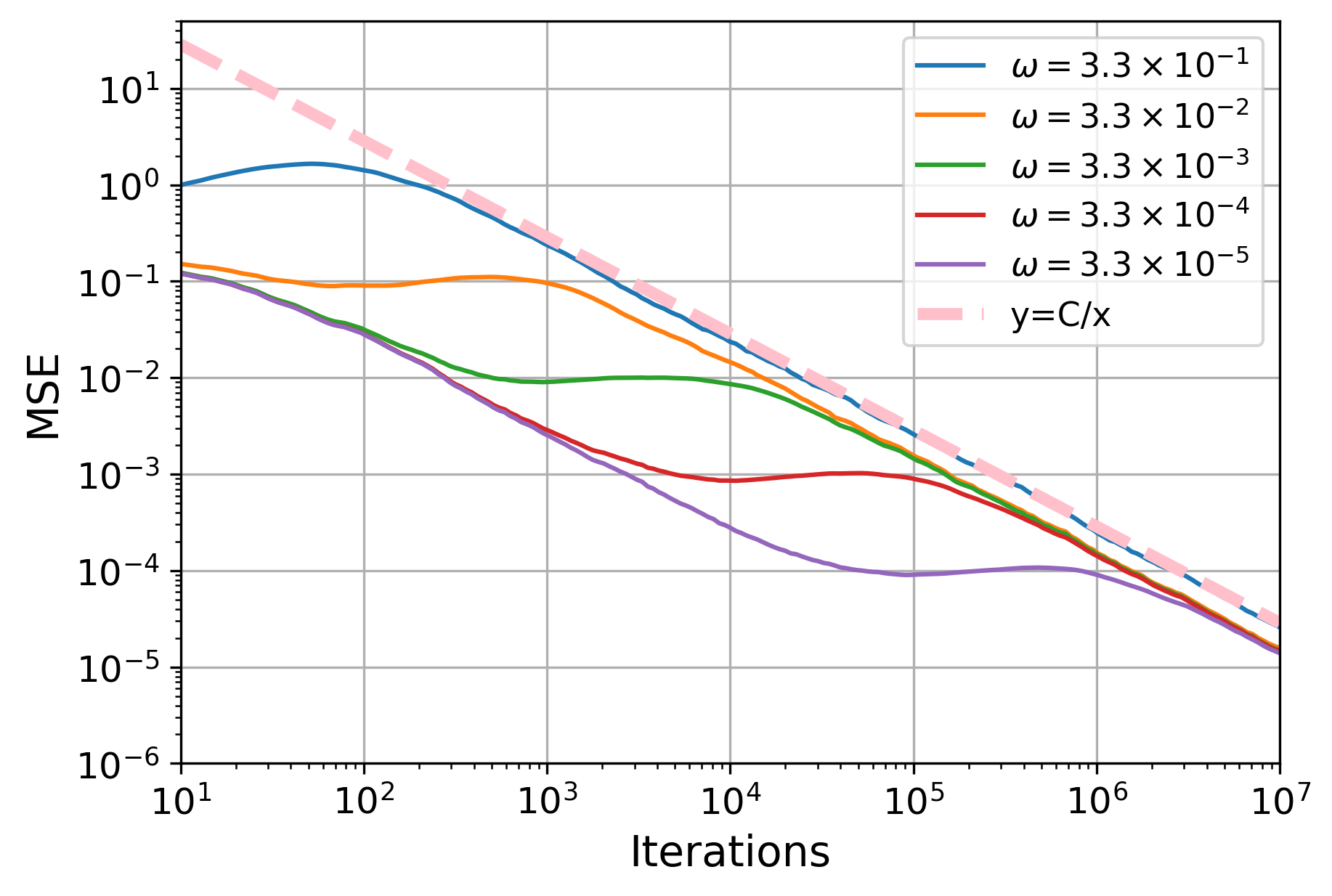}
        \caption{Total error}
    \end{subfigure}
    
    \begin{subfigure}{0.31\textwidth}
        \includegraphics[width=\textwidth]{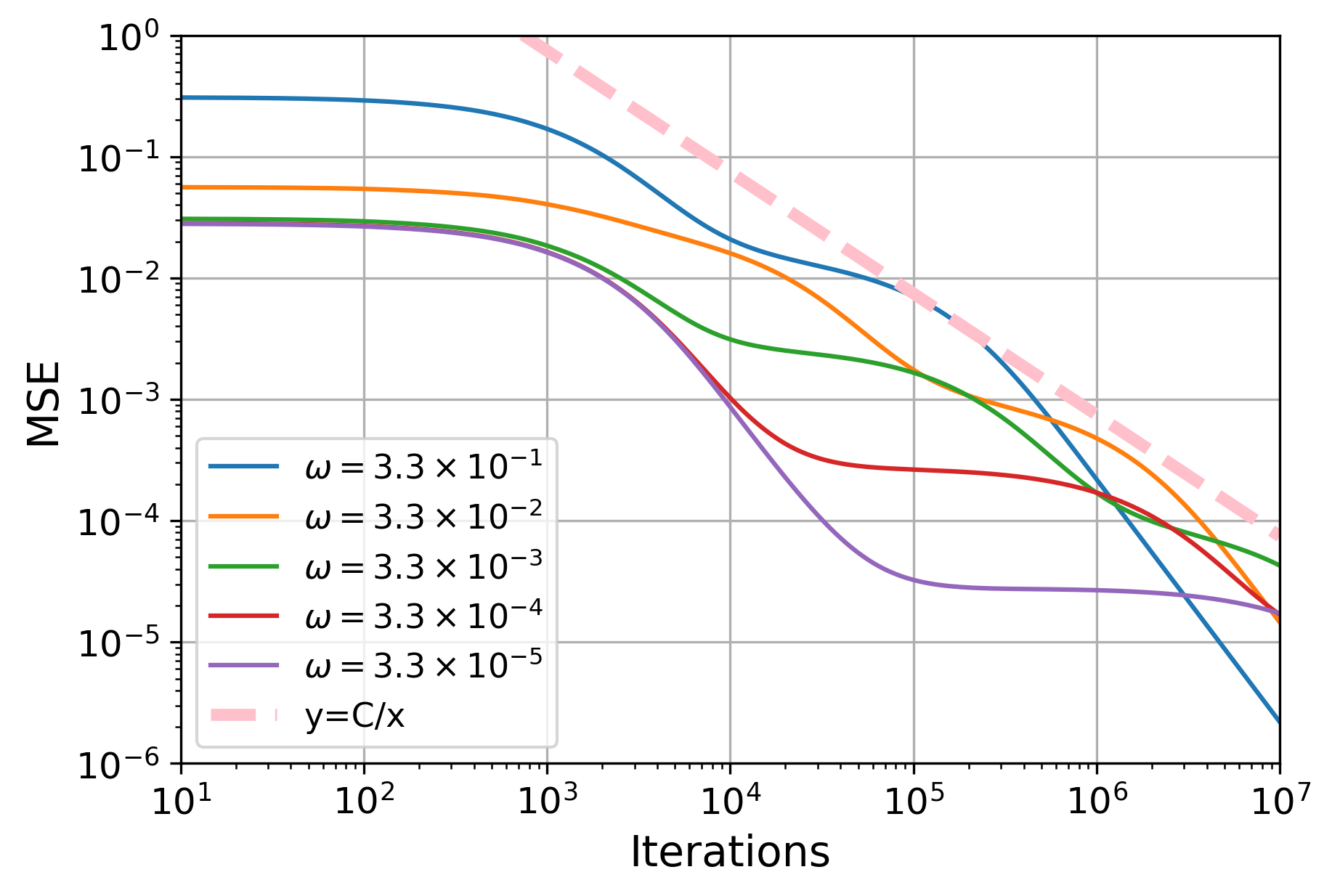}
        \caption{Bias}
    \end{subfigure}
    \hspace*{0.01\textwidth}
    \begin{subfigure}{0.31\textwidth}
        \includegraphics[width=\textwidth]{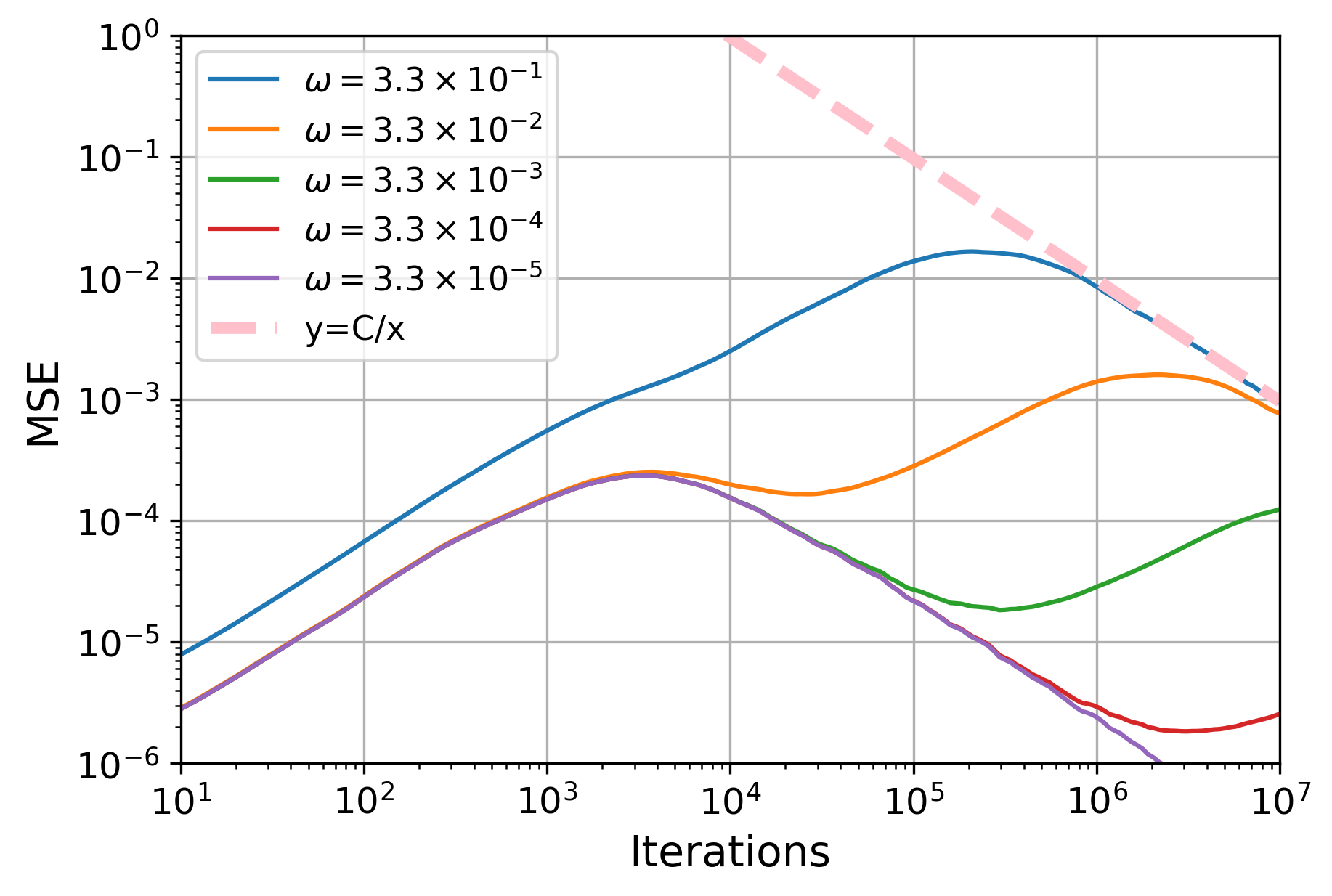}
        \caption{Variance}
    \end{subfigure}
    \hspace*{0.01\textwidth}
    \begin{subfigure}{0.31\textwidth}
        \includegraphics[width=\textwidth]{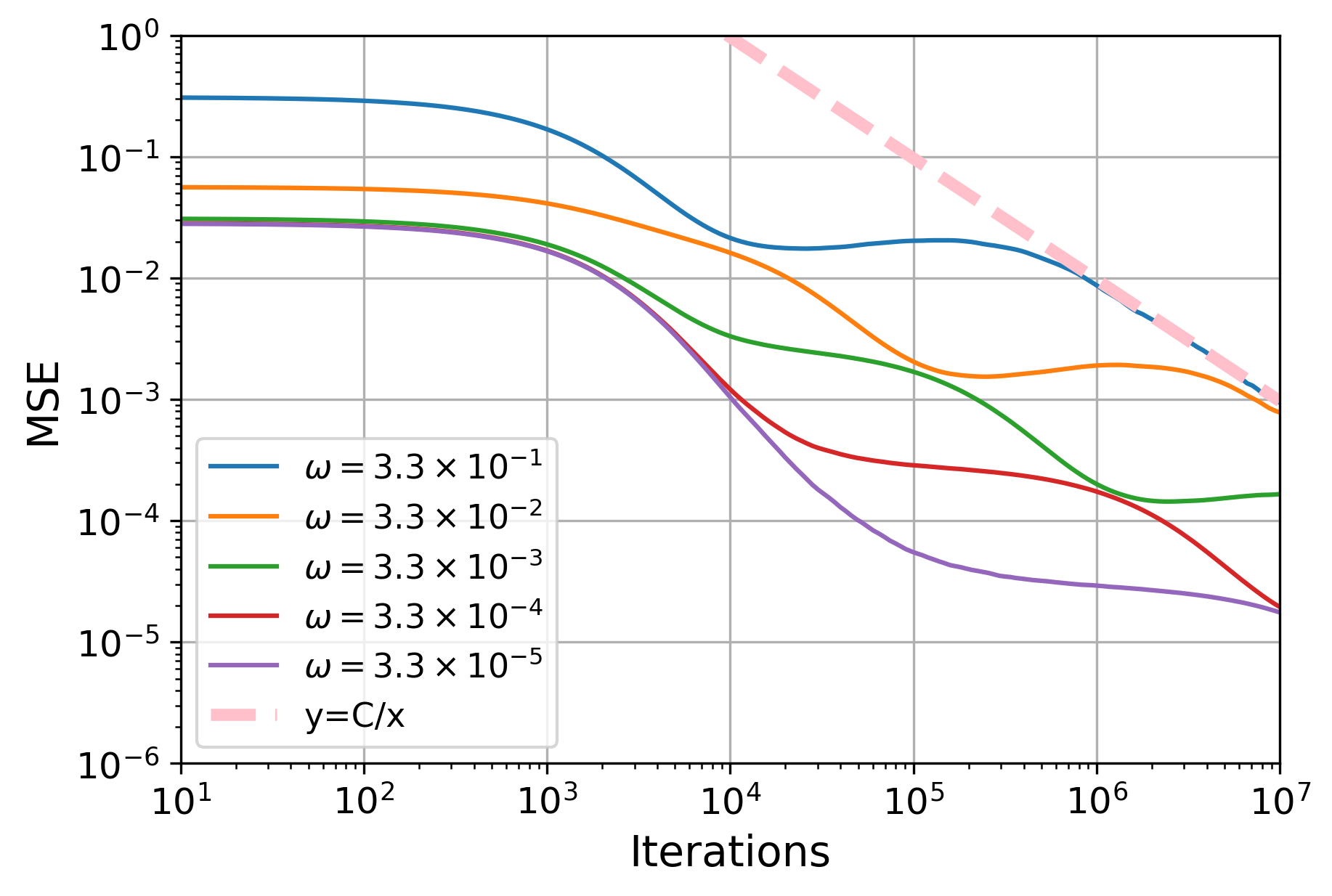}
        \caption{Total error}
    \end{subfigure}
    \end{center}
    \caption{\label{fig:d=3_bis}
    We use similar parameters as in Figure \ref{fig:d=3} but with $\gamma=0.9$ and $\alpha=1$ for the first line
    and $\gamma=0.99$ and $\alpha=\frac{1-\gamma}4=2.5\times10^{-3}$
    for the second line.
    }
\end{figure*}

\subsection{Figure \ref{fig:d=100}}
\label{subsec:fig_d=100}
Here, we give more details on the model used for the simulations in
Figure \ref{fig:d=100}. It is random, as the one described in
Subsection \ref{subsec:MRP_controlled_EV}, 
but it is no longer bistochastic and
we do not prescribe its eigenvalues.

\paragraph{The Markov decision process.}
We take $\cX=\{1,\dots,n\}$ with $n=1000$ as the state space.
The transition matrix $T\in\RR^{n\times n}$ is drawn at random: 
its rows $(T_i)_{1\leq i\leq n}$
are i.i.d., each distributed according to a Dirichlet law with parameter
$2/|\cX|=2\times10^{-3}$.
The rewards follow $\cN(1,\ss_R^2)$ with $\ss_R^2=1$.

\paragraph{The linear parametrisation.}
We take $d=100<n$, so that the space of admissible functions is strictly smaller
than the set of all functions on $\cX$. The first feature map is constant equal
to one, i.e., $\vp_1\equiv1$. The remaining feature maps are defined as follows:
let $\overline{\Phi}$ and $\widetilde{\Phi}$ be two uniform random variables with
values in $[-1,1]^{d-1}$ and $[-1,1]^{(d-1)\times n}$, respectively; we then set,
for $1\leq i\leq d-1$ and $1\leq j\leq n$,
$$
    \Phi_{i+1,j}
    =
    (\overline{\Phi}_i
    +\widetilde{\Phi}_{i,j})/Z_j
    \hspace*{0.3cm}
    \text{ with }
    \hspace*{0.3cm}
    Z_j=\sum_{\ell=1}^{d-1}
    |\overline{\Phi}_{\ell}
    +\widetilde{\Phi}_{\ell,j}|.
$$
The initial condition $\theta_0\in\RR^d$ is drawn from the uniform distribution
over $[-1,1]^d$.

\paragraph{Hyperparameters for Figure \ref{subfib:td0_vs2lr}.}
We use a minibatch size of $B=1$ and take $\gamma=0.99$, since our analysis of
TD(0) with two learning rates only covers this case. For this variant (described
in Subsection \ref{subsec:reduce_gamma_dep}), the learning rate of the first
component is $\frac{1}{4(1-\gamma)}$, while all other learning rates---those of
the remaining components and of the standard TD(0)---are set to
$\frac{\alpha_1(\gamma)}2$, where $\alpha_1(\gamma)$ is defined in
Theorem \ref{thm:TD0}.

\paragraph{Hyperparameters for Figure \ref{subfib:td0_vspctd0}.}
We use a minibatch size of $B=100$ and again take $\gamma=0.99$. For the orange
curves (TD(0) with two learning rates, as described in
Subsection \ref{subsec:reduce_gamma_dep}), the learning rate of the first
component is $\frac{1}{4(1-\gamma)}$. All remaining learning rates---those of the
other components of the orange curves and of the other curves---are set to
$\frac{\alpha_B(\gamma)}2$, where $\alpha_B(\gamma)$ is defined in
Theorem \ref{thm:minibatch_TD0}.

\paragraph{Hyperparameters for Figure \ref{subfib:pctd0s}.}
We again use a minibatch size of $B=100$. This time, each curve has its own
discount factor $\gamma\in\{0.9,0.99,1,1.15\}$. Once the MRP was drawn, we
computed the spectral gap defined in Assumption \ref{hypo:spectral_gap}, which
was approximately $0.145$. By Theorem \ref{thm:PCTD0}, the corresponding range
of admissible $\gamma$ was approximately $(0,1.17)$, which covers all the values
above; we used the learning rate $\alpha=\frac{\aah_B(\gamma)}2$.

\section{Proofs of the results on LSA}
\subsection{Proof of Proposition \ref{prop:LSA}}
\begin{proof}[Proof of Proposition \ref{prop:LSA}]
    Define $z_k=y_k-y^*$, it satisfies
    \begin{equation*}
        z_{k}
        =
        (I_d-\alpha h_k)z_{k-1}
        -\alpha (h_ky^*-b_k)
        \;\text{ with }\;
        z_0
        =
        y_0-y^*.
    \end{equation*}
    Observe that 
    $-\alpha (h_ky^*-b_k)$ plays the role of a source term.
    Therefore, by linearity (or Duhamel's principle),
    we get that $z_k=w_k+\eta_k$,
    where $w_k$ and $\eta_k$ are defined by induction as
    \begin{align*}
        w_{k}
        &=
        (I_d-\alpha h_k)w_{k-1}
        \hspace*{3.3cm}
        \text{ with }
        \hspace*{0.8cm}
        w_0
        =
        y_0-y^*,
        \\
        \eta_{k}
        &=
        (I_d-\alpha h_k)\eta_{k-1}
        -\alpha (h_ky^*-b_k)
        \hspace*{1cm}
        \text{ with }
        \hspace*{0.8cm}
        \eta_0
        =
        0.
    \end{align*}
    In the following, we will prove 
    \begin{equation*}
        \EE\lb \wo_k^{\top} \Sigma\wo_k\rb
        \leq
        \frac{c_{\Sigma}|w_0|^2}{\alpha(2-\alpha\beta_1) k}
    \;\text{ and }\;
        \EE\lb\etao_k^{\top}\Sigma\etao_k\rb
        \leq
        \frac{2\ss^2}{(2-\alpha\beta c_{\Sigma})k}
        \cS_k.
    \end{equation*}
    so that we will be able to conclude using the inequality
    $\EE\lb \zo_k^{\top} \Sigma\zo_k\rb
    \leq
    \lp\EE\lb \wo_k^{\top} \Sigma\wo_k\rb^{\frac12}
    +\EE\lb\etao_k^{\top}\Sigma\etao_k\rb^{\frac12}\rp^2$.

    \emph{First step: dealing with the initial condition (corresponding to $w$).}

    \noindent
    For $k\geq1$, 
    using $\EE\lb h_k^{\top}h_k\rb\leq \beta_1S$,
    observe that
    \begin{align*}
        \EE\lb |w_k|^2\rb
        &=
        \EE\lb w_{k-1}^{\top}(I_d-\aa h_k-\aa h_k^{\top}+\aa^2 h_k^{\top}h_k)w_{k-1}\rb
        \\
        &\leq
        \EE\lb w_{k-1}^{\top}(I_d-2\aa S+\aa^2\beta_1 S)w_{k-1}\rb
        \\
        &\leq
        \EE\lb w_{k-1}^{\top}(I_d-\aa(2-\alpha\beta_1) S)w_{k-1}\rb
        \\
        &=
        \EE\lb |w_{k-1}|^2\rb
        -\aa(2-\alpha\beta_1)\EE\lb w_{k-1}^{\top}Sw_{k-1}\rb.
    \end{align*}
    Taking the sum from $1$ to $k$ in the latter inequality and using that $u\mapsto u^{\top}Su$ is convex,
    we obtain
    \begin{equation*}
        \EE\lb \wo_k^{\top} S\wo_k\rb
        \leq
        \frac1k\sum_{i=0}^{k-1}\EE\lb w_i^{\top} Sw_i\rb
        \leq
        \frac1{\alpha k}\sum_{i=1}^{k}
        \lp\EE\lb |w_{k-1}|^2\rb
        -\EE\lb |w_k|^2\rb\rp
        \leq
        \frac{|w_0|^2}{\alpha(2-\alpha\beta_1) k}.
    \end{equation*}
    This and $\Sigma\leq c_{\Sigma}S$
    conclude the part on the bias.

    \emph{Second step: getting bounds on the uncentered covariances of variance term (corresponding to $\eta$).}

    \noindent
    Using a straightforward induction, we have $\EE[\eta_k]=0$ for $k\geq0$.
    We will now prove by induction that 
    \begin{equation}
    \label{eq:etaketakt}
        \EE\lb\eta_k\eta_k^{\top}\rb
        \leq
        \frac{\ss^2\alpha c_{\Sigma}}{2-\alpha\beta c_{\Sigma}}I_d.
    \end{equation}
    Recall that $\eta_0=0$ so the above inequality holds for $k=0$.
    Then, for $k\geq1$, we assume that it holds at index $k-1$.
    Since $\eta_{k-1}$ and $(h_k,b_k)$ are independent, we obtain
    \begin{equation*}
        \EE\lb(I_d-\alpha h_k)\eta_{k-1}(h_ky^*-b_k)^{\top}\rb
        =
        \EE\lb(I_d-\alpha h_k)\EE\lb\eta_{k-1}\rb(h_ky^*-b_k)^{\top}\rb
        =
        0.
    \end{equation*}
    Using the latter equality, Inequality \eqref{eq:bound_Id-aaH},
    $\Sigma\leq c_{\Sigma}S$,
    and the second inequality in Assumption \ref{hypo:ineq_LSA},
    we obtain
    \begin{align*}
        \EE\lb\eta_k\eta_k^{\top}\rb
        &=
        \EE\lb(I_d-\alpha h_k)\eta_{k-1}\eta_{k-1}^{\top}(I_d-\alpha h_k^{\top})\rb
        +\alpha^2\EE\lb(h_ky^*-b_k)(h_ky^*-b_k)^{\top}\rb
        \\
        &\leq
        \frac{\ss^2\alpha c_{\Sigma}}{2-\alpha\beta c_{\Sigma}}
        \EE\lb(I_d-\alpha h_k)(I_d-\alpha h_k^{\top})\rb
        +\ss^2\alpha^2\Sigma
        \\
        &\leq
        \frac{\ss^2\alpha c_{\Sigma}}{2-\alpha\beta c_{\Sigma}}
        (I_d-\alpha(2-\alpha\beta c_{\Sigma}) S)
        +\ss^2\alpha^2 c_{\Sigma}S
        \\
        &\leq
        \frac{\ss^2\alpha c_{\Sigma}}{2-\alpha\beta c_{\Sigma}}I_d.
    \end{align*}
    This concludes the induction and the proof of Inequality \eqref{eq:etaketakt}.

    \emph{Third step: introduce an appropriate decomposition for $\eta$.}

    \noindent
    Let us then define
    $(\eta^r_k)_{k\geq0,0\leq r\leq k+1}$ and
    $(\chi^r_k)_{1\leq r\leq k}$
    by induction by, 
    for $1\leq r\leq k$,
    \begin{align*}
        \eta^r_k
        &=
        (I_d-\alpha H)\eta^r_{k-1}
        +\chi_k^r
        \hspace*{0.4cm}
        \text{ with }
        \hspace*{0.4cm}
        \eta^0_{k-1}
        =
        \eta^{k}_{k-1}
        =
        0,
        \\
        \chi_k^{r}
        &=
        -\alpha(h_k-H)\eta^{r-1}_{k-1}
        -\alpha(h_ky^*-b_k)
        \delta_{\lc r=1\rc}.
    \end{align*}
    Let us check by induction on $k\geq1$ that
    $\eta_k=\sum_{r=1}^k\eta^r_k$.
    The case $k=1$ is a simple consequence of the above initial condition.
    Now assume that it holds for $k-1$, for $k\geq2$,
    and let us prove it for index $k$:
    \begin{align*}
        \eta_k
        &=
        (I_d-\alpha h_k)
        \sum_{r=1}^{k-1}\eta^r_{k-1}
        +\chi^1_k
        \\
        &=
        (I_d-\alpha H)
        \sum_{r=1}^{k-1}\eta^r_{k-1}
        -\alpha(h_k-H)
        \sum_{r=2}^{k}\eta^{r-1}_{k-1}
        +\chi^1_k
        \\
        &=
        \sum_{r=1}^{k}
        \lp(I_d-\alpha H)
        \eta^r_{k-1}
        +\chi^r_k\rp
        \\
        &=
        \sum_{r=1}^{k}
        \eta^r_k.     
    \end{align*}
    In fact, $\eta^r_k$ can be rewritten as a sum of terms,
    each of which admits exactly $r$ different multiplicative centered noises
    (of the form $h_j-H$ or $h_jy^*-b_j$ for some $1\leq j\leq k$).
    Therefore, for $r'\neq r$, any term from the development of
    $\eta^{r'}_k(\eta^{r}_k)^{\top}$
    contains at least one multiplicative centered noise that appears only once and
    is independent of the rest of this same term
    (i.e., is of the form
    $R_1(h_j-H)R_2$ or
    $R_1(h_jy^*-b_j)$
    with neither $R_1$ or $R_2$ containing $h_j$ or $b_j$).
    This implies that 
    $\EE\lb\eta^{r'}_k(\eta^{r}_k)\rb=0$,
    so that we get
    \begin{equation}
        \label{eq:decomp_eta}
        \EE\lb\eta_k(\eta_k)^{\top}\rb
        =
        \sum_{r=1}^k
        \sum_{r'=1}^k
        \EE\lb\eta^r_k(\eta^{r'}_k)^{\top}\rb
        =
        \sum_{r=1}^k
        \EE\lb\eta^r_k(\eta^{r}_k)^{\top}\rb.
    \end{equation}
    Using similar arguments, we also have that,
    for $\etao_k:=\frac1k\sum_{j=0}^{k-1}\eta_k$
    and $\etao^r_k:=\frac1k\sum_{j=0}^{k-1}\eta^r_k$,
    \begin{equation}
        \label{eq:decomp_etao}
        \EE\lb\etao_k^{\top}\Sigma\etao_k\rb
        =
        \sum_{r=1}^k
        \EE\lb(\etao^r_k)^{\top}\Sigma\etao^r_k\rb.
    \end{equation}

    \emph{Fourth step: obtaining the convergence rate.}

    \noindent
    For $1\leq r\leq k$,
    the induction relation of $\eta^r_k$ implies
    \begin{equation*}
        \eta^r_k
        =
        \sum_{j=1}^k
        (I_d-\alpha H)^{k-j}\chi^r_j.
    \end{equation*}
    For $k\geq2$,
    define
    $\etao^r_k=\frac1k\sum_{i=1}^{k-1}\eta^r_k$,
    it satisfies
    \begin{equation*}
        \etao^r_k
        =
        \frac1k
        \sum_{i=1}^{k-1}
        \sum_{j=1}^i
        (I_d-\alpha H)^{i-j}\chi^r_j
        =
        \frac1k
        \sum_{j=1}^{k-1}
        \sum_{i=j}^{k-1}
        (I_d-\alpha H)^{i-j}\chi^r_j
        =
        \frac1{\alpha k}
        \sum_{j=1}^{k-1}
        M_{k-j}H^{-1}\chi^r_j
        =
        \frac1{\alpha k}
        \sum_{j=0}^{k-1}
        M_{j}H^{-1}\chi^r_{k-j}.
    \end{equation*}
    where $M_{j}=I_d-(I_d-\alpha H)^{j}$.
    Observe that $M_{j}$ commutes with $H$ and $H^{-1}$
    (but not with $H^{\top}$ and $H^{-\top}$).
    On the one hand, using Inequality \eqref{eq:bound_Id-aaH},
    we get
    \begin{align*}
        \EE\lb(\etao^1_k)^{\top}\Sigma\etao^1_k\rb
        &=
        \frac1{\alpha^2 k^2}
        \sum_{j=0}^{k-1}
        \EE\lb(\chi^1_{k-j})^{\top}H^{-\top}
        M_{j}^{\top}\Sigma M_{j}
        H^{-1}\chi^1_{k-j}\rb
        \\
        &=
        \frac1{\alpha^2 k^2}
        \sum_{j=0}^{k-1}
        \tr\lp
        \Sigma^{\frac12}M_{j}H^{-1}
        \EE\lb\chi^1_{k-j}(\chi^1_{k-j})^{\top}\rb 
        H^{-\top}M_{j}^{\top}\Sigma^{\frac12}\rp
        \\
        &=
        \frac1{k^2}
        \sum_{j=1}^{k-1}
        \tr\lp
        \Sigma^{\frac12}M_{j}H^{-1}
        \EE\lb(h_{k-j}y^*-b_{k-j})(h_{k-j}y^*-b_{k-j})^{\top})\rb
        H^{-\top}M_{j}^{\top}\Sigma^{\frac12}\rp
        \\
        &\leq
        \frac{\ss^2}{k^2}
        \sum_{j=1}^{k-1}
        \tr\lp
        \Sigma^{\frac12}M_{j}H^{-1}\Sigma
        H^{-\top}M_{j}^{\top}\Sigma^{\frac12}\rp
    \end{align*}
    On the other hand, we have
    \begin{align*}
        \sum_{r=2}^k
        \EE\lb(\etao^r_k)^{\top}\Sigma\etao^r_k\rb
        &=
        \frac1{\alpha^2 k^2}
        \sum_{r=2}^k
        \sum_{j=0}^{k-1}
        \EE\lb(\chi^r_{k-j})^{\top}H^{-\top}
        M_{j}^{\top}\Sigma M_{j}
        H^{-1}\chi^r_{k-j}\rb
        \\
        &=
        \frac1{\alpha^2 k^2}
        \sum_{r=2}^k
        \sum_{j=0}^{k-1}
        \tr\lp
        \Sigma^{\frac12}M_{j}H^{-1}
        \EE\lb\chi^r_{k-j}\otimes\chi^r_{k-j}\rb 
        H^{-\top}M_{j}^{\top}\Sigma^{\frac12}\rp
        \\
        &=
        \frac1{k^2}
        \sum_{r=2}^k
        \sum_{j=0}^{k-1}
        \tr\lp
        \Sigma^{\frac12}M_{j}H^{-1}
        \EE\lb(h_{k-j}-H)\eta^{r-1}_{k-j-1}(\eta^{r-1}_{k-j-1})^{\top}(h_{k-j}-H)^{\top}\rb 
        H^{-\top}M_{j}^{\top}\Sigma^{\frac12}\rp
        \\
        &=
        \frac1{k^2}
        \sum_{j=0}^{k-1}
        \tr\lp
        \Sigma^{\frac12}M_{j}H^{-1}
        \EE\lb(h_{k-j}-H)
        \sum_{r=2}^k
        \EE\lb\eta^{r-1}_{k-j-1}(\eta^{r-1}_{k-j-1})^{\top}\rb
        (h_{k-j}-H)^{\top}\rb 
        H^{-\top}M_{j}^{\top}\Sigma^{\frac12}\rp
        \\
        &=
        \frac1{k^2}
        \sum_{j=0}^{k-1}
        \tr\lp
        \Sigma^{\frac12}M_{j}H^{-1}
        \EE\lb(h_{k-j}-H)
        \EE\lb\eta_{k-j-1}(\eta_{k-j-1})^{\top}\rb
        (h_{k-j}-H)^{\top}\rb 
        H^{-\top}M_{j}^{\top}\Sigma^{\frac12}\rp
        \\
        &\leq
        \frac{\ss^2\alpha c_{\Sigma}}{(2-\alpha\beta c_{\Sigma})k^2}
        \sum_{j=0}^{k-1}
        \tr\lp
        \Sigma^{\frac12}M_{j}H^{-1}
        \EE\lb(h_{k-j}-H)(h_{k-j}-H)^{\top}\rb 
        H^{-\top}M_{j}^{\top}\Sigma^{\frac12}\rp
        \\
        &\leq
        \frac{\ss^2\alpha\beta c_{\Sigma}}{(2-\alpha\beta c_{\Sigma})k^2}
        \sum_{j=0}^{k-1}
        \tr\lp
        \Sigma^{\frac12}M_{j}
        H^{-1}\Sigma H^{-\top}
        M_{j}^{\top}\Sigma^{\frac12}\rp,
        \end{align*}
        where we used Equality \eqref{eq:decomp_eta} to get the fifth line,
        the second step of this proof to get the sixth line
        and the first inequality of Assumption \ref{hypo:ineq_LSA} to obtain the last one.
        The latter two chains of inequalities and Equality \eqref{eq:decomp_etao}
        imply
        \begin{align*}
        \EE\lb\etao_k^{\top}\Sigma\etao_k\rb
        &=
        \sum_{r=1}^k
        \EE\lb(\etao^r_k)^{\top}\Sigma\etao^r_k\rb
        \\
        &\leq
        \frac{\ss^2}{k^2}\lp1+\frac{\alpha\beta c_{\Sigma}}{2-\alpha\beta c_{\Sigma}}\rp
        \sum_{j=0}^{k-1}
        \tr\lp
        \Sigma^{\frac12}M_{j}
        H^{-1}\Sigma H^{-\top}
        M_{j}^{\top}\Sigma^{\frac12}\rp
        \\
        &=
        \frac{2\ss^2}{(2-\alpha\beta c_{\Sigma})k^2}
        \sum_{j=0}^{k-1}
        \norm[2]{\Sigma^{\frac12}\lp I_d-(I_d-\alpha H)^j\rp H^{-1}\Sigma^{\frac12}}{F}
        \\
        &=
        \frac{2\ss^2}{(2-\alpha\beta c_{\Sigma})k^2}
        \sum_{j=0}^{k-1}
        \norm[2]{\lp I_d-(I_d-\alpha \Sigma^{\frac12}H\Sigma^{-\frac12})^j\rp\Sigma^{\frac12} H^{-1}\Sigma^{\frac12}}{F}
        \\
        &=
        \frac{2\ss^2}{(2-\alpha\beta c_{\Sigma})k}
        \cS_k.
    \end{align*}
    To conclude, it only remains to prove that $\cS_k\leq c_{\Sigma}^2(3d+ \cSt_k)$, 
    which is a consequence of the following computations,
    \begin{equation*}
        \Sigma^{\frac12}H^{-1}\Sigma H^{-\top}\Sigma^{\frac12}
        \leq
        c_{\Sigma}\Sigma^{\frac12}H^{-1}S H^{-\top}\Sigma^{\frac12}
        =
        c_{\Sigma}\Sigma^{\frac12}\frac{H^{-1}+H^{-\top}}2\Sigma^{\frac12}
        =
        c_{\Sigma}\Sigma^{\frac12}S^{-1}\Sigma^{\frac12}
        \leq
        c_{\Sigma}^2I_d,
    \end{equation*}
    where we used
    $\Sigma\leq c_{\Sigma}S$,
    Heinz's inequality $H^{-1}+H^{-\top}\leq 2S^{-1}$
    and $S^{-1}\leq c_{\Sigma}\Sigma^{-1}$;
    and
    \begin{equation*}
        \norm[2]{\lp I_d-Q^j\rp\Sigma^{\frac12} H^{-1}\Sigma^{\frac12}}{F}
        \leq
        c_{\Sigma}^2
        \norm[2]{I_d-Q^j}{F}
        =
        c_{\Sigma}^2\lp \tr(I_d)
        -2\tr(Q^j)
        +\norm[2]{Q^j}{F}\rp
        \leq 
        c_{\Sigma}^2(3d+\norm[2]{Q^j}{F}),
    \end{equation*}
    where we used
    $|\tr(Q^j)|=|\tr((I_d-\alpha H)^j)|
    \leq d\|I_d-\alpha H\|_{\rm op}^j\leq d$.
    This concludes the proof.
\end{proof}

\subsection{Proof of Theorem \ref{thm:LSA}}
\begin{proof}[Proof of Theorem \ref{thm:LSA}]
    This is a consequence of Proposition \ref{prop:LSA}
    and Lemma \ref{lem:cSk_LSA} below.
\end{proof}

\begin{lemma}
    \label{lem:cSk_LSA}
    Under Assumptions \ref{hypo:SPD}-\ref{hypo:ineq_LSA},
    for $0<\alpha<\frac2{\beta_1}$, we have
    \begin{equation*}
        \cS_k
        \leq
        c_{\Sigma}^2d\lp 3+\frac{1}{\alpha(2-\alpha\beta_1)\mu k}\rp.
    \end{equation*}
    Under Assumptions \ref{hypo:SPD},\ref{hypo:iid_LSA} and \ref{hypo:ineq_LSA}
    without the last inequality,
    for $0<\alpha<\frac2{\beta c_{\Sigma}}$, we have
    \begin{equation*}
        \cS_k
        \leq
        c_{\Sigma}^2d\lp 3+\frac{1}{\alpha(2-\alpha\beta c_{\Sigma})\mu k}\rp.
    \end{equation*}
\end{lemma}

\begin{proof}
    Let us prove the first inequality.
    Using $M_j=I_d-(I_d-\alpha H)^j$ and $N_j=(I_d-\alpha H)^j$ for $j\geq1$,
    we have for $k\geq1$:
    \begin{align*}
        \cS_k
        &=
        \frac1k\sum_{j=1}^{k-1}
        \tr\lp
        \Sigma^{\frac12}M_{j}H^{-1}\Sigma
        H^{-\top}M_{j}^{\top}\Sigma^{\frac12}\rp
        \\
        &\leq
        \frac{c_{\Sigma}}k\sum_{j=1}^{k-1}
        \tr\lp
        \Sigma^{\frac12}M_{j}H^{-1}S
        H^{-\top}M_{j}^{\top}\Sigma^{\frac12}\rp
        \\
        &\leq
        \frac{c_{\Sigma}}k\sum_{j=1}^{k-1}
        \tr\lp
        \Sigma^{\frac12}M_{j}\frac{H^{-1}+H^{-\top}}2
        M_{j}^{\top}\Sigma^{\frac12}\rp
        \\
        &\leq
        \frac{c_{\Sigma}}k\sum_{j=1}^{k-1}
        \tr\lp
        \Sigma^{\frac12}M_{j}S^{-1}
        M_{j}^{\top}\Sigma^{\frac12}\rp
        \\
        &=
        \frac{c_{\Sigma}}k\sum_{j=1}^{k-1}
        \tr\lp
        S^{-\frac12}M^{\top}_{j}\Sigma
        M_{j}S^{-\frac12}\rp
        \\
        &\leq
        \frac{c_{\Sigma}^2}k\sum_{j=1}^{k-1}
        \tr\lp
        S^{-\frac12}M^{\top}_{j}S
        M_{j}S^{-\frac12}\rp
        \\
        &=
        \frac{c_{\Sigma}^2}k\sum_{j=1}^{k-1}
        \tr\lp
        I_d+
        2\tr(N_j)
        +S^{-\frac12}N^{\top}_{j}S
        N_{j}S^{-\frac12}\rp
        \\
        &\leq
        3c_{\Sigma}^2d
        +\frac{c_{\Sigma}^2}k\tr\lp S^{-\frac12}\lp\sum_{j=1}^{k-1}N_j^{\top}SN_j\rp S^{-\frac12}\rp
        \\
        &\leq
        c_{\Sigma}^2\lp 3d+\frac{\tr(S^{-1})}{\alpha(2-\alpha\beta_1)k}\rp
        \\
        &\leq
        c_{\Sigma}^2d\lp 3+\frac{1}{\alpha(2-\alpha\beta_1)\mu k}\rp,
    \end{align*}
    where we used $\Sigma\leq c_{\Sigma}S$ to get the second and sixth lines,
    Heinz's inequality (Lemma~\ref{lem:SH-top}) to get the fourth one,
    $\tr(N_j)\leq\tr(I_d)^{\frac12}\tr(N_j^{\top}N_j)^{\frac12}\leq \tr(I_d)=d$
    (since $(I_d-\alpha H)^{\top}(I_d-\alpha H)\leq I_d$)
    and Inequality \eqref{eq:bound_sum_sym} to get the eighth one,
    and $S^{-1}\leq\mu^{-1}I_d$ to get the last one.

    To get the second inequality, we write 
    \begin{equation*}
        \cS_k
        =
        \frac1k\sum_{j=1}^{k-1}
        \tr\lp
        \Sigma^{\frac12}M_{j}^{\top}H^{-\top}\Sigma
        H^{-1}M_{j}\Sigma^{\frac12}\rp,
    \end{equation*}
    using that $H$ and $M_j$ commutes.
    Then, we repeat similar arguments as the ones for the first inequality.
    This concludes the proof.
\end{proof}

\subsection{An extension of Proposition \ref{prop:LSA} to the use of minibatches}
\begin{proposition}
    \label{prop:LSA_minibatch}
    Assume \ref{hypo:SPD}-\ref{hypo:ineq_LSA}.
    Consider the following LSA with minibatches
    \begin{equation*}
        y_k
        =
        y_k-\alpha (h^B_ky_{k-1}-b^B_k)
        \hspace*{0.3cm}
        \text{ with }
        \hspace*{0.3cm}
        h^B_k=\frac1B\sum_{i=1}^Bh_{k,i}
        \hspace*{0.3cm}
        \text{ with }
        \hspace*{0.3cm}
        b^B_k=\frac1B\sum_{i=1}^Bb_{k,i},
    \end{equation*}
    where $(h_{k,i},b_{k,i})_{k\geq1,1\leq i\leq B}$
    are i.i.d. copies of $(h_k,b_k)$ from Section \ref{sec:LSA}.
    For $0<\alpha<\min\lp \frac{2B}{\beta c_{\Sigma}+(B-1)\beta_1},\frac{2}{\beta_1}\rp$,
    with $\cS_k$ defined in Proposition \ref{prop:LSA},
    we have
    \begin{equation*}
        \EE\lb(\yo_k^{\top}-y^*)^{\top}\Sigma(\yo_k-y^*)\rb
        \leq
        \lp\lp
        \frac{c_{\Sigma}|\theta_0-\theta^*|^2}{\alpha(2-\alpha\beta_1) k}
        \rp^{\frac12}
        +\lp\frac{2\ss^2\cS_k}{(2B-\alpha(\beta c_{\Sigma}+(B-1)\beta_1))k}\rp^{\frac12}\rp^2.
    \end{equation*}
\end{proposition}
\begin{proof}
    The proof follows a strategy similar to that of Proposition \ref{prop:LSA} 
    and uses similar notations.
    Only the second and fourth step changes, as shown below.

    \emph{Changes in the second step.}

    \noindent
    First, let us notice that using minibatches allows us to improve Inequality \eqref{eq:bound_Id-aaH},
    we have
    \begin{equation}
    \label{eq:LSA_minibatch_aux}
    \begin{aligned}
        \EE\lb (I_d-\alpha h_k^B)(I_d-\alpha h_k^B)^{\top}\rb
        &=
        \frac1B\EE\lb (I_d-\alpha h)(I_d-\alpha h)^{\top}\rb
        +\frac{B-1}B(I_d-\alpha H)(I_d-\alpha H^{\top})
        \\
        &\leq
        \frac1B\lp I_d-\alpha(2-\alpha\beta c_{\Sigma})S\rp
        +\frac{B-1}B\lp I_d-\alpha(2-\alpha\beta_1)S\rp
        \\
        &=
        I_d-\alpha\lp 2-\frac{\alpha \beta c_{\Sigma}}B-\frac{(B-1)\alpha\beta_1}B\rp S,
    \end{aligned}
    \end{equation}
    where we used \eqref{eq:bound_Id-aaH}
    and \eqref{eq:bound_Id-aaH2} to get the second line.
    
    This then allows us to improve Inequality \eqref{eq:etaketakt}.
    More precisely, we are going to prove by induction:
    \begin{equation}
    \label{eq:etaketakt_minibatch}
        \EE\lb\eta_k\eta_k^{\top}\rb
        \leq
        \frac{\ss^2\alpha c_{\Sigma}}{2B-\alpha(\beta c_{\Sigma}+(B-1)\beta_1)}I_d.
    \end{equation}
    The case $k=0$ is straightforward.
    Let us assume that the inequality holds at index $k-1$ and prove it at $k$,
    \begin{align*}
        \EE\lb\eta_k\eta_k^{\top}\rb
        &=
        \EE\lb(I_d-\alpha h^B_k)\eta_{k-1}\eta_{k-1}^{\top}(I_d-\alpha h^B_k)^{\top}\rb
        +\alpha^2\EE\lb(h^B_ky^*-b^B_k)(h^B_ky^*-b^B_k)^{\top}\rb
        \\
        &\leq
        \frac{\ss^2\alpha c_{\Sigma}}{2B-\alpha(\beta c_{\Sigma}+(B-1)\beta_1)}
        \EE\lb(I_d-\alpha h^B_k)(I_d-\alpha h^B_k)^{\top}\rb
        +\frac{\ss^2\alpha^2}B\Sigma
        \\
        &\leq
        \frac{\ss^2\alpha c_{\Sigma}}{2B-\alpha(\beta c_{\Sigma}+(B-1)\beta_1)}
        \lp I_d-\alpha\lp 2-\frac{\alpha \beta c_{\Sigma}}B-\frac{(B-1)\alpha\beta_1}B\rp S\rp
        +\frac{\ss^2\alpha^2c_{\Sigma}}BS
        \\
        &\leq
        \frac{\ss^2\alpha c_{\Sigma}}{2B-\alpha(\beta c_{\Sigma}+(B-1)\beta_1)}I_d,
    \end{align*}
    where we used \eqref{eq:LSA_minibatch_aux} to get the third line.
    This concludes the induction.

    \emph{Changes in the fourth step.}
    Recall the notation $M_{j}=I_d-(I_d-\alpha H)^{j}$.
    On the one hand, using Inequality \eqref{eq:bound_Id-aaH},
    we get
    \begin{align*}
        \EE\lb(\etao^1_k)^{\top}\Sigma\etao^1_k\rb
        &=
        \frac1{\alpha^2 k^2}
        \sum_{j=0}^{k-1}
        \EE\lb(\chi^1_{k-j})^{\top}H^{-\top}
        M_{j}^{\top}\Sigma M_{j}
        H^{-1}\chi^1_{k-j}\rb
        \\
        &=
        \frac1{k^2}
        \sum_{j=1}^{k-1}
        \tr\lp
        \Sigma^{\frac12}M_{j}H^{-1}
        \EE\lb(h^B_{k-j}y^*-b^B_{k-j})(h^B_{k-j}y^*-b^B_{k-j})^{\top})\rb
        H^{-\top}M_{j}^{\top}\Sigma^{\frac12}\rp
        \\
        &\leq
        \frac{\ss^2}{Bk}\cS_k
    \end{align*}
    On the other hand, using similar arguments as in the proof of Proposition \ref{prop:LSA}, we have
    \begin{align*}
        \sum_{r=2}^k
        \EE\lb(\etao^r_k)^{\top}\Sigma\etao^r_k\rb
        &=
        \frac1{\alpha^2 k^2}
        \sum_{r=2}^k
        \sum_{j=0}^{k-1}
        \EE\lb(\chi^r_{k-j})^{\top}H^{-\top}
        M_{j}^{\top}\Sigma M_{j}
        H^{-1}\chi^r_{k-j}\rb
        \\
        &=
        \frac1{k^2}
        \sum_{j=0}^{k-1}
        \tr\lp
        \Sigma^{\frac12}M_{j}H^{-1}
        \EE\lb(h^B_{k-j}-H)
        \EE\lb\eta_{k-j-1}(\eta_{k-j-1})^{\top}\rb
        (h^B_{k-j}-H)^{\top}\rb 
        H^{-\top}M_{j}^{\top}\Sigma^{\frac12}\rp
        \\
        &\leq
        \frac{\ss^2\alpha c_{\Sigma}}{(2B-\alpha(\beta c_{\Sigma}+(B-1)\beta_1))k^2}
        \sum_{j=0}^{k-1}
        \tr\lp
        \Sigma^{\frac12}M_{j}H^{-1}
        \EE\lb(h^B_{k-j}-H)(h^B_{k-j}-H)^{\top}\rb 
        H^{-\top}M_{j}^{\top}\Sigma^{\frac12}\rp
        \\
        &\leq
        \frac{\ss^2\alpha\beta c_{\Sigma}}{B(2B-\alpha(\beta c_{\Sigma}+(B-1)\beta_1))k}
        \cS_k,
        \end{align*}
        where we used Inequality \eqref{eq:etaketakt_minibatch} to obtain the third line,
        and the first inequality of Assumption \ref{hypo:ineq_LSA} to obtain the last one.
        The latter two chains of inequalities and Equality \eqref{eq:decomp_etao}
        imply
        \begin{align*}
        \EE\lb\etao_k^{\top}\Sigma\etao_k\rb
        &=
        \sum_{r=1}^k
        \EE\lb(\etao^r_k)^{\top}\Sigma\etao^r_k\rb
        \\
        &\leq
        \frac{\ss^2}{Bk}\cS_k
        +\frac{\ss^2\alpha\beta c_{\Sigma}}{B(2B-\alpha(\beta c_{\Sigma}+(B-1)\beta_1))k}\cS_k
        \\
        &=
        \frac{\ss^2(2B-(B-1)\alpha\beta_1)}{B(2B-\alpha(\beta c_{\Sigma}+(B-1)\beta_1))k}\cS_k
        \\
        &\leq
        \frac{2\ss^2}{(2B-\alpha(\beta c_{\Sigma}+(B-1)\beta_1))k}\cS_k.
    \end{align*}
    From there, we conclude with similar arguments as in the proof of Proposition \ref{prop:LSA}.
\end{proof}

\subsection{Extensions of Theorem \ref{thm:LSA}}
Here, we present two extensions of Theorem \ref{thm:LSA}.
The first one, Theorem \ref{thm:LSA_alternative},
allows for learning  steps up to $\alpha<\frac{2}{c_{\Sigma}\beta}$
and keep a $O\lp\frac1k+\frac1{k^2\mu}\rp$ like Theorem \ref{thm:LSA}.

The second extension is Theorem \ref{thm:LSA_alternative2} below
that holds for $\alpha<\frac{2}{\beta_1}$, but admits a slower
convergence rate of $O\lp\frac1{k\mu}\rp$.
Still this rate remains better than any existing rate
on LSA with similar assumptions prior to this work.

\begin{theorem}
    \label{thm:LSA_alternative}
    Assume \ref{hypo:SPD}, \ref{hypo:iid_LSA} and \ref{hypo:ineq_LSA}
    without the third inequality.
    For $0<\alpha<\frac{2}{\beta c_{\Sigma}}$,
    we have
    \begin{equation*}
        \EE\lb(\yo_k^{\top}-y^*)^{\top}\Sigma(\yo_k-y^*)\rb
        \leq
        \frac{c_{\Sigma}|y_0-y^*|^2}{(2-\alpha\beta c_k)\alpha k}
        +\lp\lp\frac{\beta|y_0-y^*|^2}{k}\rp^{\frac12}
        +\lp\frac{2\ss^2}{(2-\alpha\beta c_{\Sigma})k}\rp^{\frac12}\rp^2
        \lp 3+\frac{1}{\alpha(2-\alpha\beta c_{\Sigma})\mu k}\rp.
    \end{equation*}
\end{theorem}

\begin{proof}[Proof of Theorem \ref{thm:LSA_alternative}]
    It is a direct consequence of Proposition \ref{prop:LSA_alternative}
    and Lemma \ref{lem:cSk_LSA}.
\end{proof}

\begin{theorem}
    \label{thm:LSA_alternative2}
    Assume \ref{hypo:SPD}-\ref{hypo:ineq_LSA}.
    Define $\sst^2=\EE[|h_ky^*-b_k|^2]$.
    For $0<\alpha<\frac{2}{\beta_1}$,
    we have
    \begin{equation*}
        \EE\lb(\yo_k^{\top}-y^*)^{\top}\Sigma(\yo_k-y^*)\rb
        \leq
        \frac{2c_{\Sigma}|y_0-y^*|^2}{\alpha(2-\alpha\beta_1) k}
        +\frac{2c_{\Sigma}^2d}k\lp\sigma^2
        +\frac{\alpha\beta\sst^2}{(2-\alpha\beta_1)\mu }\rp
        \lp 3+\frac{1}{\alpha(2-\alpha\beta_1)\mu k}\rp.
    \end{equation*}
\end{theorem}

\begin{proof}[Proof of Theorem \ref{thm:LSA_alternative2}]
    It is a direct consequence of Proposition \ref{prop:LSA_alternative2}
    and Lemma \ref{lem:cSk_LSA}.
\end{proof}
The following proposition is an extension of Proposition \ref{prop:LSA}
to learning rates up to $\alpha<\frac{2}{\beta c_{\Sigma}}$.
\begin{proposition}
    \label{prop:LSA_alternative}
    Under Assumptions \ref{hypo:SPD}-\ref{hypo:ineq_LSA}
    without the third inequality,
    for $0<\alpha<\frac{2}{c_{\Sigma}\beta}$,
    we have
    \begin{equation*}
        \EE\lb(\yo_k^{\top}-y^*)^{\top}\Sigma(\yo_k-y^*)\rb
        \leq
        \frac{c_{\Sigma}|y_0-y^*|^2}{(2-\alpha\beta c_{\Sigma})\alpha k}
        +\lp\lp\frac{\beta|y_0-y^*|^2}{k}\rp^{\frac12}
        +\lp\frac{2\ss^2}{(2-\alpha\beta c_{\Sigma})k}\rp^{\frac12}\rp^2
        \cS_k.
    \end{equation*}
    where $\cS_k$ is defined as in Proposition \ref{prop:LSA}.
\end{proposition}

\begin{proof}[Proof of Proposition \ref{prop:LSA_alternative}]
    The three last point of the proof of Proposition \ref{prop:LSA}
    holds when removing the third inequality in \ref{hypo:ineq_LSA}.
    Therefore, to prove Proposition \ref{prop:LSA_alternative},
    it is sufficient to prove that 
    \begin{equation*}
        \EE\lb \wo_k^{\top} \Sigma\wo_k\rb
        \leq
        \frac{c_{\Sigma}|y_0-y^*|^2}{(2-\alpha\beta c_{\Sigma})\alpha k}
        +\frac{\beta|y_0-y^*|^2}{k}\cS_k.
    \end{equation*}
    Similarly as in the third step of the proof of
    Proposition \ref{prop:LSA},
    we define
    $(w^r_k)_{k,r\geq0}$
    such that $w_k=\sum_{r=0}^kw_k^r$
    and,
    for $r\geq0$ and $k\geq1$,
    \begin{align*}
        w^r_k
        &=
        (I_d-\alpha H)w^r_{k-1}
        +\chi_k^r
        \hspace*{0.4cm}
        \text{ with }
        \hspace*{0.4cm}
        w^r_{0}
        =
        w_0\indic[r=0],
        \\
        \chi_k^{r}
        &=
        -\alpha(h_k-H)\eta^{r-1}_{k-1}\indic[r\neq0].
    \end{align*}

    \emph{First step: getting a bound on $(\wo^0_k)^{\top}S\wo^0_k$.}

    \noindent
    Take $\aat=\frac2{\beta c_{\Sigma}}$,
    \eqref{eq:bound_Id-aaH} implies that 
    $(I_d-\aat H)(I_d-\aat H)^{\top}\leq I_d$.
    This implies that the operator norm of $(I_d-\aat H)$ 
    is upper bounded by one
    and that
    \begin{equation*}
        (I_d-\aat H)^{\top}(I_d-\aat H)\leq I_d.
    \end{equation*}
    This then implies
    \begin{align*}
        (I_d-\alpha H)^{\top}(I_d-\alpha H)
        &=
        \frac{\alpha^2}{\aat^2}
        (I_d-\aat H)^{\top}(I_d-\aat H)
        +\lp 1-\frac{\alpha^2}{\aat^2}\rp I_d
        -2\lp\alpha -\frac{\alpha^2}{\aat}\rp S
        \\
        &\leq
        I_d-\alpha(2-\alpha\beta c_{\Sigma})S.
    \end{align*}
    As a consequence,
    for $k\geq1$, 
    we obtain
    \begin{align*}
        |w^0_k|^2
        &=
        (w^0_{k-1})^{\top}
        (I_d-\alpha H)^{\top}(I_d-\alpha H)
        w^0_{k-1}
        \\
        &\leq
         |w^0_{k-1}|^2
         -\alpha(2-\alpha\beta c_{\Sigma})
         (w^0_{k-1})^{\top}Sw^0_{k-1}
    \end{align*}
    Taking the sum from $1$ to $k$ in the latter inequality and using that $u\mapsto u^{\top}Su$ is convex,
    we obtain
    \begin{equation*}
        (\wo^0_k)^{\top} S\wo_k^0
        \leq
        \frac1k\sum_{i=0}^{k-1}(w_i^0)^{\top} Sw_i^0
        \leq
        \frac1{\alpha k}\sum_{i=1}^{k}
        \lp|w^0_{k-1}|^2
        -|w^0_k|^2\rp
        \leq
        \frac{|w_0|^2}{\alpha(2-\alpha\beta c_{\Sigma}) k}.
    \end{equation*}
    This and $\Sigma\leq c_{\Sigma}S$
    yield
    \begin{equation*}
        (\wo^0_k)^{\top} \Sigma\wo_k^0
        \leq
        \frac{c_{\Sigma}|w_0|^2}{\alpha(2-\alpha\beta c_{\Sigma}) k}.
    \end{equation*}

    \emph{Second step: getting the bound on the bias term.}

    Then, using similar arguments as in the third
    step of the proof of Proposition \ref{prop:LSA},
    we have
    \begin{align*}
        \sum_{r=1}^k
        \EE\lb(\wo^r_k)^{\top}\Sigma\wo^r_k\rb
        &=
        \frac1{k^2}
        \sum_{j=0}^{k-1}
        \tr\lp
        \Sigma^{\frac12}M_{j}H^{-1}
        \EE\lb(h_{k-j}-H)
        \EE\lb w_{k-j-1}(w_{k-j-1})^{\top}\rb
        (h_{k-j}-H)^{\top}\rb 
        H^{-\top}M_{j}^{\top}\Sigma^{\frac12}\rp
        \\
        &\leq
        \frac{|y_0-y^*|^2}{k^2}
        \sum_{j=0}^{k-1}
        \tr\lp
        \Sigma^{\frac12}M_{j}H^{-1}
        \EE\lb(h_{k-j}-H)
        (h_{k-j}-H)^{\top}\rb 
        H^{-\top}M_{j}^{\top}\Sigma^{\frac12}\rp
        \\
        &\leq
        \frac{\beta|y_0-y^*|^2}{k^2}
        \sum_{j=0}^{k-1}
        \tr\lp
        \Sigma^{\frac12}M_{j}H^{-1}
        \Sigma
        H^{-\top}M_{j}^{\top}\Sigma^{\frac12}\rp
        \\
        &=
        \frac{\beta|y_0-y^*|^2}{k}\cS_k.
    \end{align*}
    In the end, we obtain
    \begin{equation*}
        \EE\lb\wo_k^{\top}\Sigma\wo_k\rb
        =
        \sum_{r=0}^k
        \EE\lb(\wo^r_k)^{\top}\Sigma\wo^r_k\rb
        \leq
        \frac{c_{\Sigma}|w_0|^2}{\alpha(2-\alpha\beta c_{\Sigma}) k}
        +\frac{\beta|y_0-y^*|^2}{k}\cS_k.
    \end{equation*}

    \emph{Third step: deal with the cross term and conclude.}

    \noindent
    We have
    \begin{equation*}
        \EE\lb\zo_k^{\top}\Sigma\zo_k\rb
        =
        \EE\lb\wo_k^{\top}\Sigma\wo_k\rb
        +\EE\lb\etao_k^{\top}\Sigma\etao_k\rb
        +2\EE\lb\wo_k^{\top}\Sigma\etao_k\rb,
    \end{equation*}
    where the last term is the cross term and satisfies
    \begin{align*}
        \EE\lb\wo_k^{\top}\Sigma\etao_k\rb
        &=
        \EE\lb\lp\sum_{r=0}^k\wo^r_k\rp^{\top}
        \Sigma\lp\sum_{r=1}^k\etao^r_k\rp\rb
        \\
        &=
        \EE\lb\lp\sum_{r=1}^k\wo^r_k\rp^{\top}
        \Sigma\lp\sum_{r=1}^k\etao^r_k\rp\rb
        \\
        &=
        \EE\lb\lp\sum_{r=1}^k\wo^r_k\rp^{\top}
        \Sigma\lp\sum_{r=1}^k\wo^r_k\rp\rb^{\frac12}
        \EE\lb\lp\sum_{r=1}^k\etao^r_k\rp^{\top}
        \Sigma\lp\sum_{r=1}^k\etao^r_k\rp\rb^{\frac12}
        \\
        &\leq
        \lp\frac{\beta|y_0-y^*|^2}{k}\cS_k\rp^{\frac12}
        \lp\frac{2\ss^2\cS_k}{(2-\alpha\beta c_{\Sigma})k}\rp^{\frac12}.
    \end{align*}
    The above inequalities imply the upper bound in Proposition \ref{prop:LSA_alternative}.
\end{proof}

\begin{proposition}
    \label{prop:LSA_alternative2}
    Under Assumptions \ref{hypo:SPD}-\ref{hypo:ineq_LSA},
    with $\sst^2=\EE[|h_ky^*-b_k|^2]$,
    for $0<\alpha<\frac{2}{\beta_1}$,
    we have
    \begin{equation*}
        \EE\lb(\yo_k^{\top}-y^*)^{\top}\Sigma(\yo_k-y^*)\rb
        \leq
        \frac{2c_{\Sigma}|y_0-y^*|^2}{\alpha(2-\alpha\beta_1) k}
        +\frac{2}k\lp\sigma^2
        +\frac{\alpha\beta\sst^2}{(2-\alpha\beta_1)\mu }\rp\cS_k
    \end{equation*}
\end{proposition}
\begin{proof}[Proof of Proposition \ref{prop:LSA_alternative2}]
    Let us use $z_k=y_k-y^*$ and the decomposition $z_k=w_k+\eta_k$,
    similarly as defined in the proof of Proposition~\ref{prop:LSA}.
    
    Repeating the first step of the proof of Proposition \ref{prop:LSA},
    we obtain
    \begin{equation*}
        \EE\lb \wo_k^{\top} \Sigma\wo_k\rb
        \leq
        \frac{c_{\Sigma}|w_0|^2}{\alpha(2-\alpha\beta_1) k}.
    \end{equation*}
    Therefore, using inequality
    $\EE\lb \zo_k^{\top} \Sigma\zo_k\rb
    \leq
    2\EE\lb \wo_k^{\top} \Sigma\wo_k\rb
    +2\EE\lb\etao_k^{\top}\Sigma\etao_k\rb$,
    to conclude it only remains to prove that
    \begin{equation*}
        \EE\lb\etao_k^{\top}\Sigma\etao_k\rb
        \leq
        \frac{1}k
        \lp\sigma^2
        +\frac{\alpha\beta\sst^2}{(2-\alpha\beta_1)\mu }\rp
        \cS_k.
    \end{equation*}

    \emph{First step: getting an upper bound on $\EE[|\eta_k|^2]$}

    \noindent
    Recall that, using a straightforward induction, we have $\EE[\eta_k]=0$ for $k\geq0$.
    We will now prove that 
    \begin{equation}
        \label{eq:etak^2}
        \EE\lb|\eta_k|^2\rb
        \leq
        \frac{\alpha\sst^2}{(2-\alpha\beta_1)\mu}.
    \end{equation}
    For $k=0$, the result holds since $\eta_0=0$.
    Then, for $k\geq1$, we assume that it holds for index $k-1$.
    Since $\eta_{k-1}$ and $(h_k,b_k)$ are independent, we obtain
    \begin{equation*}
        \EE\lb\eta_{k-1}^{\top}(I_d-\alpha h_k)^{\top}(h_ky^*-b_k)\rb
        =
        \EE\lb\eta_{k-1}^{\top}\rb\EE\lb(I_d-\alpha h_k)^{\top}(h_ky^*-b_k)\rb
        =
        0.
    \end{equation*}
    This and inequality \eqref{eq:bound_Id-aaHtop} imply that
    \begin{align*}
        \EE\lb|\eta_k|^2\rb
        &=
        \EE\lb|(I_d-\alpha h_k)\eta_{k-1}|^2\rb
        +\alpha^2
        \EE\lb|h_ky^*-b_k|^2\rb
        \\
        &\leq
        \EE\lb\eta_{k-1}^{\top}(I_d-\alpha(2-\alpha\beta_1)S)\eta_{k-1}\rb
        +\alpha^2\sst^2
        \\
        &\leq
        (1-\alpha(2-\alpha\beta_1)\mu)
        \EE\lb|\eta_{k-1}|^2\rb
        +\alpha^2\sst^2
        \\
        &\leq
        \frac{\alpha\sst^2}{(2-\alpha\beta_1)\mu},
    \end{align*}
    which concludes the induction.

    \emph{Second step: obtaining the convergence rate.}

    \noindent
    Let us use the same decomposition $\eta_k=\sum_{r=1}^k\eta_k^r$ as introduced in the
    proof of Proposition \ref{prop:LSA}.
    Similarly as in the third step of the latter proof, we get
    \begin{equation*}
        \EE\lb(\etao^1_k)^{\top}\Sigma\etao^1_k\rb
        \leq
        \frac{\ss^2}{k^2}
        \sum_{j=1}^{k-1}
        \tr\lp
        \Sigma^{\frac12}M_{j}H^{-1}\Sigma
        H^{-\top}M_{j}^{\top}\Sigma^{\frac12}\rp
        =
        \frac{\ss^2}{k}\cS_k,
    \end{equation*}
    and
    \begin{align*}
        \sum_{r=2}^k
        \EE\lb(\etao^r_k)^{\top}\Sigma\etao^r_k\rb
        &=
        \frac1{k^2}
        \sum_{j=0}^{k-1}
        \tr\lp
        \Sigma^{\frac12}M_{j}H^{-1}
        \EE\lb(h_{k-j}-H)
        \EE\lb\eta_{k-j-1}(\eta_{k-j-1})^{\top}\rb
        (h_{k-j}-H)^{\top}\rb 
        H^{-\top}M_{j}^{\top}\Sigma^{\frac12}\rp
        \\
        &\leq
        \frac{\alpha\sst^2}{(2-\alpha\beta_1)\mu k^2}
        \sum_{j=0}^{k-1}
        \tr\lp
        \Sigma^{\frac12}M_{j}H^{-1}
        \EE\lb(h_{k-j}-H)(h_{k-j}-H)^{\top}\rb 
        H^{-\top}M_{j}^{\top}\Sigma^{\frac12}\rp
        \\
        &\leq
        \frac{\alpha\beta\sst^2}{(2-\alpha\beta_1)\mu k^2}
        \sum_{j=0}^{k-1}
        \tr\lp
        \Sigma^{\frac12}M_{j}
        H^{-1}\Sigma H^{-\top}
        M_{j}^{\top}\Sigma^{\frac12}\rp
        \\
        &=
        \frac{\alpha\beta\sst^2}{(2-\alpha\beta_1)\mu k}\cS_k,
        \end{align*}
        where we used 
        $\EE\lb\eta_{j}(\eta_{j})^{\top}\rb\leq 
        \EE\lb|\eta_{j}|^2\rb I_d
        \leq
        \frac{\alpha\sst^2}{(2-\alpha\beta_1)\mu k^2}I_d$
        to get the second line,
        and the first inequality of Assumption \ref{hypo:ineq_LSA} for the third line.
        This concludes the proof.
\end{proof}

\subsection{Some results from linear algebra}

\begin{lemma}
    \label{lem:generic_bounds}
    Under assumptions \ref{hypo:SPD}-\ref{hypo:ineq_LSA},
    for $\alpha\leq \frac{2}{\beta c_{\Sigma}}$, we have
    \begin{align}
        \label{eq:bound_Id-aaH}
        (I_d-\alpha H)(I_d-\alpha H)^{\top}
        &\leq
        \EE\lb(I_d-\alpha h)(I_d-\alpha h)^{\top}\rb
        \leq
        I_d-\alpha(2-\alpha\beta c_{\Sigma}) S,
    \end{align}
    Moreover, for $\alpha< \frac2{\beta_1}$, we get
    \begin{align}
        \label{eq:bound_Id-aaHtop}
        (I_d-\alpha H)^{\top}(I_d-\alpha H)
        &\leq
        \EE\lb(I_d-\alpha h)^{\top}(I_d-\alpha h)\rb
        \leq
        I_d-\alpha(2-\alpha\beta_1) S,
        \\
        \label{eq:bound_Id-aaH2}
        (I_d-\alpha H)(I_d-\alpha H)^{\top}
        &\leq
        I_d-\alpha(2-\alpha\beta_1) S,
        \\
        \label{eq:bound_sum_sym}
        \sum_{i=0}^{k-1}
        (I_d-\alpha H^{\top})^iS(I_d-\alpha H)^i
        &\leq
        \frac1{\alpha(2-\alpha\beta_1)}I_d.
    \end{align}
\end{lemma}
\begin{proof}
    The first inequality in \eqref{eq:bound_Id-aaH} is straightforward.
    Then, we have,
    \begin{equation*}
        \EE\lb(I_d-\alpha h)(I_d-\alpha h)^{\top}\rb
        =
        I_d-2\alpha S+\alpha^2\EE[hh^{\top}]
        \leq
        I_d-\alpha(2-\alpha\beta c_{\Sigma}) S,
    \end{equation*}
    using $\EE[hh^{\top}]\leq \beta \Sigma\leq \beta c_{\Sigma} S$.
    Inequality \eqref{eq:bound_Id-aaHtop} can be proved 
    with a similar computation.

    In particular,
    the operator norm of $I_d-\frac{2}{\beta_1} H$ is bounded by one,
    so is the one of $I_d-\frac{2}{\beta_1} H^{\top}$,
    i.e.,
    \begin{equation*}
        \lp I_d-\frac{2}{\beta_1} H\rp
        \lp I_d-\frac{2}{\beta_1} H\rp^{\top}
        \leq
        I_d.
    \end{equation*}
    This allows us to prove 
    \eqref{eq:bound_Id-aaH2} as follows,
    for $\alpha\leq \frac{2}{\beta_1}$,
    \begin{align*}
        (I_d-\alpha H)(I_d-\alpha H)^{\top}
        &=
        I_d-2\alpha S+\alpha^2HH^{\top}
        \\
        &=
        \frac{\alpha^2\beta_1^2}4
        \lp I_d-\frac{2}{\beta_1}H\rp
        \lp I_d-\frac{2}{\beta_1}H\rp^{\top}
        +\lp 1-
        \frac{\alpha^2\beta_1^2}4\rp I_d
        -2\lp\alpha-\alpha^2\beta_1\rp S
        \\
        &\leq
        I_d
        +\lp 1-
        \frac{\alpha^2\beta_1^2}4\rp I_d
        -\lp2\alpha-\alpha^2\beta_1\rp S
        \\
        &=
        I_d-\alpha (2-\alpha\beta_1)S.
    \end{align*}

    Then,
    using the notation $N_i=(I_d-\alpha H)^i$, 
    we have
    \begin{align*}
        I_d
        \geq
        I_d-N_k^{\top}N_k
        &=
        \sum_{i=0}^{k-1}
        N_i^{\top}N_i
        -N_{i+1}^{\top}N_{i+1}
        \\
        &=
        \sum_{i=0}^{k-1}
        N_i^{\top}N_i
        -N_i^{\top}(I_d-\alpha H^{\top})
        (I_d-\alpha H)N_i
        \\
        &=
        \sum_{i=0}^{k-1}
        N_i^{\top}(I_d-(I_d-\alpha H^{\top})(I_d-\alpha H))N_i
        \\
        &=
        \sum_{i=0}^{k-1}
        N_i^{\top}(2\alpha S-\alpha^2H^{\top}H)N_i
        \\
        &\geq
        \alpha(2-\alpha\beta_1)\sum_{i=0}^{k-1}N_i^{\top}SN_i,
    \end{align*}
    where we used $\alpha H^{\top}H\leq (2-\alpha\beta_1)S$ to get the last line.
\end{proof}

\begin{lemma}[Heinz's inequality]
    \label{lem:SH-top}
    Let $H\in\RR^{d\times d}$ be such that $S:=\frac{H+H^{\top}}2>0$,
    then $H^{-1}+H^{-\top}$ is positive definite and 
    $H^{-1}+H^{-\top}\leq 2S^{-1}$.
    Moreover, we have
    $\tr(SH^{-\top})\leq d$.
\end{lemma}
\begin{proof}
    First, let us prove that $H^{-1}+H^{-\top}$ is 
    positive definite.
    Take $x\in\RR^d\backslash\{0\}$,
    define $y=H^{-1}x$ we have
    $$
        x^{\top}(H^{-1}+H^{-\top})x
        =
        2x^{\top}H^{-\top}x
        =
        2y^{\top}Hy
        =
        y^{\top}(H+H^{\top})y
        >0.
    $$
    Then, let us prove that $\tr(SH^{-\top}) \leq d$ is a consequence
    of $H^{-1}+H^{-\top}\leq 2S^{-1}$:
    $$
        \tr(SH^{-\top})
        =
        \frac12\tr(S^{\frac12}(H^{-1}+H^{-\top})S^{\frac12})
        \leq
        \tr(S^{\frac12}S^{-1}S^{\frac12})
        =
        d.
    $$
    Now, observe that
    \begin{align*}
        \lp S^{\frac12}(H^{-1}+H^{-\top})S^{\frac12}\rp^2
        &=
        2S^{\frac12}H^{-1}SH^{-\top}S^{\frac12}
        +2S^{\frac12}H^{-\top}SH^{-1}S^{\frac12}
        \\
        &\hspace*{2cm}
        -\lp S^{\frac12}(H^{-1}-H^{-\top})S^{\frac12}\rp\lp S^{\frac12}(H^{-1}-H^{-\top})S^{\frac12}\rp^{\top}
        \\
        &\leq
        2S^{\frac12}H^{-1}SH^{-\top}S^{\frac12}
        +2S^{\frac12}H^{-\top}SH^{-1}S^{\frac12}.
    \end{align*}
    Therefore, to conclude, it is sufficient to prove
    $S^{\frac12}H^{-1}SH^{-\top}S^{\frac12}\leq I_d$
    and
    $S^{\frac12}H^{-1}SH^{-\top}S^{\frac12}\leq I_d$.
    For the former inequality, invert the matrix
    and use the notation $A=\frac{H-H^{\top}}2$,
    we get
    \begin{align*}
        (S^{\frac12}H^{-1}SH^{-\top}S^{\frac12})^{-1}
        &=
        S^{-\frac12}H^{\top}S^{-1}HS^{-\frac12}
        \\
        &=
        S^{-\frac12}(S-A)S^{-1}(S+A)S^{-\frac12}
        \\
        &=
        I_d-S^{-\frac12}AS^{-1}AS^{-\frac12}
        \\
        &=
        I_d+\lp S^{-\frac12}AS^{-\frac12}\rp \lp S^{-\frac12}AS^{-\frac12}\rp^{\top}
        \\
        &\geq 
        I_d,
    \end{align*}
    which, indeed, implies
    $S^{\frac12}H^{-1}SH^{-\top}S^{\frac12}\leq I_d$.
    Then, applying the later inequality to $\Ht=H^{\top}$,
    we obtain $S^{\frac12}H^{-1}SH^{-\top}S^{\frac12}\leq I_d$.
    This concludes the proof.
\end{proof}

\section{Proofs specific to TD(0)}

\subsection{Proof of Proposition \ref{prop:Kreiss}}
\begin{proof}
Fix $k\geq2$ and take $c_k\in(0,1)$ such that $c_k^2=1-\frac1k$.
In particular, observe that 
\begin{equation*}
    c_k^{-2(k-1)}=\lp\frac{k}{k-1}\rp^{k-1}=\lp1+\frac1{k-1}\rp^{k-1}\leq e.
\end{equation*}
Recall that $Q$ is defined by
\begin{equation*}
    Q=I_d-\alpha\Sigma_0^{\frac12}H\Sigma_0^{-\frac12}
                = I_d-\alpha \Sigma_0(I_d-\gamma U),
\end{equation*}
and $\cS_k$ satisfies
\begin{equation}
    \label{eq:aux_cSk}
    \cS_k
    =
    \tr\lp
    \frac1k\sum_{\ell=0}^{k-1}
    (I_d-\gamma U)^{-\top}
    (I_d-Q^{\ell})^{\top}(I_d-Q^{\ell})
    (I_d-\gamma U)^{-1}
    \rp.
\end{equation}
Define $\Qt=c_kQ$, we have
\begin{equation*}
        c_k^{2(k-1)}
        \sum_{\ell=0}^{k-1}
        (I_d-Q^{\ell\top})(I_d-Q^{\ell})
        \leq
        \sum_{\ell=0}^{k-1}
        c_k^{2\ell}
        (I_d-Q^{\ell\top})(I_d-Q^{\ell})
\end{equation*}
Then, Lemma \ref{lem:series} yields
\begin{equation*}
    \begin{aligned}
    \sum_{\ell=0}^{k-1}
    (I_d-Q^{\ell\top})&(I_d-Q^{\ell})
    \\
    &\leq
    \frac1{2i\pi c_k^{2(k-1)}}
    \int_{|z|=1}
    \zo\lp (\zo-c_k)^{-1}I_d-(\zo I_d-\Qt)^{-\top}\rp
    \lp (z-c_k)^{-1}I_d-(z I_d-\Qt)^{-1}\rp dz
    \\
    &=
    \frac1{2\pi c_k^{2(k-1)}}
    \int_{-\pi}^{\pi}
    \lp (e^{-i\psi}-c_k)^{-1}I_d
    -(e^{-i\psi} I_d-\Qt)^{-\top}\rp
    \lp (e^{i\psi}-c_k)^{-1}I_d-(e^{i\psi} I_d-\Qt)^{-1}\rp d\psi
    \\
    &\leq
    \frac{e}{2\pi}
    \int_{-\pi}^{\pi}
    \lp (e^{-i\psi}-c_k)^{-1}I_d
    -(e^{-i\psi} I_d-\Qt)^{-\top}\rp
    \lp (e^{i\psi}-c_k)^{-1}I_d-(e^{i\psi} I_d-\Qt)^{-1}\rp d\psi
    \end{aligned}
\end{equation*}
In the following, we will use both the notation $z$ or $e^{i\psi}$,
depending on the situation.
This and \eqref{eq:aux_cSk} imply
the following upper bound on $\cS_k$
\begin{equation}
    \label{eq:bound_cS}
    \cS_k
    \leq
    \frac{e}{2\pi k}
    \int_{-\pi}^{\pi}
    \norm[2]{\lp(e^{i\psi}-c_k)^{-1}I_d-(e^{i\psi}I_d-\Qt)^{-1}\rp(I_d-\gamma U)^{-1}}{F}d\psi.
\end{equation}
We are going to separate the latter integral into three parts
corresponding to: 
\begin{enumerate}[$(i)$]
    \item
        \label{case:1}
        $\cR(z)\geq1-(1-\gamma)^2$,
        i.e., $|\psi|\leq\psi_{\gamma}$
    \item 
        \label{case:2}
        $|\cR(z)|<1-(1-\gamma)^2$,
        i.e., $\psi_{\gamma}<|\psi|<\pi-\psi_{\gamma}$
    \item 
        \label{case:3}
        $\cR(z)\leq-1+(1-\gamma)^2$,
        i.e., $|\psi|\geq\pi-\psi_{\gamma}$
\end{enumerate}
where
$\mathcal{R}(z)$ is the real part of $z$ and
$\psi_{\gamma}$ is defined by $\psi_{\gamma}=\arccos(1-(1-\gamma)^2)$.

\emph{First step: dealing with case \ref{case:1}.}

\noindent
Observe that
\begin{align*}
    z-\Qt
    &=
    (z-c_k)I_d+c_k\alpha\Sigma_0(I_d-\gamma U)
    \\
    &=
    (z-c_k)\lp I_d+\frac{c_k\alpha}{|z-c_k|}\frac{\zo-c_k}{|z-c_k|}\Sigma_0(I_d-\gamma U)\rp
    \\
    &=
    (z-c_k)\lp I_d+\zt\Sigma(I_d-\gamma U)\rp,
\end{align*}
with $\zt=\frac{\zo-c_k}{|z-c_k|}$
and $\Sigma=\frac{c_k\alpha}{|z-c_k|}\Sigma_0$.
Using Lemma \ref{lem:bound_case1}, for any $x\in\CC^d$,
we have
\begin{align*}
    \labs\lp(z-c_k)^{-1}I_d-(z-\Qt)^{-1}\rp(I_d-\gamma U)^{-1}x\rabs
    &=
    |z-c_k|^{-1}\labs\lp I_d-(I_d+\zt\Sigma(I_d-\gamma U))^{-1}\rp
    (I_d-\gamma U)^{-1}x\rabs
    \\
    &\leq
    \frac{2}{1-\gamma^2}|z-c_k|^{-1},
\end{align*}
which holds only if 
$-\cR(\zt)=-\cR(\frac{\zo-c_k}{|z-c_k|})\leq\frac{1-\gamma}{\sqrt2}$,
that we prove as follows,
\begin{align*}
    -\cR\lp\frac{\zo-c_k}{|z-c_k|}\rp
    =
    \frac{c_k-\cR(z)}{\sqrt{(c_k-\cR(z))^2+\cI(z)^2}}
    &\leq
    \frac{1-\cR(z)}{\sqrt{(1-\cR(z))^2+\cI(z)^2}}
    \\
    &=
    \frac{1-\cR(z)}{\sqrt{(1-\cR(z))^2+1-\cR(z)^2}}
    \\
    &=
    \frac{1-\cR(z)}{\sqrt{2-2\cR(z)}}
    \\
    &=
    \sqrt{\frac{1-\cR(z)}2}
    \\
    &\leq
    \sqrt{\frac{1-(1-(1-\gamma)^2)}2}
    =\frac{1-\gamma}{\sqrt{2}},
\end{align*}
where the first line is due to $u\mapsto\frac{u}{u^2+v^2}$ is non-decreasing for any $v\in\RR$,
and the last one comes from $\cR(z)\geq 1-(1-\gamma)^2$.
Let us integrate for $\psi\in[-\psi_{\gamma},\psi_{\gamma}]$
and use $\|\cdot\|^2_F\leq d\|\cdot\|^2_{\rm op}$,
\begin{align*}
    \frac{e}{2\pi}
    \int_{-\psi_{\gamma}}^{\psi_{\gamma}}
    \norm[2]{\lp(e^{i\psi}-c_k)^{-1}I_d-(e^{i\psi}I_d-\Qt)^{-1}\rp(I_d-\gamma U)^{-1}}{F}d\psi
    &\leq
    \frac{e}{2\pi}
    \int_{-\psi_k}^{\psi_k}
    \norm[2]{\frac{4I_d}{(1-\gamma^2)^2|e^{i\psi}-c_k|^2}}{F}\,d\psi
    \\
    &\leq
    \frac{2de}{\pi(1-\gamma^2)^2}
    \int_{-\pi}^{\pi}
    \frac{1}{|e^{i\psi}-c_k|^2}d\psi
    \\
    &\leq
    \frac{4de}{(1-\gamma^2)^2(1-c_k^2)}
    \\
    &=
    \frac{kde}{(1-\gamma)^2}\frac{4}{(1+\gamma)^2}
    \\
    &\leq
    \frac{kde}{(1-\gamma)^2}(4-3\gamma),
\end{align*}
where the third line is obtained using the residue Theorem,
and the last line uses
$\frac4{(1+\gamma)^2}\leq 4-3\gamma$.
This concludes the case \ref{case:1}.

\emph{Second step: dealing with case \ref{case:2}.}

\noindent
For cases \ref{case:2} and \ref{case:3}, we will use the following inequalities
\begin{equation}
\label{eq:case23}
\begin{aligned}
    \norm{\lp(z-c_k)^{-1}I_d-(z-\Qt)^{-1}\rp(I_d-\gamma U)^{-1}}{F}
    &\leq
    \norm{(z-c_k)^{-1}I_d-(z-\Qt)^{-1}}{F}
    \norm{(I_d-\gamma U)^{-1}}{\rm op}
    \\
    &\leq
    \frac{\sqrt{d}}{1-\gamma}
    \norm{(z-c_k)^{-1}I_d-(z-\Qt)^{-1}}{{\rm op}}
    \\
    &\leq
    \frac{\sqrt{d}}{1-\gamma}
    \lp|z-c_k|^{-1}
    +\norm{(z-\Qt)^{-1}}{{\rm op}}\rp,
\end{aligned}
\end{equation}
where $\norm{\cdot}{{\rm op},\RR^d}$ is the operator norm
restricted to vectors on $\RR^d$.
Define $Q_R:=\cR(z)-\Qt=\cR(z)-c_k+\alpha c_k\Sigma_0(I_d-\gamma U)$,
so that $z-\Qt=Q_R+i\cI(z)$.
Take $x\in\RR^d$ and $y_1+iy_2=y=(z-\Qt)^{-1}x$,
we have
\begin{equation*}
    \RR^d\ni x
    =
    (z-\Qt)y
    =
    (Q_R+i\cI(z))(y_1+iy_2)
    =
    Q_Ry_1-\cI(z)y_2+i(Q_Ry_2+\cI(z)y_1),
\end{equation*}
which implies
    $Q_Ry_2=-\cI(z)y_1$,
so that we obtain
\begin{equation}
\label{eq:aux_case2}
\begin{aligned}
    |(z-\Qt)y|^2
    &=
    (y_1-iy_2)^{\top}(Q_R^{\top}-i\cI(z))(Q_R+i\cI(z))(y_1+iy_2)
    \\
    &=
    |Q_Ry_1|^2
    +|Q_Ry_2|^2
    +\cI(z)^2|y|^2
    -2\cI(z)y_1^{\top}(Q_R^{\top}-Q_R)y_2
    \\
    &=
    |Q_Ry_1|^2
    +\cI(z)^2(2|y_1|^2+|y_2|^2)
    -2\alpha c_k\gamma\cI(z)y_1^{\top}(\Sigma_0U-U^{\top}\Sigma_0)y_2
    \\
    &\geq
    \cI(z)^2(2|y_1|^2+|y_2|^2)
    -\frac{8\gamma\cI(z)^2}{(1+\gamma)^2}|y_1||y_2|
    \\
    &\geq
    \cI(z)^2(2|y_1|^2+|y_2|^2)
    -2\cI(z)^2|y_1||y_2|
    \\
    &\geq
    \cI(z)^2(2|y_1|^2+|y_2|^2)
    -\cI(z)^2\lp\frac32|y_1|^2+\frac23|y_2|^2\rp
    \\
    &=
    \cI(z)^2\lp\frac12|y_1|^2+\frac13|y_2|^2\rp
    \geq
    \frac{\cI(z)^2}3|y|^2
\end{aligned}
\end{equation}
where we used $|Q_Ry_2|^2=\cI(z)^2|y_1|^2$
and $Q_R^{\top}-Q_R=\alpha c_k\gamma(\Sigma_0U-U^{\top}\Sigma_0)$
to get the third line;
$|\cI(z)|\geq\sqrt{1-(1-(1-\gamma)^2)^2}\geq\sqrt{1-(1-(1-\gamma)^2)}=1-\gamma$,
$\alpha<\frac{2(1-\gamma)}{(1+\gamma)^2}\leq \frac{2|I(z)|}{(1+\gamma)^2}$, $c_k\leq1$,
$\|\Sigma_0U\|_{\rm op}\leq\|\Sigma_0\|_{\rm op}\|U\|_{\rm op}\leq1$
to obtain the fourth line;
$4\gamma\leq(1+\gamma)^2$ to get the fifth line;
and the Young's inequality
$2|y_1||y_2|\leq\frac32|y_1|^2+\frac23|y_2|^2$ to obtain the last one.
This implies
\begin{equation*}
    \norm{(z-\Qt)^{-1}}{\rm op}
    \leq
    \sqrt{\frac3{\cI(z)^2}}
    =
    \frac{\sqrt{3}}{|\sin(\psi)|}
\end{equation*}
for $z=e^{i\psi}$ such that $|\cR(z)|\leq\gamma(2-\gamma)$.
Similarly, we have $|z-c_k|^{-1}\leq \frac1{|\sin(\psi)|}$.
Then using the above computations and $(\sqrt3+1)^2\leq 8$
we obtain
\begin{align*}
    &\frac{e}{2\pi}
    \int_{|\psi|\in(\psi_{\gamma},\pi-\psi_{\gamma})}
    \norm[2]{\lp(e^{i\psi}-c_k)^{-1}I_d-(e^{i\psi}I_d-\Qt)^{-1}\rp(I_d-\gamma U)^{-1}}{F}d\psi
    \\
    &\hspace*{3cm}\leq
    \frac{8de}{\pi(1-\gamma)^2}
    \int_{\psi_{\gamma}}^{\pi-\psi_{\gamma}}
    \frac1{\sin(\psi)^2} d\psi
    \\
    &\hspace*{3cm}\leq
    \frac{8de}{\pi(1-\gamma)^2}
    \bigl[-{\rm cotan}(\psi)\bigr]_{\psi_{\gamma}}^{\pi-\psi_{\gamma}}
    \\
    &\hspace*{3cm}=
    \frac{16de}{\pi(1-\gamma)^2}
    {\rm cotan}(\psi_{\gamma})
    \\
    &\hspace*{3cm}=
    \frac{16de}{\pi(1-\gamma)^2}
    \frac{\gamma(2-\gamma)}{\sqrt{1-\gamma^2(2-\gamma)^2}}
    \\
    &\hspace*{3cm}\leq
    \frac{16de}{\pi(1-\gamma)^2}
    \frac{1+(1-\gamma)}{\sqrt{(1-\gamma)^2(1+2\gamma-\gamma^2)}}\gamma
    \\
    &\hspace*{3cm}\leq
    \frac{16de}{\sqrt{2}\pi(1-\gamma)^3}
    (1+(1-\gamma))
    \\
    &\hspace*{3cm}\leq
    \frac{4de}{(1-\gamma)^3}
    (1+(1-\gamma)),
\end{align*}
where we used that the primitive of $\frac{-1}{\sin^2}$ is ${\rm cotan}$,
${\rm cotan}(\arccos(a))=\frac{a}{\sqrt{1-a^2}}$ for $a\in(0,1)$,
$\frac{\gamma}{\sqrt{1+2\gamma-\gamma^2}}\leq\frac{\gamma}{\sqrt{1+\gamma}}\leq \frac1{\sqrt2}$
and $\frac{16}{\sqrt{2}\pi}\leq4$.

\emph{Third step: dealing with case \ref{case:3}.}

\noindent
Once again take $x\in\CC^d$ and
$y=(z-\Qt)^{-1}x\in\CC^d$, we have
\begin{align*}
    |x|=|(z-\Qt)y|
    &\geq
    |(z-c_k+c_k\alpha\Sigma_0)y|
    -\alpha\gamma|\Sigma_0Uy|
    \\
    &\geq
    \lp\min_{\lambda\in {\rm Sp}(\Sigma_0)}
    |z-c_k+c_k\alpha\lambda|
    -\alpha\gamma\rp|y|
    \\
    &\geq
    \lp\min_{u\in[0,1]}
    |z-c_k+c_k\alpha u|
    -\alpha\gamma\rp|y|
\end{align*}
Then, let us focus on the term
$|z-c_k+c_k\alpha u|$.
Define $s=1-c_k^{-1}\cR(z)\in[1,1+c_k^{-1}]$, we get
\begin{align*}
    |z-c_k+c_k\alpha u|^2
    &=
    (\cR(z)-c_k+c_k\alpha u)^2+\cI(z)^2
    \\
    &=
    c_k^2(\alpha u-s)^2+1-c_k^2(s-1)^2
    \\
    &=
    c_k^2(\alpha^2u^2-2s\alpha u+2s-1)
    +1
    =
    f(u,s),
\end{align*}
where $f(u,s):=
c_k^2(\alpha^2u^2+2s(1-\alpha u)-1)+1$ for $u\in[0,1]$
and $s\in[1,1+c_k^{-1}]$.
We distinguish two cases.
First, if $\alpha u\leq1$, we have
\begin{equation*}
    \min_sf(u,s)
    =
    f(u,1)
    =
    c_k^2(\alpha^2u^2+
    2(1-\alpha u)-1)
    +1
    =
    c_k^2(\alpha u-1)^2+1
    \geq
    1,
\end{equation*}
which implies 
\begin{equation*}
    |x|
    \geq
    (1-\alpha\gamma)|y|
    \geq
    \lp1-\frac{2\gamma(1-\gamma)}{(1+\gamma)^2}\rp|y|
    \geq
    \lp1-\frac{1-\gamma}{2}\rp|y|
    \geq
    \frac{|y|}2,
\end{equation*}
where we used the above computations, $\alpha<\frac{2(1-\gamma)}{(1+\gamma)^2}$
and $4\gamma\leq(1+\gamma)^2$.
Second, assume $\alpha u>1$, we have
\begin{align*}
    \min_sf(u,s)
    =
    f(u,1+c_k^{-1})
    &=
    c_k^2(\alpha^2u^2+2(1+c_k^{-1})(1-\alpha u)-1)+1
    \\
    &=
    c_k^2(\alpha^2u^2
    -2(1+c_k^{-1})\alpha u
    +2+2c_k^{-1}-1+c_k^{-2})
    \\
    &=
    c_k^2(\alpha^2u^2
    -2(1+c_k^{-1})\alpha u
    +(1+c_k^{-1})^2)
    \\
    &=
    c_k^2(\alpha u-(1+c_k^{-1}))^2
    \\
    &=
    (1-c_k(\alpha u-1))^2
    \\
    &\geq
    (1-(\alpha u-1))^2
    \\
    &=
    (2-\alpha u)^2
    \geq
    (2-\alpha)^2,
\end{align*}
where we used $1-c_k(\alpha u-1)\geq0$,
$\alpha u-1>0$ and $c_k\leq1$ to get the sixth line;
$\alpha<2$ and $u\leq1$ to get the last one.
This implies 
\begin{equation*}
    |x|
    \geq
    (2-\alpha-\gamma\alpha)|y|
    =
    (2-(1+\gamma)\alpha)|y|,
\end{equation*}
where we have $(1+\gamma)\alpha<\frac{2(1-\gamma)}{1+\gamma}\leq 2$.
In all cases, we obtain
\begin{equation*}
    |(z-\Qt)^{-1}x|
    \leq
    \min\lp\frac12,2-(1+\gamma)\alpha\rp^{-1}|x|
    =
    \max\lp2,\frac1{2-(1+\gamma)\alpha}\rp|x|
    \leq
    \lp2 +\frac1{2-(1+\gamma)\alpha}\rp|x|
\end{equation*}
This, \eqref{eq:case23} and $|z-c_k|\geq\sqrt{1+c_k^2}\geq1$, since $\cR(z)\leq0$, imply
\begin{align*}
    \norm[2]{\lp(z-c_k)^{-1}I_d-(z-\Qt)^{-1}\rp(I_d-\gamma U)^{-1}}{F}
    &\leq
    \frac{2d}{(1-\gamma)^2}
    \lp|z-c_k|^{-2}
    +\norm[2]{(z-\Qt)^{-1}}{{\rm op}}\rp
    \\
    &\leq
    \frac{2d}{(1-\gamma)^2}
    \lp3 +\frac1{2-(1+\gamma)\alpha}\rp,
\end{align*}
and then
\begin{align*}
    \frac{e}{2\pi}
    \int_{|\psi|\in[\pi-\psi_{\gamma},\pi]}
    \norm[2]{\lp(e^{i\psi}-c_k)^{-1}I_d-(e^{i\psi}I_d-\Qt)^{-1}\rp(I_d-\gamma U)^{-1}}{F}d\psi
    \leq
    \frac{2de}{(1-\gamma)^2}
    \lp3 +\frac1{2-(1+\gamma)\alpha}\rp.
\end{align*}

\emph{Final step: concluding.}
From \eqref{eq:bound_cS} and the three previous steps of the proof,
we obtain
\begin{align*}
    \cS_k
    &\leq
    \frac{de}{(1-\gamma)^2}(1+3(1-\gamma))
    +\frac{4de}{(1-\gamma)^3k}
    (1+(1-\gamma))
    +\frac{2de}{(1-\gamma)^2k}
    \lp3 +\frac1{2-(1+\gamma)\alpha}\rp
    \\
    &\leq
    \frac{de}{(1-\gamma)^2}
    \lp 1+ \frac{4}{k(1-\gamma)}
    +3(1-\gamma)
    +\frac{10}k+\frac1{k(1-\frac{(1+\gamma)\alpha}2)}
    \rp.
\end{align*}

\end{proof}

\begin{lemma}
    \label{lem:bound_case1}
    Let $\Sigma\in\RR^{d\times d}$ be symmetric positive semi-definite,
    let $U\in\RR^{d\times d}$ be such that $\norm{U}{\rm op}\leq1$,
    $\gamma\in[0,1)$ and $z\in\CC$ with $|z|=1$ and $\cR(z)\geq -\frac{1-\gamma}{\sqrt{2}}$.
    We have, for $x\in\CC^d$,
    \begin{equation*}
        \labs\lp I_d-(I_d+z\Sigma(I_d-\gamma U))^{-1}\rp
        (I_d-\gamma U)^{-1}x\rabs
        \leq
        \frac{2}{1-\gamma^2}|x|.
    \end{equation*}
\end{lemma}

\begin{proof}
    Define $A=I_d-\gamma U$ and $M=I_d+z\Sigma A$,
    we have
    \begin{align*}
        \lp I_d-(I_d+z\Sigma(I_d-\gamma U))^{-1}\rp
        (I_d-\gamma U)^{-1}
        &=
        (I_d-M^{-1})A^{-1}
        \\
        &=
        M^{-1}(M-I_d)A^{-1}
        \\
        &=
        z(I_d+z\Sigma A)^{-1}\Sigma 
        \\
        &=
        z\Sigma(I_d+zA\Sigma)^{-1}.
    \end{align*}
    Therefore, to conclude, it is sufficient to prove that,
    for an arbitrary $x\in\CC^d$,
    \begin{equation}
        \label{eq:near_one_aux}
        |\Sigma(I_d+zA\Sigma)^{-1}x|
        =
        |z\Sigma(I_d+zA\Sigma)^{-1}x|
        \leq
        \frac{2}{1-\gamma^2}|x|,
    \end{equation}
    The remainder of the proof is dedicated to prove the latter inequality.
    Take $u=(I_d+zA\Sigma)^{-1}x$ and $r=\Sigma u$.
    Let us start with the subsequent observation, 
    \begin{equation*}
        |\ro^{\top}Ar-|r|^2|
        =\gamma|\ro^{\top}Ur|
        \leq \gamma|r|^2,
    \end{equation*}
    so that $\frac{\ro^{\top}Ar}{|r|^2}$
    belongs to $\overline{B}_{\CC}(1,\gamma)$ the closed ball
    of center $1$ and radius $\gamma$.
    Then, using $x=\Sigma^{-1}r+zAr$, we have
    \begin{align*}
        \frac{|\ro^{\top}x|}{|r|^2}
        =
        \labs
        \frac{\ro^{\top}\Sigma^{-1} r}{|r|^2}
        +z\frac{\ro^{\top}A r}{|r|^2}\rabs
        =
        \labs
        \frac{\ro^{\top}A r}{|r|^2}
        +\zo\frac{\ro^{\top}\Sigma^{-1} r}{|r|^2}
        \rabs
        \geq
        d_{|\cdot|}(\overline{B}_{\CC}(1,\gamma),-\zo\RR_+),
    \end{align*}
    where we used $\ro^{\top}\Sigma^{-1} r\geq0$,
    and $d_{|\cdot|}\lp\overline{B}_{\CC}(1,\gamma),-\zo\RR_+\rp$
    is the distance between the sets $\overline{B}_{\CC}(1,\gamma)$
    and $-\zo\RR_+$.
    Then, using Cauchy-Schwarz' inequality $|\ro^{\top}x|\leq |r||x|$,
    we obtain
    \begin{align*}
        |x|\geq \frac{|\ro^{\top}x|}{|r|}
        \geq
        d_{|\cdot|}(\overline{B}_{\CC}(1,\gamma),-\zo\RR_+)|r|.
    \end{align*}
    It only remains to derive an appropriate lower bound for the latter distance.
    Let us consider two cases.
    First, if $\cR(z)\geq0$, it is easy to check that this distance is
    equal to $1-\gamma$ that is reached for $1-\gamma\in\overline{B}_{\CC}(1,\gamma)$
    and $0\in \zo\RR_-$.
    Second, if $0\geq \cR(z)\geq -\frac{1-\gamma}{\sqrt2}$.
    we have
    \begin{align*}
        d_{|\cdot|}\lp\overline{B}_{\CC}(1,\gamma),-\zo\RR_+\rp
        &=
        d_{|\cdot|}\lp 1,-\zo\RR_+\rp - \gamma
        \\
        &=
        \inf_{t\in\RR_+}
        |1+t\zo|
        -\gamma
        \\
        &=
        \inf_{t\in\RR_+}
        \sqrt{|1+t\cR(z)|^2+t^2(1-|\cR(z)|^2)}
        -\gamma
        \\
        &=
        \inf_{t\in\RR_+}
        \sqrt{1+2t\cR(z)+t^2}
        -\gamma
        \\
        &\geq
        \inf_{t\in\RR_+}
        \sqrt{1-\sqrt{2}(1-\gamma)t+t^2}
        -\gamma
        \\
        &=
        \sqrt{1-\frac{(1-\gamma)^2}2}
        -\gamma
        \\
        &\geq
        1-\frac{(1-\gamma)^2}2-\gamma
        \\
        &=
        \frac{1-\gamma^2}{2}
    \end{align*}
    where we used the fact that the minimum of $t\in\RR_+\mapsto 1-2at+t^2$ 
    equals $1-a^2$ for $a\geq0$ to get the sixth line,
    and $1-\frac{a}2\geq(1+a)^{-1}$ for $a\in[0,1]$ to get the last one.
    Using the above calculations, we obtain
    \begin{equation*}
        |\Sigma(I_d+zA\Sigma)^{-1}x|
        =
        |r|
        \leq
        d_{|\cdot|}\lp\overline{B}_{\CC}(1,\gamma),-\zo\RR_+\rp^{-1}|x|
        \leq
        \frac{2}{1-\gamma^2}|x|,
    \end{equation*}
    which is exactly Inequality \eqref{eq:near_one_aux}.
    This concludes the proof.
\end{proof}

\begin{lemma}
    \label{lem:series}
    Let $M\in\RR^{d\times d}$ be a matrix with a spectrum included in the open unit disk, and $c\in(0,1)$.
    We have
    \begin{equation*}
        \sum_{\ell=0}^{\infty}
        c^{2\ell}(I_d-M^{\ell\top})(I_d-M^{\ell})
        =
        \frac1{2i\pi}
        \int_{|z|=1}
        z^{-1}\lp(z^{-1}-c)^{-1}
        -(z^{-1}I_d-cM)^{-\top}\rp
        \lp (z-c)^{-1}-(zI_d-cM)^{-1}\rp dz.
    \end{equation*}
\end{lemma}
\begin{proof}
    Since the spectral radius of $M$ is lower than $1$,
    it is power-bounded, i.e., $\sup_{k\geq0}\|M^k\|_{\rm op}<\infty$.
    Therefore, the series 
    \begin{equation*}
        F_0(z)
        =
        \sum_{\ell=0}^{\infty}
        c^{\ell}z^{-\ell-1}M^{\ell},
    \end{equation*}
    is uniformly convergent on $\Gamma:=\{z\in\CC, |z|=1\}$.
    This implies that
    $(I_d-cz^{-1}M) F_0(z)=z^{-1}I_d$ and then
    \begin{equation*}
        F_0(z)
        =
        z^{-1}(I_d-cz^{-1}M)^{-1}
        =
        (zI_d-cM)^{-1}.
    \end{equation*}
    Similarly, we obtain
    \begin{equation*}
        F(z)
        :=
        (z-c)^{-1}-(zI_d-cM)^{-1}
        =
        \sum_{\ell=0}^{\infty}
        c^{\ell}z^{-\ell-1}(I_d-M^{\ell}),
    \end{equation*}
    so that, we get
    \begin{align*}
        z^{-1} F(z^{-1})^{\top}F(z)
        &=
        z^{-1}\lp\sum_{\ell=0}^{\infty}
        c^{\ell}z^{\ell+1}(I_d-M^{\ell})\rp^{\top}
        \lp\sum_{\ell=0}^{\infty}
        c^{\ell}z^{-\ell-1}(I_d-M^{\ell})\rp
        \\
        &=
        \sum_{k,\ell\geq0}
        c^{k+\ell}z^{k-\ell-1}(I_d-M^k)^{\top}(I_d-M^{\ell}).
    \end{align*}
    For similar arguments as above, the latter double series is uniformly
    convergent on $\Gamma$ which allows us to interchange summation
    and integration as follows,
    \begin{align*}
        \int_{|z|=1}
        z^{-1} F(z^{-1})^{\top}F(z)dz
        &=
        \sum_{k,\ell\geq0}
        c^{k+\ell}(I_d-M^k)^{\top}(I_d-M^{\ell})
        \int_{|z|=1}
        z^{k-\ell-1}dz
        \\
        &=
        2i\pi\sum_{k,\ell\geq0}
        c^{k+\ell}(I_d-M^k)^{\top}(I_d-M^{\ell}),
    \end{align*}
    where we used that $\int_{|z|=1}z^{k-\ell-1}=2i \pi$
    by Residue theorem.
    This concludes the proof.
\end{proof}

\subsection{Proof of Theorem \ref{thm:TD0}}
\begin{proof}[Proof of Theorem \ref{thm:TD0}]
    This is a straightforward consequence of 
    Propositions \ref{prop:LSA} and \ref{prop:Kreiss},
    with $\beta=(1+\gamma)^2$, $\beta_1=1+\gamma$,
    $c_{\Sigma}=(1-\gamma)^{-1}$ and $\ss^2=\ss_0^2$.
\end{proof}

\subsection{Proof of Proposition \ref{prop:sharpness}}
\begin{proof}[Proof of Proposition \ref{prop:sharpness}]
    In this case, the matrices $\Sigma_0,\Sigma_1,H,S$
    are straightfoward to compute. In particular, we have
    \begin{equation*}
        \Sigma_0=\Sigma_1=\omega I_d
        \hspace*{1cm}
        \text{ and }
        \hspace*{1cm}
        H=S=(1-\gamma)\omega I_d.   
    \end{equation*}
    Moreover, we have $b=\EE[R_k\vp(X_k)]=0$
    so that $\theta^*=0$.
    Therefore, we have
    \begin{equation*}
        \EE_{X\sim m}\lb|v(X,\tho_k)-v(X,\theta^*)|^2\rb
        =
        \omega\EE\lb|\tho_k|^2\rb.
    \end{equation*}
    For $1\leq i\leq d$,
    we have
    $\theta_{k,i}$ satisfies
    \begin{equation*}
        \theta_{k,i}
        =
        (1-d\alpha(1-\gamma)\omega B_{k,i})\theta_{k-1,i}
        +\alpha\sqrt{d\omega}B_{k,i}R_k,
    \end{equation*}
    where $(B_{k})_{k\geq1}$ are i.i.d. $d$-dimensional
    random variables, with $B_{1,i}\sim$Ber$(\frac1d)$
    and $\sum_{i=1}^dB_{1,i}=1$ a.e..

    On the one hand,
    from Lemma \ref{lem:aux_sharpness}, under the regime described
    in Proposition \ref{prop:sharpness}, for $C=5.5$,
    \begin{equation}
        \label{eq:sharpness_aux}
        \EE_{X\sim m}\lb|v(X,\tho_k)-v(X,\theta^*)|^2\rb
        \gtrsim
        \frac{(1-e^{-C})^2|\theta_0|^2}{C\alpha(1-\gamma)k}
        +\frac{\ss_0^2}{(1-\gamma)^2 k}
        \lp1-\frac{3-4e^{-C}+e^{-2C}}{2C}\rp.
    \end{equation}
    On the other hand,
    under the same convergence regime,
    observe that the upper bound in Theorem \ref{thm:TD0}
    is equivalent to
    \begin{align*}
        E
        &=
        \lp\lp\frac{|\theta_0|^2}{2\alpha(1-\gamma)k}\rp^{\frac12}
        +\lp\frac{e\ss_0^2}{(1-\lambda)(1-\gamma)^2 k}\rp^{\frac12}\rp^2
        \\
        &\leq
        \frac{(1+A)|\theta_0|^2}{2\alpha(1-\gamma)k}
        +\frac{2(1+A^{-1})e\ss_0^2}{(1-\gamma)^2 k},
    \end{align*}
    where we take $A=2.7$.
    Then, we can easily check that 
    $1.85=\frac{1+A}2<11\frac{(1-e^{-C})^2}C\approx1.98$
    and
    $7.45\approx2e(1+A^{-1})< 11\lp1-\frac{3-4e^{-C}+e^{-2C}}{2C}\rp\approx 8.02$.
    Therefore we obtain
    \begin{equation*}
        \EE_{X\sim m}\lb|v(X,\tho_k)-v(X,\theta^*)|^2\rb
        \gtrsim
        11E,
    \end{equation*}
    which concludes the proof.
\end{proof}

\begin{corollary}
    \label{cor:sharpness}
    Consider the same MRP and feature functions
    as in Proposition \ref{prop:sharpness}.
    Take $\alpha=\lambda\alpha_0(\gamma)$ for $\lambda\in(0,1)$.
    For $k\to\infty$, $\gamma\to1$ and $k(1-\gamma)\to1.25$,
    we have
    \begin{equation*}
        \EE_{X\sim m}\lb|v(X,\tho_k)-v(X,\theta^*)|^2\rb
        \gtrsim
        \frac1{1.25}
        \frac{|\theta_0-\theta^*|^2}
        {2\lp 1-\frac{\alpha}{\alpha_1(\gamma)}\rp\alpha(1-\gamma)k}.
    \end{equation*}
    Moreover, 
    for $\gamma\to1$ and $k(1-\gamma)\to\infty$,
    we have
    \begin{equation*}
        \EE_{X\sim m}\lb|v(X,\tho_k)-v(X,\theta^*)|^2\rb
        \gtrsim
        \frac{1-\frac{\alpha}{\alpha_0(\gamma)}}{e}
        \frac{de\ss_0^2(1+\ee_{k,\gamma})}{(1-\frac{\alpha}{\alpha_0(\gamma)})(1-\gamma)^2k}.
    \end{equation*}
\end{corollary}
\begin{proof}
    Similarly as in the proof of Proposition \ref{prop:sharpness},
    for $k\to\infty$, $\gamma\to1$ and $k(1-\gamma)\to1.25$,
    the asymptotic inequality \ref{eq:sharpness_aux} holds
    for $C=1.25$. This implies 
    \begin{equation*}
        \EE_{X\sim m}\lb|v(X,\tho_k)-v(X,\theta^*)|^2\rb
        \gtrsim
        \frac{(1-e^{-C})^2|\theta_0|^2}{C\alpha(1-\gamma)k}
    \end{equation*}
    which implies the first inequality in Corollary \ref{cor:sharpness},
    since $\frac{2(1-e^{-C})^2}{C}\geq\frac1{1.25}$.

    Now consider
    for $\gamma\to1$ and $k(1-\gamma)\to\infty$.
    Repeating the arguments in the proof of Lemma \ref{lem:aux_sharpness},
    we obtain
    \begin{equation*}
        \EE_{X\sim m}\lb|v(X,\tho_k)-v(X,\theta^*)|^2\rb
        \gtrsim
        \frac{d\ss_0^2}{(1-\gamma)^2k}
        \sim
        \frac{1-\frac{\alpha}{\alpha_0(\gamma)}}{e}
        \frac{de\ss_0^2(1+\ee_{k,\gamma})}{(1-\frac{\alpha}{\alpha_0(\gamma)})(1-\gamma)^2k}.
    \end{equation*}
    This concludes the proof.
\end{proof}

\begin{lemma}
\label{lem:aux_sharpness}
    Let $(Y_k)_{k\geq0}$ be a stochastic process with value in $\RR$ such
    that $Y_0=y_0\in\RR$ is fixed and, for $k\geq1$
    \begin{equation*}
        Y_{k}
        =
        (1-\lambda B_k)Y_{k-1}
        +B_kZ_k,
    \end{equation*}
    for $p\in[0,1]$ and $\lambda\in(0,2)$,
    where $(B_k,Z_k)_{k\geq1}$ are i.i.d. with $B_1$ independent of $Z_1$,
    $B_1\sim$Ber$(p)$, $\EE[Z_1]=0$ and $\Var(Z_1)=\sigma^2>0$.
    Define $\Yo_k=\frac1k\sum_{i=0}^{k-1}Y_i$.
    In the regime $k\to\infty$, $\lambda\to0$ and $pk\lambda\to C$,
    for some $C>0$,
    we have
    \begin{equation*}
        \EE\lb\Yo_k^2\rb
        \gtrsim
        \frac{(1-e^{-C})^2y_0^2}{Cp\lambda k}
        +\frac{\ss^2}{p\lambda^2 k}
        \lp
        1-\frac{3-4e^{-C}+e^{-2C}}{2C}\rp.
    \end{equation*}
\end{lemma}

\begin{proof}[Proof of Lemma \ref{lem:aux_sharpness}]
    Using superposition principle, we get
    $Y_k=Y^b_k+Y^v_k$, with 
    \begin{align*}
        Y^b_{k}
        &=
        (1-\lambda B_k)Y^b_{k-1},
        \hspace*{2cm}Y^b_0=y_0
        \\
        Y^v_{k}
        &=
        (1-\lambda B_k)Y^v_{k-1}
        +B_kZ_k,
        \hspace*{1cm}Y^v_0=0.
    \end{align*}
    Moreover, using that $Z_k$ is centered and independent with $Y^b_k$,
    we get $\EE[Y_k^v]=0$ and $\EE\lb Y_k^bY_k^v\rb=0$.
    Therefore, we obtain
    \begin{equation*}
        \EE\lb Y_k^2\rb
        =
        \EE\lb (Y_k^b)^2\rb
        +\EE\lb (Y_k^v)^2\rb
    \end{equation*}

    \emph{First step: the bias.}

    \noindent
    In this step, we only consider $Y_k^b$.
    We have
    \begin{equation*}
        \EE\lb Y^b_k\rb
        =
        p(1-\lambda)\EE\lb Y^b_{k-1}\rb
        +(1-p)\EE\lb Y^b_{k-1}\rb
        =
        (1-p\lambda)\EE\lb Y^b_{k-1}\rb
        =
        (1-p\lambda)^ky_0,
    \end{equation*}
    so that we get
    \begin{equation*}
        \EE\lb(\Yo^b_k)^2\rb
        \geq
        \EE\lb\Yo^b_k\rb^2
        =
        \lp\frac1k\sum_{i=0}^{k-1}
        (1-p\lambda)^iy_0^2\rp
        =
        \frac{(1-(1-p\lambda)^k)^2y_0^2}{\lambda^2p^2k^2}.
    \end{equation*}
    Therefore, for $k\to \infty$, $\lambda\to 0$ and $k\lambda p\to C>0$,
    using $(1-p\lambda)^k\to e^{-C}$, we obtain
    \begin{equation*}
        \EE\lb(\Yo^b_k)^2\rb
        \gtrsim
        \frac{(1-e^{-C})^2y_0^2}{C\lambda pk}.
    \end{equation*}

    \emph{Second step: the variance term.}

    \noindent
    Recall that $\EE[Y^v_k]=0$.
    Then, for $\lambda'=2\lambda-\lambda^2\in(0,1)$, we have
    \begin{align*}
        \EE\lb(Y^v_k)^2\rb
        &=
        \lp p(1-\lambda)^2+(1-p)\rp
        \EE\lb(Y^v_{k-1})^2\rb
        +p\ss^2
        \\
        &=
        \lp 1-p(2\lambda-\lambda^2)\rp
        \EE\lb(Y^v_{k-1})^2\rb
        +p\ss^2
        \\
        &=
        p\sigma^2\sum_{i=0}^{k-1}
        \lp 1-p\lambda'\rp^i
        \\
        &=
        \frac{(1-(1-p\lambda')^k)\sigma^2}{\lambda'}.
    \end{align*}
    Let $(\cF_k)_{k\geq1}$ be the filtration associated to $(Z_k,B_k)_{k\geq1}$.
    Using similar arguments as in the first step,
    for $1\leq i\leq j$, we get
    \begin{equation*}
        \EE\lb Y^v_j\,|\,\cF_i\rb
        =
        (1-p\lambda)^{j-i} Y^v_i.
    \end{equation*}
    This implies that
    \begin{equation*}
        \EE\lb Y^v_jY^v_i\rb
        =
        (1-p\lambda)^{j-i}
        \EE\lb (Y^v_i)^2\rb
        =
        (1-p\lambda)^{j-i}
        \frac{(1-(1-p\lambda')^i)\sigma^2}{\lambda'}.
    \end{equation*}
    Therefore, we obtain
    \begin{align*}
        \EE\lb (\Yo^v_k)^2\rb
        &=
        \frac1{k^2}
        \lp
        2\sum_{0\leq i\leq j\leq k-1}\EE\lb Y^v_jY^v_i\rb
        -\sum_{i=0}^{k-1}\EE\lb (Y^v_i)^2\rb
        \rp
        \\
        &=
        \frac{\ss^2}{\lambda'k^2}
        \sum_{i=0}^{k-1}
        \lp
        (1-(1-p\lambda')^i)
        \lp2\sum_{j=i}^{k-1}(1-p\lambda)^{j-i}
        -1\rp\rp
        \\
        &=
        \frac{\ss^2}{\lambda'k^2}
        \sum_{i=0}^{k-1}
        \lp
        (1-(1-p\lambda')^i)
        \lp\frac{2(1-(1-p\lambda)^{k-i})}{p\lambda}
        -1\rp\rp
        \\
        &=
        \frac{2\ss^2}{p\lambda\lambda'k^2}
        \sum_{i=0}^{k-1}
        \lp
        1
        -(1-p\lambda')^i
        -(1-p\lambda)^{k-i}
        +(1-p\lambda)^{k}\lp\frac{1-p\lambda'}{1-p\lambda}\rp^{i}
        \rp
        -\frac{\ss^2}{\lambda'k^2}
        \sum_{i=0}^{k-1}(1-(1-p\lambda')^i)
        \\
        &=
        \frac{2\ss^2}{p\lambda\lambda'k^2}
        \lp k
        -\frac{1-(1-p\lambda')^k}{p\lambda'}
        -\frac{1-(1-p\lambda)^k}{p\lambda}
        +(1-p\lambda)\frac{(1-p\lambda)^k-(1-p\lambda')^k}{p(\lambda'-\lambda)}\rp
        \\
        &\hspace*{1cm}
        -\frac{\ss^2}{\lambda'k^2}
        \lp k-\frac{1-(1-p\lambda')^k}{p\lambda'}\rp.
    \end{align*}
    Therefore, for $k\to \infty$, $\lambda\to 0$ and $k\lambda p\to C>0$, we obtain the
    following asymptotic equivalent
    \begin{align*}
        \EE\lb (\Yo^v_k)^2\rb
        &\sim
        \frac{\ss^2}{p\lambda^2 k}
        \lp
        1-\frac{1-e^{-2C}}{2C}
        -\frac{1-e^{-C}}{C}
        +\frac{e^{-C}-e^{-2C}}{C}\rp
        -\frac{\ss^2}{2\lambda k}
        \lp 1-\frac{1-e^{-2C}}{2C}\rp
        \\
        &=
        \frac{\ss^2}{p\lambda^2 k}
        \lp
        1-\frac{3-4e^{-C}+e^{-2C}}{2C}\rp
        -\frac{\ss^2}{2\lambda k}
        \lp 1-\frac{1-e^{-2C}}{2C}\rp
        \\
        &\sim
        \frac{\ss^2}{p\lambda^2 k}
        \lp
        1-\frac{3-4e^{-C}+e^{-2C}}{2C}\rp,
    \end{align*}
    where we used $\lambda'\sim2\lambda$,
    $(1-p\lambda)^k\sim e^{-C}$
    and $(1-p\lambda')^k\sim e^{-2C}$
    to get the first line,
    $\lambda\to0$ and 
    $1-\frac{3-4e^{-C}+e^{-2C}}{2C}>0$ for $C>0$
    to get the last one.

    This concludes the proof.
\end{proof}

\subsection{Proof of Theorem \ref{thm:TD0_larger_alpha}}
\begin{proof}
    The convergence rate is a consequence of 
    Theorem \ref{thm:LSA_alternative2}
    with $\beta=(1+\gamma)^2$, $\beta_1=1+\gamma$,
    $c_{\Sigma}=(1-\gamma)^{-1}$ and $\ss^2=\ss_0^2$.

    Therefore, it only remains to prove the second part of the theorem.
    Take the Markov Process on $\cX=\{1,-1\}$ such that $X'=-X$ a.e..
    Take $m\equiv\frac12$ the uniform probability distribution,
    $d=1$ and $\phi(x)=x$.
    Consider $(R_k)_{k\geq1}$ i.i.d. with $R_1\sim\cN(0,\ss_0^2)$.
    For $\alpha\geq\alpha_1(\gamma)=\frac2{1+\gamma}$,
    $k\geq1$, we have
    \begin{equation*}
        \theta_k
        =
        (1-\alpha(1+\gamma))\theta_{k-1}
        +\alpha R_k
        =
        r^k\theta_0
        +\alpha\sum_{i=0}^{k-1}r^{i}R_{k-i},
    \end{equation*}
    with $r=1-\alpha(1+\gamma)<-1$.
    Therefore, we obtain
    \begin{equation*}
        \EE\lb|\theta_k|^2\rb
        =
        r^{2k}\theta_0^2
        +\alpha^2\ss_0^2\sum_{i=0}^{k-1}r^{2i}
        =
        r^{2k}\theta_0^2
        +\alpha^2\ss_0^2
        \frac{r^{2k}-1}{r^2-1}
        \underset{k\to\infty}{\longrightarrow}+\infty.
    \end{equation*}
    Similarly, we have
    \begin{equation*}
        \EE_{m}\lb|v(X,\tho_k)-v(X,\theta^*)|^2\rb
        =
        \EE\lb|\tho_k|^2\rb
        =
        \frac{(r^k-1)^2}{(1-r)^2k^2}|\theta_0|^2
        +\frac{\alpha^2\sigma_0^2}{(1-r)^2k^2}|\theta_0|^2
        \sum_{i=1}^{k-1}|r^i-1|^2
        \underset{k\to\infty}{\longrightarrow}+\infty.
    \end{equation*}
    This concludes the proof.
\end{proof}

\subsection{Proof of Theorem \ref{thm:minibatch_TD0}}
\begin{proof}[Proof of Theorem \ref{thm:minibatch_TD0}]
    This is a straightforward consequence of 
    Proposition \ref{prop:LSA_minibatch},
    the equality $\cS_k=\tr(T_k)$ and
    Proposition \ref{prop:Kreiss_larger_alpha},
    with $\beta=(1+\gamma)^2$, $\beta_1=1+\gamma$,
    $c_{\Sigma}=(1-\gamma)^{-1}$ and $\ss^2=\ss_0^2$.
\end{proof}

\begin{proposition}
    \label{prop:Kreiss_larger_alpha}
    Assume \ref{hypo:iid}-\ref{hypo:linear}.
    For $Q = I_d-\alpha\Sigma_0(I_d-\gamma U)$,
    $0<\alpha<\alpha_1(\gamma)$ and
    $k\geq1$,
    define $T_k$ by
    \begin{equation*}
        T_k=
        \frac1k\sum_{j=0}^{k-1}
        (I_d-\gamma U)^{-\top}
        (I_d-Q^j)^{\top}(I_d-Q^j)
        (I_d-\gamma U)^{-1}.
    \end{equation*}
    we have
    \begin{equation*}
        \norm{T_k}{\rm op}
        \leq
        \frac{e}{(1-\gamma)^2}
        \lp 1 
        +\frac{3}{k(1-\gamma)^2}
        +\frac{5}{k(1-\gamma)}
        +3(1-\gamma)
        +\frac1{4k(1-\frac{\alpha}{\alpha_1(\gamma)})^2}\rp.
    \end{equation*}
\end{proposition}

\begin{proof}
    Using similar arguments as in the beginning of the proof
    of Proposition \ref{prop:Kreiss}, we have
\begin{equation}
    \label{eq:bound_Tk}
    \norm{T_k}{\rm op}
    \leq
    \frac{e}{2\pi k}
    \int_{-\pi}^{\pi}
    \norm[2]{\lp(e^{i\psi}-c_k)^{-1}I_d-(e^{i\psi}I_d-\Qt)^{-1}\rp(I_d-\gamma U)^{-1}}{\rm op}d\psi.
\end{equation}
Then,
continuing similarly as in the proof of Proposition \ref{prop:Kreiss},
we obtain 
\begin{equation*}
    \frac{e}{2\pi}
    \int_{-\psi_{\gamma}}^{\psi_{\gamma}}
    \norm[2]{\lp(e^{i\psi}-c_k)^{-1}I_d-(e^{i\psi}I_d-\Qt)^{-1}\rp(I_d-\gamma U)^{-1}}{\rm op}d\psi
    \leq
    \frac{ke}{(1-\gamma)^2}(4-3\gamma),
\end{equation*}
where we recall that $\psi_{\gamma}$ satisfies
$\psi_{\gamma}=\arccos(1-(1-\gamma)^2)$.

It only remains to deal with the part of the integral in \eqref{eq:bound_Tk}
on $[-\pi,\pi]\backslash[-\psi_{\gamma},\psi_{\gamma}]$,
or more simply on $[\psi_{\gamma},\pi]$ using parity.
First, for $\psi\in[\psi_{\gamma},\pi]$, we have
\begin{equation*}
    \norm[2]{\lp(e^{i\psi}-c_k)^{-1}I_d-(e^{i\psi}I_d-\Qt)^{-1}\rp(I_d-\gamma U)^{-1}}{\rm op}
    \leq
    \frac2{(1-\gamma)^2}
    \lp|e^{i\psi}-c_k|^{-2} + \norm[2]{(e^{i\psi}-\Qt)^{-1}}{\rm op}\rp.
\end{equation*}
Let us start with the first term in the above sum.
Using $|e^{i\psi}-c_k|^2\geq \cI(e^{i\psi}-c_k)^2=\sin(\psi)^2$
when $\psi\in[\psi_{\gamma},\frac{\pi}2]$,
we obtain
\begin{align*}
    2\int_{\psi_{\gamma}}^{\frac{\pi}2}|e^{i\psi}-c_k|^{-2}d\psi
    \leq
    2\int_{\psi_{\gamma}}^{\pi}\sin(\psi)^{-2}d\psi
    \leq\frac{\sqrt{2}(2-\gamma)}{1-\gamma}
    \leq
    \frac{\pi}{1-\gamma},
\end{align*}
where we used a similar argument as in the second step
of the proof of Proposition \ref{prop:Kreiss} to get the second inequality.
Then, when $\psi\in[\frac{\pi}2,\pi]$,
we have $\Re(e^{i\psi})\leq 0$ so that
$|e^{i\psi}-c_k|^2\geq \cR(e^{i\psi}-c_k)^2\geq c_k^2\geq 4$
and then
\begin{equation*}
    2\int_{\frac{\pi}2}^{\pi}|e^{i\psi}-c_k|^{-2}d\psi
    \leq
    4\pi.
\end{equation*}
From the above computations and \eqref{eq:bound_Tk},
we deduce
\begin{equation}
\label{eq:bound_Tk2}
    \norm{T_k}{\rm op}
    \leq
    \frac{e}{(1-\gamma)^2}(4-3\gamma)
    +\frac{5e}{(1-\gamma)^3k}
    +\frac{2e}{\pi(1-\gamma)^2k}
    \int_{\psi_{\gamma}}^{\pi}
    \norm[2]{(e^{i\psi}-\Qt)^{-1}}{\rm op}d\psi.
\end{equation}
The remainder of the proof is inspired by the third step of
the proof of Proposition \ref{prop:Kreiss}.

\emph{First step: preliminary considerations.}

\noindent 
Take $z=e^{i\psi}$ with $\psi\in[\psi_{\gamma},\pi]$.
For $x\in\CC^d$ and $y=(z-\Qt)^{-\top}x$, we have
\begin{align*}
    c_k^{-1}|x|
    =
    c_k^{-1}|(z-\Qt)^{\top}y|
    &=
    |(c_k^{-1}z-(I_d-\alpha(I_d-\gamma U^{\top})\Sigma_0))y|
    \\
    &\geq
    |(c_k^{-1}z-1+\alpha\Sigma_0)y|
    -\alpha\gamma|U^{\top}\Sigma_0y|.
    \\
    &\geq
    |(c_k^{-1}z-1+\alpha\Sigma_0)y|
    -\alpha\gamma|\Sigma_0y|.
\end{align*}
Recall that $\Sigma_0$ is symmetric positive definite.
Let us denote $(\lambda_i)_{1\leq i\leq d}$
the eigenvalues of $\Sigma_0$,
associated with the eigenvectors
$(e_i)_{1\leq i\leq d}$.
Let us decompose $y$ in the basis of eigenvectors
as $y=\sum_{i=1}^dy_ie_i$.
Recall the notations 
$s=1-c_k^{-1}\cR(z)$
and $f(u,s)=|z-c_k+c_k\alpha u|^2$
\begin{align*}
    |(z-c_k+c_k\alpha\Sigma_0)y|^2
    &=
    \sum_{i=1}^d
    |z-c_k+c_k\alpha\lambda_i|^2y_i^2
    \\
    &=
    \sum_{i=1}^d
    f(\lambda_i,s)y_i^2
    \\
    &=
    c_k^2\sum_{i=1}^d
    (\alpha^2\lambda_i^2+2s(1-\alpha \lambda_i)-1+c_k^{-2})y_i^2
    \\
    &=
    c_k^2\lp\alpha^2u^2
    +2s(1-\alpha\llo)-1+c_k^{-2}\rp|y|^2
    \\
    &\geq
    c_k^2\lp\alpha^2u^2
    +2s(1-\alpha u)-1+c_k^{-2}\rp|y|^2
    \\
    &=
    f(u,s),
\end{align*}
with $\llo:=\sum_{i=1}^d\lambda_i\frac{y_i^2}{|y|^2}$
and $u=\sqrt{\sum_{i=1}^d\lambda_i^2\frac{y_i^2}{|y|^2}}$,
where the computations to get the third line are 
detailed in the third step of the proof of Proposition \ref{prop:Kreiss},
and the fifth line is obtained by Cauchy-Schwarz inequality $\llo\leq u$.

As in the third step of the proof of Proposition \ref{prop:Kreiss},
we distinguish two cases.
First, if $\alpha u>1$, we have once again
$f(u,s)\geq (2-\alpha)^2$
and then
\begin{equation}
    \label{eq:aux_alphau}
    |x|
    \geq
    (2-\alpha-\gamma\alpha)|y|
    =
    (2-\alpha(1+\gamma))|y|.
\end{equation}
Only the case $\alpha u\leq1$ remains to be dealt with.
To do so, we introduce further materials in this step,
starting with $\delta=c_k^{-1}-1>0$,
$w=\alpha u$
and $g(w,s)$ defined by
\begin{equation*}
    g(w,s)
    =
    c_k^{-1}\sqrt{f(u,s)}-\gamma w
    =
    \sqrt{w^2+2s(1-w)-1+c_k^{-2}}
    -\gamma w
    =
    \sqrt{(w-s)^2+(2+\delta -s)(s+\delta)}
    -\gamma w,
\end{equation*}
with $s\in[(1-\gamma)^2(1+\delta)-\delta,2+\delta]\subset[(1-\gamma)^2-\delta,2+\delta]$.
Using the above computations,
we have
\begin{equation}
    \label{eq:aux_g}
    |x|\geq c_kg(w,s)|y|.
\end{equation}
Therefore, to conclude it is sufficient to obtain a convenient lower bound on $g$.
Observe that 
\begin{align*}
    \partial_w g
    &=
    \frac{w-s}{\sqrt{(w-s)^2+(2+\delta -s)(s+\delta)}}
    -\gamma 
    \hspace*{0.3cm}
    \text{ and }
    \\
    \partial^2_{w,w} g
    &=
    \frac{1}{\sqrt{(w-s)^2+(2+\delta -s)(s+\delta)}}
    -\frac{(w-s)^2}{\lp(w-s)^2+(2+\delta -s)(s+\delta)\rp^{\frac32}}
    \geq
    0.
\end{align*}
Therefore, $g$ is convex in $w$.
We are now in position to derive useful estimates
to deal with the remaining term in \eqref{eq:bound_Tk2}.
In the two subsequent steps, we consider
separately the case $\Re(z)\leq \gamma$
and $\Re(z)\in[1-(1-\gamma)^2,\gamma]$,
always under the assumption $w=\alpha u\leq 1$.

\emph{Second step: the case of $\Re(z)\leq \gamma$.}

\noindent
Observe that
\begin{equation*}
    g(1,s)
    =
    \sqrt{1+2\delta+\delta^2}-\gamma
    =
    1-\gamma +\delta
    \hspace*{0.3cm}
    \text{ and }
    \hspace*{0.3cm}
    \partial_wg(1,s)
    =
    \frac{1-s}{1+\delta}-\gamma
    =
    c_k(1-s)-\gamma.
\end{equation*}
Then, assuming $\Re(z)\leq \gamma$,
we have $\partial_wg(1,s)\leq 0$
for $w\in[0,1]$.
Therefore, using a convex inequality,
we obtain,
\begin{align*}
    g(w,s)
    &\geq
    g(1,s)-(1-w)\partial_wg(1,s)
    \\
    &\geq
    1-\gamma.
\end{align*}
Consequently, 
using  the latter inequality, \eqref{eq:aux_alphau} and \eqref{eq:aux_g},
for $\Re(z)\leq \gamma$,
we obtain
\begin{align*}
    \norm[-2]{(e^{i\psi}-Q)^{-1}}{\rm op}
    &\leq
    \max((2-\alpha(1+\gamma))^{-2},c_k^{-2}(1-\gamma)^{-2})
    \\
    &\leq
    (2-\alpha(1+\gamma))^{-2}
    +c_k^{-2}(1-\gamma)^{-2}.
\end{align*}
We conclude this step with the following estimate
\begin{equation*}
    \int_{\psit_{\gamma}}^{\pi}
    \norm[-2]{(e^{i\psi}-Q)^{-1}}{\rm op}
    d\psi
    \leq
    \frac{\pi}{(2-\alpha(1+\gamma))^{2}}
    +\frac{\pi}{c_k^{2}(1-\gamma)^{2}},
\end{equation*}
where $\psit:=\arccos(\gamma)$.

\emph{Third step: the case of $\gamma\leq \Re(z)\leq 1- (1-\gamma)^2$.}

\noindent
The minimum of $g(\cdot,s)$ on $[0,1]$
is lower or equal to the minimum on $\RR$
that is reached for $w(s)$ such that
$\partial_wg(w(s),s)=0$, i.e.,
\begin{equation*}
    w(s)
    =
    s+\gamma\sqrt{\frac{(2+\delta-s)(s+\delta)}{1-\gamma^2}}.
\end{equation*}
The latter minimum then satisfies
\begin{align*}
    g(w(s),s)
    &=
    \sqrt{(2+\delta -s)(s+\delta)\lp\frac{\gamma^2}{1-\gamma^2}+1\rp}
    -\gamma\lp s+\frac{\gamma^2}{\sqrt{1-\gamma^2}}\sqrt{(2+\delta -s)(s+\delta)}\rp
    \\
    &=
    \sqrt{(1-\gamma^2)(2+\delta -s)(s+\delta)}
    -\gamma s
    \\
    &\geq
    \sqrt{(1-\gamma^2)c_k^{-1}(1+\gamma)(s+\delta)}
    -\gamma \sqrt{(1-\gamma)c_k^{-1}(s+\delta)}
    \\
    &=
    \sqrt{(1-\gamma)c_k^{-1}(s+\delta)}
    \\
    &=
    c_k^{-1}\sqrt{(1-\gamma)(1-\cR(z))}
\end{align*}
where the third line is obtained
using $c_k(s+\delta)=1-\Re(z)\leq 1-\gamma$ and $c_k^{-1}=1+\delta$,
so that we have
$2+\delta-s\geq 2+2\delta-c_k^{-1}(1-\gamma)=c_k^{-1}(1+\gamma)$
and $s\leq s+\delta\leq\sqrt{c_k^{-1}(1-\gamma)(s+\delta)}$.
Recall that $|z|=1$, so that 
\begin{equation*}
    |z-1|^2
    =
    (1-\cR(z))^2+\cI(z)^2
    =
    (1-\cR(z))^2+1-\cR(z)^2
    =
    2(1-\cR(z)).
\end{equation*}

\emph{Final step: getting the desired bound.}

\noindent

Recall that \eqref{eq:aux_alphau} holds for $\alpha u>1$
and \eqref{eq:aux_g} for $\alpha u\leq1$.
Therefore, in both cases, we have
\begin{align*}
    \norm[-2]{(e^{i\psi}-Q)^{-1}}{\rm op}
    &\leq
    \max\lp\frac{1}{(2-(1+\gamma)\alpha)^2},
    \frac{1}{c_k^2g(w,s)^2}\rp
    \\
    &\leq
    \frac{1}{(2-(1+\gamma)\alpha)^2}
    +\frac{1}{c_k^2g(w,s)^2}.
\end{align*}
Using the lower bound on $g$ obtained in the two previous steps,
we obtain, for $\psit_{\gamma}=\arccos(\gamma)$,
\begin{align*}
    \int_{\psi_{\gamma}}^{\pi}
    \norm[-2]{(e^{i\psi}-Q)^{-1}}{\rm op}
    d\psi
    &\leq
    \frac{(\pi-\psi_{\gamma})}{(2-\alpha(1+\gamma))^{2}}
    +\int_{\psi_{\gamma}}^{\pi}
    \frac{1}{(1-\gamma)^2}
    +\int_{\psi_{\gamma}}^{\psit_{\gamma}}
    \frac{2}{(1-\gamma)|1-z|^2}
    d\psi
    \\   
    &\leq
    \frac{\pi}{(2-\alpha(1+\gamma))^{2}}
    +\frac{\pi}{(1-\gamma)^2}
    +\frac{1}{1-\gamma}
    \int_{\psi_{\gamma}}^{\psit_{\gamma}}
    \frac{2}{\sin(\psi)^2}d\psi
    \\
    &\leq
    \frac{\pi}{(2-\alpha(1+\gamma))^{2}}
    +\frac{\pi}{(1-\gamma)^2}
    +\frac{\sqrt{2}}{(1-\gamma)^2},
\end{align*}
where the derivation of the upper bound of the integral from the second line
is detailed in the second step of the proof of Proposition \ref{prop:Kreiss}.
This and \eqref{eq:bound_Tk2} imply
\begin{equation*}
    \norm{T_k}{\rm op}
    \leq
    \frac{e}{(1-\gamma)^2}(4-3\gamma)
    +\frac{5e}{(1-\gamma)^3k}
    +\frac{2e}{(2-\alpha(1+\gamma))^{2}(1-\gamma)^2k}
    +\frac{2e}{(1-\gamma)^4k}
    +\frac{2\sqrt{2}e}{\pi(1-\gamma)^4k}.
\end{equation*}
This and $\frac{2\sqrt2}{\pi}\leq1$ lead to the desired inequality,
this concludes the proof.
\end{proof}

\subsection{Proof of Lemma \ref{lem:parametrization_trick}}
\begin{proof}[Proof of Lemma \ref{lem:parametrization_trick}]
    We consider $\vp$ as in the first case in Lemma
    \ref{lem:parametrization_trick} and $(\theta_k)_{k\geq0}$
    the induced sequence obtained using TD(0).
    Similarly, we define $\vpt$ and $(\tht_k)_{k\geq0}$
    according to the second case.
    We are going to prove by induction
    on $k\geq0$ 
    \begin{equation*}
        \vp(x)^{\top}\theta_k
        =
        \vpt(x)^{\top}\tht_k
        \;\text{ for any }x\in\cX.
    \end{equation*}
    First for $k=0$, we have
    \begin{equation*}
        \vpt(x)^{\top}\tht_0
        =
        \tht_{0,1}
        +\vpt_{(-1)}(x)^{\top}\tht_{0,(-1)}
        =
        C\theta_{0,1}
        +\vp_{(-1)}(x)^{\top}\theta_{0,(-1)}
        =
        \vp(x)^{\top}\theta_0.
    \end{equation*}
    Then, for $k\geq1$,
    assume that the equality holds at index $k-1$.
    Observe that $\vpt(x)^{\top}\aat\vpt(y)=\alpha\vp(x)^{\top}\vp(y)$,
    so that we obtain
    \begin{align*}
        \vpt(x)^{\top}\tht_k
        &=
        \vpt(x)^{\top}
        \lp\tht_{k-1}
        -\aat\vpt(X_k)\lp(\vpt(X_k)-\gamma\vpt(X_k'))^{\top}\tht_{k-1}-R_k\rp\rp
        \\
        &=
        \vp(x)^{\top}\theta_{k-1}
        -\alpha\vp(x)^{\top}\vp(X_k)
        \lp(\vp(X_k)-\gamma\vp(X_k'))^{\top}\theta_{k-1}-R_k\rp
        \\
        &=
        \vp(x)^{\top}
        \lp\theta_{k-1}
        -\alpha\vp(X_k)
        \lp(\vp(X_k)-\gamma\vp(X_k'))^{\top}\theta_{k-1}-R_k\rp\rp
        \\
        &=
        \vp(x)^{\top}\theta_k.
    \end{align*}

\end{proof}

\subsection{Asymptotic behavior of $\theta^*$ as $\gamma$ tends to $1$}
\begin{lemma}
\label{lem:theta*_to_infty}
    Assume \ref{hypo:iid}-\ref{hypo:linear}
    and that there exists $\theta_{\mathds{1}}\in\RR^d$
    such that $v(\cdot,\theta_{\mathds{1}})=\mathds{1}$,
    where $\mathds{1}$ is the constant function equal to one.
    Then, we have
    \begin{equation*}
        \mu^{\top}\theta^*
        =
        \EE\lb v(X,\theta^*)\rb
        =\frac{\EE[R]}{1-\gamma}.
    \end{equation*}
    In particular,
    if $\EE[R]\neq 0$, we obtain
    \begin{equation*}
        |\theta^*|\geq\frac{|\EE[R]|}{|\mu|(1-\gamma)}
        =
        O((1-\gamma)^{-1}).
    \end{equation*}
\end{lemma}

\begin{proof}
    Observe that $\Sigma_0\theta_1=\mu=\Sigma_1^{\top}\theta_{\mathds{1}}$,
    which implies $\theta_{\mathds{1}}^{\top}(\Sigma_0-\gamma\Sigma_1)
    =\theta_{\mathds{1}}^{\top}H=(1-\gamma)\mu^{\top}$, so that we get
    \begin{equation*}
     \EE\lb v(X,\theta^*)\rb
     =
     \mu^{\top}\theta^*
     =
     (1-\gamma)^{-1}\theta_{\mathds{1}}^{\top}H\theta^*
     =
     (1-\gamma)^{-1}\theta_{\mathds{1}}^{\top}b
     =
     (1-\gamma)^{-1}\EE[R].
    \end{equation*}
    Recall that $\mu\neq 0$, otherwise
    $\theta_{\mathds{1}}$ cannot exist.
    The second inequality in the lemma is straightforward.
\end{proof}

\begin{lemma}
\label{lem:bound_param_trick}
    Under the same assumptions as Theorem \ref{thm:TD0_param_trick},
    $\theta^*$ is uniformly bounded with respect to $\gamma$.
\end{lemma}

\begin{proof}
    Observe that $H$ and $b$ write as
    \begin{equation*}
        H
        =
        \begin{pmatrix}
        (1-\gamma)^{-1} &\mu_{(-1)}^{\top}\\
        \mu_{(-1)}& \Ht,
        \end{pmatrix}
        \hspace*{2cm}
        b
        =
        \begin{pmatrix}
            (1-\gamma)^{-1}\EE[R]\\
            \bt
        \end{pmatrix}
    \end{equation*}
    with $\Ht=\EE\lb\vp_{(-1)}(X)(\vp_{(-1)}(X)-\gamma \vp_{(-1)}(X'))^{\top}\rb$
    and $\bt=\EE[R\vp_{(-1)}(X)]$.
    From the first line in $H\theta^*=b$,
    we obtain
    \begin{equation}
    \label{eq:bound_param_trick_theta1}
        \theta^*_1
        =
        \EE[R]-(1-\gamma)\mu_{(-1)}^{\top}\theta^*_{(-1)}.
    \end{equation}
    Using this and the other lines from $H\theta^*=b$, we get
    \begin{equation*}
        \Ht\theta^*_{(-1)}
        =
        \bt
        -\EE[R]\mu_{(-1)},
    \end{equation*}
    where $\Ht=H-(1-\gamma)\mu_{(-1)}\mu_{(-1)}^{\top}$.

    On the one hand,
    using similar arguments as in the proof of Theorem \ref{thm:PCTD0},
    we have
    \begin{equation*}
        (\theta^*_{(-1)})^{\top}\Ht\theta^*_{(-1)}
        \geq
        (1-\gamma(1-g))
        (\theta^*_{(-1)})^{\top}\SSt_0\theta^*_{(-1)}
        \geq
        (1-\gamma(1-g))
        \oot|\theta^*_{(-1)}|^2,
    \end{equation*}
    where $\SSt_0=\EE\lb(\vp_{(-1)}(X)-\mu_{(-1)})(\vp_{(-1)}(X)-\mu_{(-1)})^{\top}\rb$
    and $\oot>0$ is the smallest eigenvalue of $\SSt_0$ which is positive using
    the linear independence of the features and the fact that $\vp_1$ is constant.

    On the other hand, we have
    \begin{equation*}
        (\theta^*_{(-1)})^{\top}\Ht\theta^*_{(-1)}
        =
        (\theta^*_{(-1)})^{\top}(\bt-\EE[R]\mu_{(-1})
        \leq
        |\theta^*_{(-1)}|\EE[R(\vp_{(-1)}-\mu_{(-1)})]|
        \leq
        C_R^{\frac12}|\theta_{(-1)}|,
    \end{equation*}
    where $C_R$ is defined in Assumption \ref{hypo:R}.

    Consequently, we obtain
    \begin{equation}
    \label{eq:bound_param_trick_theta-1}
        |\theta^*_{(-1)}|
        \leq
        \frac{C_R^{\frac12}}{\oot(1-\gamma(1-g))}
        \leq
        \frac{C_R^{\frac12}}{\oot g},
    \end{equation}
    and then $|\theta^*_1|\leq C_R^{\frac12}+\frac{C_R^{\frac12}}{\oot }$.
    This concludes the uniform boundedness of $\theta^*$.
    
    Let us now prove the uniform boundedness of $\ss_0^2$ that is given by
    \begin{align*}
        \ss_0^2
        &=
        \sup_{x\in\supp m}
        \EE\lb |(\vp(X)-\gamma\vp(X'))^{\top}\theta^*-R|^2\,|\,X=x\rb
        \\
        &=
        \sup_{x\in\supp m}
        \EE\lb |\theta^*_1+(\vp_{(-1)}(X)-\gamma\vp_{(-1)}(X'))^{\top}\theta_{(-1)}^*-R|^2\,|\,X=x\rb
        \\
        &=
        \sup_{x\in\supp m}
        \EE\lb |
        \EE[R]-(1-\gamma)\mu_{(-1)}^{\top}\theta^*_{(-1)}
        +(\vp_{(-1)}(X)-\gamma\vp_{(-1)}(X'))^{\top}\theta_{(-1)}^*-R|^2\,|\,X=x\rb
        \\
        &\leq
        \sup_{x\in\supp m}
        2\EE\lb |\EE[R]-R|^2\,|\,X=x\rb
        +2\EE\lb
        (\vp_{(-1)}(X)-\gamma\vp_{(-1)}(X')-(1-\gamma)\mu_{(-1)})^{\top}\theta_{(-1)}^*|^2\,|\,X=x\rb
        \\
        &\leq
        8C_R+18|\theta^*_{(-1)}|^2
        \\
        &\leq
        C_R\lp 8 +\frac{18}{\oot^2g^2}\rp,
    \end{align*}
    where we used \eqref{eq:bound_param_trick_theta1} to get the third line,
    and \eqref{eq:bound_param_trick_theta-1} to obtain the last one.
    This concludes the proof.
    
\end{proof}

\subsection{Proof of Theorem \ref{thm:TD0_param_trick}}
\begin{proof}[Proof of Theorem \ref{thm:TD0_param_trick}]
    From the inequalities
    stated in Lemma \ref{lem:specific_bounds_different_LR},
    we can apply Proposition \ref{prop:LSA_alternative}
    with $\Sigma=\Sigma_0$,
    $\beta=1+(1+\gamma)^2$, $c_{\Sigma}=(1-\gamma)^{-1}$
    and $\ss^2=\ss_0^2$.
    It only remains to upper bound $\cS_k$ using Proposition
    \ref{prop:Kreiss_bis} to obtain the desired rate.

    The uniform boundedness of $\theta^*$ with respect to $\gamma$ is 
    given by Lemma \ref{lem:bound_param_trick}.
\end{proof}

\begin{proposition}
    \label{prop:Kreiss_bis}
    Under Assumptions \ref{hypo:iid},\ref{hypo:R},
    \ref{hypo:linear_bis},
    the upper bound in Proposition \ref{prop:Kreiss}
    applies.
\end{proposition}
\begin{proof}
    First observe that 
    $Q=\Sigma_0^{\frac12}(I_d-\alpha H)\Sigma_0^{-\frac12}$,
    thus it admits the same spectrum
    as $I_d-\alpha H$.
    Moreover, 
    using \eqref{eq:bound_S_bis},
    \eqref{eq:bound_HHtop_bis} we can take
    $c_{\Sigma}=(1-\gamma)^{-1}$
    $\beta=1+(1+\gamma)^2$
    so that 
    \eqref{eq:bound_Id-aaH} applies
    for $\alpha<\frac{2}{\beta c_{\Sigma}}
    =\frac{2(1-\gamma)}{1+(1+\gamma)^2}=\alpha_2(\gamma)$.
    Then \eqref{eq:bound_Id-aaH} implies that the spectrum
    of $I_d-\alpha H$ is contained in the complex unit sphere.
    Therefore, Lemma \ref{lem:series} applies
    to $Q$ and Inequality \eqref{eq:bound_cS} holds.

    Then, the first step of the proof of Proposition
    \ref{prop:Kreiss} can be repeated without any change.
    Concerning the second step,
    the only change occurs to obtain the fourth line
    in the chain of inequalities \eqref{eq:aux_case2}:
    it requires that
    $\norm{\Sigma_0U-U^{\top}\Sigma_0}{\rm op}\leq 2$
    is indeed satisfied under the present assumption
    as stated in \eqref{eq:aux_antisym}.
    
    Therefore, there only remains the third step,
    i.e., for $\Re(z)\leq -1+(1-\gamma)^2$,
    that we handle with a different approach using $\gamma\geq\frac12$.
    First, we have
    \begin{equation*}
        \alpha
        <
        \frac{2(1-\gamma)}{1+(1+\gamma)^2}
        \leq
        \frac{2(1-\gamma)}{1+\lp1+\frac12\rp^2}
        =
        \frac{8}{13}(1-\gamma),
    \end{equation*}
    then, using Inequality \eqref{eq:Sima0_Id-gammaU}, we get
    \begin{align*}
        \alpha^2\norm[2]{\Sigma_0(I_d-\gamma U)}{\rm op}
        &\leq
        \frac{(1-\gamma)^2}4\norm[2]{\Sigma_0(I_d-\gamma U)}{F}
        \\
        &\leq
        \frac{8^2(1-\gamma)^2}{13^2}
        \lp(1-\gamma)^2(((1-\gamma)^{-2}+1)^2+1)+2(1+\gamma)^2\rp
        \\
        &=
        \frac{8^2}{13^2}\lp(1+(1-\gamma)^2)^2+(1-\gamma)^4)+2(1-\gamma^2)^2\rp
        \\
        &\leq
        \frac{8^2}{13^2}\lp\frac{25}{16}+\frac1{16}+\frac{9}{8}\rp
        \leq
        1.01^2.
    \end{align*}
    Moreover, using $-\cR(z)\geq 1-(1-\gamma)^2\geq \frac34$,
    we obtain
    \begin{equation*}
        |z-c_k|^2
        =
        (c_k-\cR(z))^2+\cI(z)^2
        =
        c_k^2-2\cR(z)c_k+1
        \geq
        c_k^2+\frac32c_k+1.
    \end{equation*}
    Therefore, the above inequalities imply that,
    for $y\in\CC^d$ with $|y|=1$,
    we have
    \begin{align*}
        |(z-\Qt)y|
        &=
        |\lp(c_k-z)I_d
        -\alpha c_k\Sigma_0(I_d-\gamma U)\rp y|
        \\
        &\geq
        |c_k-z||y|
        -\alpha c_k\norm{\Sigma_0(I_d-\gamma U)}{\rm op}|y|
        \\
        &\geq
        \sqrt{c_k^2+\frac32c_k+1}
        -1.01c_k=f(c_k),
    \end{align*}
    where $f:x\in[0,1]\to\sqrt{x^2+\frac32x+1}-1.01x$.
    We have
    \begin{equation*}
        f'(x)
        =
        \frac{x+\frac34x}{\sqrt{1+\frac32x+x^2}}-1.01
        =
        \sqrt{\frac{x^2+\frac32x+\frac9{16}}{1+\frac32x+x^2}}-1.01
        \leq
        0.
    \end{equation*}
    Consequently, we obtain
    $\norm{(z-\Qt)^{-1}}{\rm op}
    \leq f(1)^{-1}=\lp\sqrt{\frac72}-1.01\rp^{-1}
    \leq2$.

    The remainder of the proof is similar to that of Proposition
    \ref{prop:Kreiss}.
\end{proof}

\subsection{Proof of Theorem \ref{thm:PCTD0}}
\begin{proof}[Proof of Theorem \ref{thm:PCTD0}]
    Define $\vph=\vp-\mu$,
    $\SSh_0=\EE[\vph(X)\vph(X)^{\top}]=\Sigma_0-\mu\mu^{\top}$,
    $\SSh_1=\EE[\vph(X)\vph(X')^{\top}]=\Sigma_1-\mu\mu^{\top}$,
    $\Hh=\SSh_0-\gamma \SSh_1$,
    $\Sh=\frac{\Hh+\Hh^{\top}}2$ and
    $\bh=\EE[R\vph(X)]=b-\EE[R]\mu$.
    From Lemma \ref{lem:specific_bounds_PCTD0},
    we can apply Proposition \ref{prop:LSA}
    with $\Sigma=\SSh_0$, $\beta=2(1+\gamma)^2$, $\beta_1=2(1-\gamma)$,
    $c_{\Sigma}=1-(1-g)\gamma$,
    this implies
    \begin{equation*}
        \EE_{X\sim m}\lb|\vh(X,\tho_k)-\vh(X,\theta^*)|^2\rb
        \leq
        \frac1k
        \lp\lp\frac{|\theta_0-\thh^*|^2}
        {2\lp 1-\frac{\alpha}{\aah_1(\gamma)}\rp
        \alpha(1-(1-g)\gamma)}\rp^{\frac12}
        +\lp\frac{2\ssh_0^2\cS_k}{1-\frac{\alpha}{\aah_0(\gamma)}}\rp^{\frac12}\rp^2.
    \end{equation*}  
    It only remains to prove that Proposition \ref{prop:Kreiss}
    holds with $\gamma$ replaced by $(1-g)\gamma$.
    This is what we prove in the remainder of this proof.
    
    Let us define $\Ut$ by
    $\Ut=(1-g)^{-1}\SSh_0^{-\frac12}\SSh_1\SSh_0^{-\frac12}$,
    it satisfies $\|\Ut\|_{\rm op}\leq 1$ by
    Inequality \eqref{eq:bound_Uh}.
    Then we obtain 
    \begin{equation*}
        \Hh=\SSh_0^{\frac12}(I_d-(1-g)\gamma\Ut)\SSh_0^{\frac12},
    \end{equation*}
    which is the same structure as for TD(0) but with $\gamma$
    replaced with $(1-g)\gamma$.

    The only property of $U$ used to prove Proposition \ref{prop:Kreiss}
    was its operator norm bounded by one.
    Therefore, we can repeat all the proof in the case of
    $\Hh$ with $\gamma$ replaced by $(1-g)\gamma$.

    This concludes the extension of Theorem \ref{thm:TD0}
    to PCTD(0) as stated in Theorem \ref{thm:PCTD0}.
    The extensions of Theorems \ref{thm:TD0_larger_alpha} and
    \ref{thm:minibatch_TD0} works exactly the same.
\end{proof}

\subsection{Some results on complex integration}
\begin{lemma}
    \label{lem:int2}
    For $u\in[0,1)$, we have
    \begin{equation*}
        \int_0^{2\pi}
        \frac{d\theta}{|e^{-i\theta}-u|^2}
        =
        \frac{2\pi}{1-u^2}
    \end{equation*}
\end{lemma}
\begin{proof}
    Let $G=\{z\in\CC, |z|=1\}$,
    we have
    \begin{align*}
        \int_0^{2\pi}
        \frac{d\theta}{|e^{-i\theta}-u|^2}
        &=
        \int_G
        \frac{1}{|z-u|^2}
        \frac{dz}{iz}
        \\
        &=
        \int_G
        \frac{dz}{(z-u)(z^{-1}-u)iz}
        \\
        &=
        \int_G
        \frac{dz}{i(z-u)(1-zu)}.
    \end{align*}
    Recall that the function 
    $\frac{1}{i(z-u)(1-zu)}$ is holomorphic on
    $D(0,1)\backslash\{u\}$ with a pole in $u$.
    Therefore, we can use the Residue Theorem to get
    \begin{equation*}
        \int_0^{2\pi}
        \frac{d\theta}{|e^{-i\theta}-u|^2}
        =
        2\pi i \;{\rm Res}_u
        \lp z\mapsto\frac{1}{i(z-u)(1-zu)}\rp
        =
        2\pi i \frac1{i(1-u^2)}
        =
        \frac{2\pi}{1-u^2}.
    \end{equation*}
    This concludes the proof.
\end{proof}

\subsection{Few results from linear algebra}
\subsubsection{Results for standard TD(0)}
The results presented here are used to prove the results
in \ref{subsec:cvg_TD0}.
\begin{lemma}
    \label{lem:specific_bounds}
    Assume \ref{hypo:iid}-\ref{hypo:linear}.
    Using the notations of the algorithm TD(0)
    from Section \ref{sec:intro},
    define $h^{\rm TD}=\vp(X)(\vp(X)-\gamma\vp(X'))$
    and $S^{\rm TD}=\Sigma_0-\gamma\Sigma_1$,
    we have
    \begin{align}
        \label{eq:bound_S}
        (1-\gamma)\Sigma_0
        &\leq
        S^{\rm TD}
        \leq
        (1+\gamma)\Sigma_0,
        \\
        \label{eq:bound_HHtop}
        \EE\lb h^{\rm TD}(h^{\rm TD})^{\top}\rb
        &\leq
        (1+\gamma)^2\Sigma_0,
        \\
        \label{eq:bound_htoph}
        \EE\lb (h^{\rm TD})^{\top}h^{\rm TD}\rb
        &\leq
        (1+\gamma)S^{\rm TD},
        \\
        \label{eq:ss0}
        \EE\lb (h^{\rm TD}\theta^*-b^{\rm TD})(h^{\rm TD}\theta^*-b^{\rm TD})^{\top}\rb
        &\leq
        \ss_0^2\Sigma_0.
    \end{align}
\end{lemma}
\begin{proof}
    To get \eqref{eq:bound_S},
    it is sufficient to observe
    $$
        S^{\rm TD}
        =
        \Sigma_0-\frac{\gamma}2
        \EE\lb\vp(X)\vp(X')^{\top}
        +\vp(X')\vp(X)^{\top}\rb,
    $$
    and use Cauchy-Schwarz inequality.

    To get \eqref{eq:bound_HHtop},
    we compute
    \begin{align*}
        \EE\lb h^{\rm TD}(h^{\rm TD})^{\top}\rb
        &=
        \EE\lb|\vp(X)-\gamma\vp(X')|^2\vp(X)\vp(X)^{\top}\rb
        \\
        &\leq
        (1+\gamma)^2\EE\lb\vp(X)\vp(X)^{\top}\rb
        =
        (1+\gamma)^2\Sigma_0.
    \end{align*}
    Now let us prove \eqref{eq:bound_htoph}
    \begin{align*}
        \EE\lb (h^{\rm TD})^{\top}h^{\rm TD}\rb
        &=
        \EE\lb|\vp(X)|^2
        (\vp(X)-\gamma\vp(X'))
        (\vp(X)-\gamma\vp(X'))^{\top}\rb
        \\
        &\leq
        \EE\lb
        (\vp(X)-\gamma\vp(X'))
        (\vp(X)-\gamma\vp(X'))^{\top}\rb
        \\
        &=
        \EE\lb
        (1+\gamma^2)\vp(X)\vp(X)^{\top}
        -\gamma(\vp(X)\vp(X')^{\top}+\vp(X')\vp(X))\rb
        \\
        &=
        (1+\gamma^2)\Sigma_0
        -\gamma(\Sigma_1+\Sigma_1^{\top})
        \\
        &=
        (1+\gamma)S^{\rm TD}
        -\gamma(1-\gamma)\lp\Sigma_0+\frac{\Sigma_1+\Sigma_1^{\top}}2\rp
        \\
        &\leq
        (1+\gamma)S^{\rm TD}.
    \end{align*}
    We conclude the proof with \eqref{eq:ss0}:
    \begin{equation*}
        \EE\lb (h^{\rm TD}\theta^*-b^{\rm TD})(h^{\rm TD}\theta^*-b^{\rm TD})^{\top}\rb
        =
        \EE\lb|\delta(X,X',\theta^*)|^2\vp(X)\vp(X)^{\top}\rb
        \leq
        \ss_0^2\Sigma_0.
    \end{equation*}
\end{proof}

\subsubsection{TD(0) with two separate learning rates}
The results presented here are used to prove the results
in \ref{subsec:reduce_gamma_dep}.
\begin{lemma}
    \label{lem:specific_bounds_different_LR}
    Assume \ref{hypo:iid}, \ref{hypo:R} and \ref{hypo:linear_bis}.
    We have
    \begin{align}
        \label{eq:bound_S_bis}
        (1-\gamma)\Sigma_0
        &\leq
        S^{\rm TD}
        \leq
        (1+\gamma)\Sigma_0,
        \\
        \label{eq:bound_HHtop_bis}
        \EE\lb h^{\rm TD}(h^{\rm TD})^{\top}\rb
        &\leq
        ((1-\gamma)^2c_{\gamma}^2+(1+\gamma)^2)\Sigma_0,
        \\
        \label{eq:ss0_bis}
        \EE\lb (h^{\rm TD}\theta^*-b^{\rm TD})(h^{\rm TD}\theta^*-b^{\rm TD})^{\top}\rb
        &\leq
        \ss_0^2\Sigma_0,
        \\
        \label{eq:Sigma0U}
        \norm{\Sigma_0U-U^{\top}\Sigma_0}{\rm op}
        &\leq
        2,
        \\
        \label{eq:Sima0_Id-gammaU}
        \norm[2]{\Sigma_0(I_d-\gamma U)}{F}
        &\leq
        (1-\gamma)^2((c_{\gamma}^2+1)^2+1)+2(1+\gamma)^2,
    \end{align}
\end{lemma}

\begin{proof}
    The proofs of \eqref{eq:bound_S_bis} and \eqref{eq:ss0_bis}
    are similar as the ones of \eqref{eq:bound_S} and \eqref{eq:ss0}.
    To get \eqref{eq:bound_HHtop_bis}, we compute
    \begin{align*}
        \EE\lb h^{\rm TD}(h^{\rm TD})^{\top}\rb
        &=
        \EE\lb|\vp(X)-\gamma\vp(X')|^2\vp(X)\vp(X)^{\top}\rb
        \\
        &\leq
        ((1-\gamma)^2c_{\gamma}^2+(1+\gamma)^2)\Sigma_0,
    \end{align*}
    where we used that 
    $|\vp(X)-\gamma\vp(X')|^2=(1-\gamma)^2c_{\gamma}^2+|\vp_{(-1)}(X)-\gamma\vp_{(-1)}(X')|^2
    \leq
    (1-\gamma)^2c_{\gamma}^2+(1+\gamma)^2$.

    It only remains to prove \eqref{eq:Sigma0U}.
    Let us define $\SSt_0,\SSt_1\in\RR^{(d-1)\times(d-1)}$ by
    \begin{equation*}
        \SSt_0
        =
        \EE\lb(\vp_{(-1)}(X)-\mu_{(-1)})(\vp_{(-1)}(X)-\mu_{(-1)})^{\top}\rb
        \hspace*{0.5cm}
        \text{ and }
        \hspace*{0.5cm}
        \SSt_1
        =
        \EE\lb(\vp_{(-1)}(X)-\mu_{(-1)})(\vp_{(-1)}(X')-\mu_{(-1)})^{\top}\rb.
    \end{equation*}
    Define $\Ut=\SSt_0^{-\frac12}\SSt_1\SSt_0^{-\frac12}$.
    Consider 
    \begin{equation*}
        L
        =
        \begin{pmatrix}
        c_{\gamma} & 0\\
        \mu_{(-1)} & \SSt_0^{\frac12}
        \end{pmatrix}
        \hspace*{0.2cm}
        \text{ so that}
        \hspace*{0.2cm}
        LL^{\top}=\SSt_0,
        \hspace*{0.5cm}
        L\begin{pmatrix}
            1&0\\
            0&\Ut
        \end{pmatrix}L^{\top}=\SSt_1,
        \hspace*{0.5cm}
        L^{\top}L
        =
        Q^{\top}\SSt_0Q
        \hspace*{0.2cm}
        \text{ and }
        \hspace*{0.2cm}
        U=Q\begin{pmatrix}
            1&0\\
            0&\Ut
            \end{pmatrix}Q^{\top},
    \end{equation*}
    where $Q:=\SSt_0^{-\frac12}L\in \cO(\RR^{d-1})$.
    This implies that
    \begin{align*}
        Q^{\top}(\Sigma_0U-U^{\top}\Sigma_0)Q
        &=
        L^{\top}L
        \begin{pmatrix}
            1&0\\
            0&\Ut
        \end{pmatrix}
        -\begin{pmatrix}
            1&0\\
            0&\Ut^{\top}
        \end{pmatrix}L^{\top}L
        \\
        &=
        \begin{pmatrix}
            c_{\gamma}^2+|\mu_{(-1)}|^2&
            \mu_{(-1)}^{\top}\SSt_0^{\frac12}\\
            \SSt_0^{\frac12}\mu_{(-1)}&\SSt_0
        \end{pmatrix}
        \begin{pmatrix}
            1&0\\
            0&\Ut
        \end{pmatrix}
        -\begin{pmatrix}
            1&0\\
            0&\Ut
        \end{pmatrix}
        \begin{pmatrix}
            c_{\gamma}^2+|\mu_{(-1)}|^2
            &\mu_{(-1)}^{\top}\SSt_0^{\frac12}\\
            \SSt_0^{\frac12}\mu_{(-1)}&\SSt_0
        \end{pmatrix}
        \begin{pmatrix}
            1&0\\
            0&\Ut
        \end{pmatrix}
        \\
        &=
        \begin{pmatrix}
            c_{\gamma}^2+|\mu_{(-1)}|^2
            &\mu_{(-1)}^{\top}\SSt_0^{\frac12}\Ut\\
            \SSt_0^{\frac12}\mu_{(-1)}&\SSt_0\Ut
        \end{pmatrix}
        -\begin{pmatrix}
            c_{\gamma}^2+|\mu_{(-1)}|^2&
            \mu_{(-1)}^{\top}\SSt_0^{\frac12}\\
            \Ut^{\top}\SSt_0^{\frac12}\mu_{(-1)}&\Ut^{\top}\SSt_0
        \end{pmatrix}
        \\
        &=
        \begin{pmatrix}
            0&\mu_{(-1)}^{\top}\SSt_0^{\frac12}(\Ut-I_d)\\
            (I_d-\Ut^{\top})\SSt_0^{\frac12}\mu_{(-1)}&\SSt_0\Ut-\Ut^{\top}\SSt_0
        \end{pmatrix}.
        \\
        &=
        \begin{pmatrix}
            0&-y^{\top}\\
            y
            &0
        \end{pmatrix}
        +\begin{pmatrix}
            0&0\\
            0&\SSt_0\Ut-\Ut^{\top}\SSt_0
        \end{pmatrix},
    \end{align*}
    where $y:=(I_d-\Ut^{\top})\SSt_0^{\frac12}\mu_{(-1)}$.
    Therefore, we obtain
    \begin{equation}
        \label{eq:aux_antisym}
        \norm{\Sigma_0U-U^{\top}\Sigma_0}{\rm op}
        \leq
        |y|^2
        +\norm{\SSt_0\Ut-\Ut^{\top}\SSt_0}{\rm op}.
    \end{equation}
    On the one hand, we have
    \begin{align*}
        |y|^2
        \leq
        2\norm{\SSt_0^{\frac12}}{\rm op}|\mu_{(-1)}|
        \leq
        \norm{\SSt_0}{\rm op}+|\mu_{(-1)}|^2
        &\leq
        \tr(\SSt_0)+|\mu_{(-1)}|^2
        \\
        &=
        \EE\lb|\vp_{(-1)}(X)-\mu_{(-1)}|^2\rb+\mu_{(-1)}^2
        =
        \EE\lb|\vp_{(-1)}(X)|^2\rb
        \leq
        1.
    \end{align*}
    On the other hand, take $(\lambda_i,v_i)_{1\leq i\leq d-1}$
    the couples of eigenvalues and eigenvectors of $\SSt_0$,
    we have
    \begin{align*}
        \norm{\SSt_0\Ut-\Ut^{\top}\SSt_0}{\rm op}
        &=
        \sum_{i=1}^{d-1}
        \lambda_i(v_i(\Ut^{\top}v_i)^{\top}-(\Ut^{\top}v_i)v_i^{\top})
        \\
        &\leq
        \sum_{i=1}^{d-1}
        \lambda_i
        \norm{v_i(\Ut^{\top}v_i)^{\top}-(\Ut^{\top}v_i)v_i^{\top}}{\rm op}
        \\
        &\leq
        \sum_{i=1}^{d-1}
        \lambda_i
        |v_i||\Ut^{\top}v_i|
        \\
        &\leq
        \sum_{i=1}^{d-1}\lambda_i
        =\tr(\SSt_0)\leq 1.
    \end{align*}
    where we used Lemma \ref{lem:antisym_operator} to
    get the third line,
    $\|\Ut\|_{\rm op}\leq1$ to get the last line.
    We conclude using \eqref{eq:aux_antisym}
    and the above inequalities.

    Let us now prove Inequality \eqref{eq:Sima0_Id-gammaU}.
    Using similar computations as above, we have
    \begin{align*}
        Q\Sigma_0(I_d-\gamma U)Q^{\top}
        =
        L^{\top}L \lp I_d-\gamma
        \begin{pmatrix}
            1&0\\
            0&\Ut
        \end{pmatrix}\rp
        &=
        \begin{pmatrix}
            c_{\gamma}^2+|\mu_{(-1)}|^2&
            \mu_{(-1)}^{\top}\SSt_0^{\frac12}\\
            \SSt_0^{\frac12}\mu_{(-1)}&\SSt_0
        \end{pmatrix}
        \begin{pmatrix}
            1-\gamma&0\\
            0&I_{d-1}-\gamma\Ut
        \end{pmatrix}
        \\
        &=
        \begin{pmatrix}
            (1-\gamma)(c_{\gamma}^2+|\mu_{(-1)}|^2)&
            \mu_{(-1)}^{\top}\SSt_0^{\frac12}(I_{d-1}-\gamma \Ut)\\
            (1-\gamma)\SSt_0^{\frac12}\mu_{(-1)}&\SSt_0(I_{d-1}-\gamma \Ut)
        \end{pmatrix}.
    \end{align*}
    Therefore, we obtain
    \begin{align*}
        \norm[2]{\Sigma_0(I_d-\gamma U)Q^{\top}}{F}
        &=
        (1-\gamma)^2(c_{\gamma}^2+|\mu_{(-1)}|^2)^2
        +\labs(I_d-\gamma \Ut^{\top})\SSt_0^{\frac12}\mu_{(-1)}\rabs^2
        \\
        &\hspace*{1cm}+\labs(1-\gamma)\SSt_0^{\frac12}\mu_{(-1)}\rabs^2
        +\norm[2]{\SSt_0(I_{d-1}-\gamma \Ut)}{F}
        \\
        &\leq
        (1-\gamma)^2(c_{\gamma}^2+1)^2
        +(1+\gamma)^2
        +(1-\gamma)^2
        +(1+\gamma)^2
        \\
        &=
        (1-\gamma)^2((c_{\gamma}^2+1)^2+1)+2(1+\gamma)^2,
    \end{align*}
    where we used $\|\Ut\|_{\rm op},\|\SSt\|_{\rm op}\leq 1$
    and $\norm[2]{\SSt_0(I_{d-1}-\gamma \Ut)}{F}
    \leq(1+\gamma)^2\tr(\SSt_0^2)\leq (1+\gamma)^2\tr(\SSt_0)\leq (1+\gamma)^2$
    since $\tr(\SSt_0)=\EE[|\vp_{(-1)}(X)|]\leq1$.
\end{proof}

\begin{lemma}
    \label{lem:antisym_operator}
    For $u,v\in\RR^d$,
    we have $\norm{uv^{\top}-vu^{\top}}{\rm op}\leq |v||u|$.
\end{lemma}

\begin{proof}
    It is sufficient to prove the result for $|u|=|v|=1$.
    Take $v=\lambda_1 u+\lambda_2w$, with 
    $\lambda_1^2+\lambda_2^2=1$, $w\in u^{\perp}$ and $|w|=1$.
    We have
    \begin{equation*}
        uv^{\top}-vu^{\top}
        =
        \lambda_2(uw^{\top}-wu^{\top}).
    \end{equation*}
    Moreover, since $u\perp w$, for any $x\in\RR^d$, we have
    \begin{equation*}
        |(uw^{\top}-wu^{\top})x|^2
        =
        |w^{\top}x|^2+|u^{\top}x|^2
        =
        \labs\Pi_{{\rm Span}(u,w)}x\rabs^2
        \leq
        |x|^2,
    \end{equation*}
    where $\Pi_{{\rm Span}(u,w)}$ is the orthogonal
    projection on Span$(u,w)$.
    Therefore, we obtain
    \begin{equation*}
        \norm{uv^{\top}-vu^{\top}}{\rm op}
        \leq
        \lambda_2
        \norm{uw^{\top}-wu^{\top}}{\rm op}
        \leq
        \lambda_2
        \leq
        1.
    \end{equation*}
    This concludes the proof.
\end{proof}

\subsubsection{PCTD(0)}
For PCTD(0), we need to inequalities stated in the following lemma.
\begin{lemma}
    \label{lem:specific_bounds_PCTD0}
    Assume \ref{hypo:iid}-\ref{hypo:spectral_gap}.
    Define $\hhh=\frac12(\vp(X)-\vp(\Xt))
    \lp\vp(X)-\gamma\vp(X')
    -(\vp(\Xt)-\gamma\vp(\Xt'))\rp^{\top}$,
    $\SSh_0=\Sigma_0-\mu\mu^{\top}$,
    $\SSh_1=\Sigma_1-\mu\mu^{\top}$,
    $\Uh=\SSh_0^{-\frac12}\SSh_1\SSh_0^{-\frac12}$,
    $\Hh=\EE[\hhh]=\SSh_0-\gamma\SSh_1=H-(1-\gamma)\mu\mu^{\top}$
    and $\Sh=\frac{\Hh+\Hh^{\top}}2$.
    we have
    \begin{align}
        \label{eq:bound_Uh}
        \norm{U}{\rm op}
        &\leq
        1-g
        \\
        \label{eq:bound_Sh}
        (1-(1-g)\gamma)\SSh_0
        &\leq
        \Sh
        \leq
        (1+(1-g)\gamma)\SSh_0,
        \\
        \label{eq:bound_HhHhtop}
        \EE\lb \hhh\hhh^{\top}\rb
        &\leq
        2(1+\gamma)^2\SSh_0,
        \\
        \label{eq:bound_hhhtophhh}
        \EE\lb \hhh^{\top}\hhh\rb
        &\leq
        2(1+\gamma)\Sh,
        \\
        \label{eq:ssh0}
        \EE\lb (\hhh\theta^*-\bh)(\hhh\theta^*-\bh)^{\top}\rb
        &\leq
        2\ssh_0^2\SSh_0.
    \end{align}
\end{lemma}

\begin{proof}
    Let $x,y\in\RR^d$ with $|x|=|y|=1$,
    define $\theta_x=\SSh_0^{-\frac12}x$
    and $\theta_y=\SSh_0^{-\frac12}y$,
    we have
    \begin{align*}
        x^{\top}\Uh y
        =
        \theta_x^{\top}\SSh_1\theta_y
        &=
        \EE\lb \vh(X,\theta_x)\vh(X',\theta_y)\rb
        \\
        &=
        \EE\lb \vh(X,\theta_x)\EE\lb\vh(X',\theta_y)\,\big\|\,X\rb\rb
        \\
        &\leq
        \EE\lb \vh(X,\theta_x)^2\rb^{\frac12}
        \EE\lb\EE\lb\vh(X',\theta_y)\,\big\|\,X\rb^2\rb^{\frac12}
        \\
        &\leq
        (1-g)
        \EE\lb \vh(X,\theta_x)^2\rb^{\frac12}
        \EE\lb \vh(X,\theta_y)^2\rb^{\frac12}
        \\
        &=
        (1-g)
        (\theta_x^{\top}\SSh_0\theta_x)^{\frac12}
        (\theta_y^{\top}\SSh_0\theta_y)^{\frac12}
        \\
        &=
        (1-g)
        |x||y|
        =1-g,
    \end{align*}
    which implies Inequality \eqref{eq:bound_Uh}.
    Using this, Inequality \eqref{eq:bound_Sh} is straightforward.
    Finally, the proofs of the remaining inequalities are similar
    as the ones of their counterparts from Lemma \ref{lem:specific_bounds}.
\end{proof}

\end{document}